\documentclass[preprint,3p]{elsarticle}

\usepackage{tabularx}
\usepackage{colortbl}
\usepackage{amsthm}
\usepackage{amsmath,amssymb,enumerate,subcaption,epsfig,psfrag,hhline,srcltx}

\usepackage{color}
\usepackage{xcolor}
\usepackage{hyperref}
\hypersetup{breaklinks=true,colorlinks=true,citecolor=blue}
\usepackage{graphicx}          
\usepackage{multirow}
\usepackage{array}
\newcolumntype{P}[1]{>{\centering\arraybackslash}p{#1}}
\usepackage{enumitem}
\usepackage{makecell}

\usepackage[noabbrev,capitalise,nameinlink]{cleveref}

\biboptions{sort&compress}

\newdefinition{example}{Example}

\newdefinition{remark}{Remark}

\newtheorem{lemma}{Lemma}
\newtheorem{definition}{Definition}
\newtheorem{proposition}{Proposition}

\usepackage{algorithm}
\usepackage{algpseudocode}

\allowdisplaybreaks

\begin{document}
\begin{frontmatter}

\title{Safe-Support Q-Learning: Learning without Unsafe Exploration}

\author{Yeeun Lim\footnotemark[1]$^a$}\ead{yeeun141054@gmail.com}
\author{Narim Jeong\footnotemark[1]$^b$}\ead{nrjeong@kaist.ac.kr}
\author{Donghwan Lee$^b$}\ead{donghwan@kaist.ac.kr}

\address[a]{HD Korea Shipbuilding \& Offshore Engineering}
\address[b]{Department of Electrical Engineering, KAIST, Daejeon, South Korea}

\footnotetext[1]{Equal contribution}

\begin{abstract}
Ensuring safety during reinforcement learning (RL) training is critical in real-world applications where unsafe exploration can lead to devastating outcomes. While most safe RL methods mitigate risk through constraints or penalization, they still allow exploration of unsafe states during training. In this work, we adopt a stricter safety requirement that eliminates unsafe state visitation during training. To achieve this goal, we propose a Q-learning-based safe RL framework that leverages a behavior policy supported on a safe set. Under the assumption that the induced trajectories remain within the safe set, this policy enables sufficient exploration within the safe region without requiring near-optimality. We adopt a two-stage framework in which the Q-function and policy are trained separately. Specifically, we introduce a KL-regularized Bellman target that constrains the Q-function to remain close to the behavior policy. We then derive the policy induced from the trained Q-values and propose a parametric policy extraction method to approximate the optimal policy. Our approach provides a unified framework that can be adapted to different action spaces and types of behavior policies. Experimental results demonstrate that the proposed method achieves stable learning and well-calibrated value estimates and yields safer behavior with comparable or better performance than existing baselines.
\end{abstract}


\end{frontmatter}

\section{Introduction}
Reinforcement learning (RL)~\cite{sutton1998reinforcement} has become a powerful framework for autonomous decision-making. By optimizing cumulative rewards through the interaction with the environment, RL agents learn appropriate behaviors with minimal human intervention. However, as RL systems are increasingly deployed in high-risk domains such as autonomous driving and robotics, ensuring safety during the learning and implementation process has become a critical challenge.

The core difficulty of this issue arises from the fundamental nature of RL: optimal policies are often discovered through exploration. This process inevitably requires visiting a wide range of state-action pairs, including those that can lead to undesirable or unsafe outcomes. Traditional safe RL~\cite{gu2022review} attempts to mitigate this issue by incorporating constraints or risk-sensitive objectives into the learning process. For example, constrained policy optimization and actor-critic methods aim to maximize the expected return while satisfying safety constraints~\cite{yang2021wcsac,kim2024trust}, whereas conditional value-at-risk focuses on reducing the risk of catastrophic events in constrained settings~\cite{ying2022towards,kim2022trc,kim2022efficient}.

Despite these advances, most existing approaches still rely on interactions with unsafe states. For instance, safety violations are often controlled through penalties, expectation-based constraints, or bounded violation probabilities~\cite{xu2022constraints,liu2022constrained,peng2022model}, but they are not completely eliminated. Even methods such as safety critics~\cite{bharadhwaj2020conservative,yang2023safety} require exposure to unsafe transitions to estimate the risk of state-action pairs. Consequently, safety has been incorporated as a regularized objective rather than being strictly guaranteed in the majority of prior works.

In parallel, an alternative perspective on safety has been extensively studied in control theory. Proportional–integral–derivative (PID) control~\cite{johnson2005pid}, Lyapunov-based stabilization~\cite{sastry1999lyapunov}, and model predictive control (MPC)~\cite{holkar2010overview} are examples of classical approaches for designing controllers whose trajectories stay inside the certified safe or stable regions. Rather than penalizing unsafe outcomes, these methods ensure that the system never deviates from a predetermined set. While such approaches do not directly optimize long-term reward in the RL sense, they provide principled mechanisms for guaranteeing safety at the trajectory level. This motivates the consideration of whether RL can be constrained to operate within the support of safe states.

Based on this consideration, this study adopts a stricter safety requirement: the learning process must not produce trajectories that visit unsafe states. Unlike conventional constrained RL~\cite{altman2021constrained, chow2018lyapunov, srinivasan2020learning, yang2021wcsac}, we require that unsafe states be entirely excluded from the trajectory distribution. However, restricting exploration to a safe subset of states may limit the achievable performance. This raises a fundamental question: how can an agent sufficiently explore and improve its policy under such strict safety constraints?

To address this challenge, we propose a Q-learning-based safe RL algorithm with three key components. First, we employ a behavior policy supported on a safe set. This policy does not need to be optimal but is intentionally stochastic to enable sufficient exploration within the safe region. It can be constructed either by using well-performing controllers or by training from safe datasets. Second, we modify the Q-learning objective by introducing a KL divergence regularization term that aligns the Q-function with the behavior policy. This regularization constrains the Q-function to remain within the support of safe behaviors while enabling improvement in expected return. Third, we provide a principled method for deriving a policy from the trained Q-function with theoretical justification. Since exact policy extraction is often computationally intractable, we further propose a surrogate policy that approximates the optimal policy induced by the Q-function.

The main contributions of this work are summarized as follows:
\begin{enumerate}
    \item We introduce a behavior policy that enables RL without unsafe exploration in both online and offline settings, requiring only stochastic (not near-optimal) policies under the assumption that the induced trajectories remain within the safe set.

    \item We adopt a two-stage safe-support Q-learning framework in which the Q-function and policy are trained separately. We first propose a KL-regularized Bellman target that constrains the Q-function to remain close to the behavior policy. We then derive the optimal policy and show that the safe Bellman operator admits a unique fixed point. In addition, we formulate a KL-regularized objective that balances reward and safety and propose a parametric policy extraction method for continuous actions.
    
    \item The proposed method provides a unified framework that can be adapted to different action spaces and types of behavior policies.

    \item Experiments show that the proposed method achieves stable learning and well-calibrated value estimates and yields safer behavior with comparable or better performance than baselines.
\end{enumerate}

\section{Related Work}
\textbf{Safe RL.}
Prior research in safe RL has primarily focused on incorporating safety constraints into the optimization objective. One example of this research is the constrained Markov decision process~\cite{altman2021constrained,wachi2020safe}, where the agent seeks to maximize cumulative rewards while satisfying explicit safety constraints. Risk-sensitive objectives such as conditional value-at-risk have also been adopted to control tail risks~\cite{ying2022towards,kim2022trc,kim2022efficient}. Chance-constrained formulations explicitly bound violation probabilities~\cite{chow2018risk} but still allow unsafe states with limited probability. Entropy regularization is often incorporated to improve exploration and stability in constrained variants such as WCSAC~\cite{yang2021wcsac}. Another line of work employs reward shaping~\cite{laud2004theory}, penalizing unsafe states or actions without explicit constraints. In value-based settings, safe Q-learning methods leverage safety through cost critics or safety critics that estimate the likelihood of unsafe transitions and discourage risky behaviors during learning~\cite{xu2022constraints,srinivasan2020learning}. Despite these advances, existing safe RL methods primarily aim to reduce or bound the risk associated with unsafe states, typically in expectation or probabilistic terms. As a result, unsafe states may still be visited during training, even if their frequency or impact is controlled. In contrast, our work imposes a stricter requirement: under the safe set assumption, unsafe states are excluded from the learning process. This enforces safety at the level of trajectory support rather than through risk penalization.

\textbf{Control-based safety.}
Control theory has provided principled tools for guaranteeing stability and safety. Classical controllers such as PID control~\cite{johnson2005pid} achieve robust stabilization via feedback. Lyapunov theory~\cite{sastry1999lyapunov} provides a systematic framework to certify stability by ensuring monotonic decrease of a Lyapunov function. MPC~\cite{holkar2010overview} further enables constraint-aware control by optimizing control actions over a finite horizon. These control-based approaches have also been combined with RL to improve safety. Lyapunov-based RL methods leverage Lyapunov functions as safety certificates or constraints during policy optimization~\cite{chow2018lyapunov,chow2019lyapunov}. MPC can also be used as a safety layer or a backup controller that filters or overrides unsafe actions~\cite{li2020robust,zanon2020safe}. Hybrid schemes that integrate conventional controllers with RL have also been explored to exploit the reliability of classical feedback control while benefiting from data-driven improvement~\cite{yoo2020hybrid}.

A key distinction between control-based safety and RL lies in their objectives. Control methods typically prioritize stability and constraint satisfaction, often without explicitly optimizing long-term reward. In contrast, RL focuses on maximizing cumulative reward through exploration, which may induce unsafe behaviors. Our work draws inspiration from the control-theoretic view of safety as an invariance property. Rather than penalizing or bounding violations, we enforce safety at the level of trajectory support by restricting the learning process to trajectories that remain within a safe set. To this end, we employ a behavior policy supported on the safe set, which can be constructed using control-theoretic principles or safe datasets. Under the safe set assumption, it generates trajectories within the safe region and enables reward-driven policy improvement.

\textbf{Behavior Regularization and Offline RL.}
Off-policy and offline RL methods~\cite{levine2020offline,fujimoto2019off,munos2016safe} often assume limited divergence between the behavior and target policies to ensure stable value estimation and mitigate extrapolation error. This proximity is commonly encouraged through KL divergence regularization~\cite{schulman2017proximal}, entropy-based objectives~\cite{haarnoja2018soft}, or conservative value estimation. In offline RL, where policies are optimized from fixed datasets without further interaction, such regularization is essential to address distributional shift between the behavior policy and the learned policy.

To reduce extrapolation error and overestimation on out-of-distribution (OOD) actions, several behavior-regularized and conservative approaches have been proposed. For example, conservative Q-learning (CQL)~\cite{kumar2020conservative} penalizes Q-values for actions outside the dataset distribution to prevent overestimation. Batch-Constrained Q-learning (BCQ)~\cite{fujimoto2019off} constrains action selection to those sampled from a learned behavior model, effectively limiting deviation from the dataset support. BEAR~\cite{kumar2019stabilizing} enforces closeness between the learned policy and the behavior policy using a divergence constraint to control extrapolation error. TD3+BC~\cite{fujimoto2021minimalist} introduces a behavior cloning regularization term to stabilize policy updates, while Implicit Q-Learning (IQL)~\cite{kostrikov2021offline} reduces distributional shift by implicitly restricting updates to high-advantage in-dataset actions. These methods primarily aim to improve statistical reliability under limited data coverage. However, such constraints do not imply safety, as standard offline RL neither assumes that unsafe states are excluded from the dataset nor guarantees that the learned policy remains within a safe region. In contrast, our framework assigns a different role to the behavior policy: it is assumed that the induced trajectory remains within the safe set. Accordingly, regularization is not introduced to stabilize training or reduce estimation bias but to ensure that policy updates remain within the safe set. As a result, the behavior policy acts as a structural safety constraint, restricting the learning process to trajectories within the safe set rather than serving as a stabilizing prior for value estimation.

\section{Preliminaries}
\subsection{Markov Decision Process}
We consider an infinite-horizon discounted Markov decision process (MDP). An MDP is defined by a tuple $({\cal S}, {\cal A}, P, r, \gamma)$, where ${\cal S}$ and ${\cal A}$ denote the state space and action space, respectively, $P(\cdot|s,a)$ is the transition dynamics, $r(s,a,s') \in \mathbb{R}$ is the reward function, and $\gamma \in [0,1)$ is the discount factor. We take into account both discrete and continuous spaces. In the discrete setting, ${\cal S}$ and ${\cal A}$ are finite sets, i.e., ${\cal S} = \{1,2,\ldots,|{\cal S}|\}$ and ${\cal A} = \{1,2,\ldots,|{\cal A}|\}$. In the continuous setting, ${\cal S} \subset \mathbb{R}^{d'}$ and ${\cal A} \subset \mathbb{R}^{d}$ are measurable spaces.

At each time step $k \in \{0,1,2,\ldots\}$, the agent observes a state $s_k \in {\cal S}$, selects an action $a_k \in {\cal A}$, transitions to the next state $s_{k+1} \sim P(\cdot|s_k,a_k)$, and receives a reward $r_{k+1} := r(s_k,a_k,s_{k+1})$. For simplicity, we assume that the reward function is deterministic, although all results extend to stochastic rewards. Then, a policy $\pi$ maps a state to a distribution over actions. In the discrete case, $\pi(\cdot|s) \in \Delta_{|{\cal A}|}$, where $\Delta_{|{\cal A}|}$ is the set of all probability distributions over the action space ${\cal A}$. In the continuous case, $\pi(\cdot|s)$ denotes a probability density over ${\cal A}$. In this setting, the objective of the MDP is to find an optimal policy $\pi^*$ that maximizes the expected cumulative discounted reward:
\[ \pi^* := \arg\max_{\pi \in \Theta} 
\mathbb{E}\left[ \left. \sum_{k=0}^{\infty} \gamma^k r_{k+1} \right| \pi \right], \]
where $\Theta$ denotes the set of all admissible policies, and the expectation is taken over trajectories $(s_0,a_0,s_1,a_1,\ldots)$ induced by the policy $\pi$ and the transition dynamics. Moreover, the Q-function associated with a policy $\pi$ is defined as
\[ Q^\pi(s,a) := \mathbb{E}\left[ \left. \sum_{k=0}^{\infty} \gamma^k r_{k+1} \right| s_0 = s, a_0 = a, \pi \right], \quad (s,a) \in {\cal S} \times {\cal A}. \]
Then, the optimal Q-function can be expressed as $Q^*(s,a) := Q^{\pi^*}(s,a)$.

Throughout the paper, we assume that the MDP satisfies standard regularity conditions (e.g., ergodicity or sufficient exploration) so that the relevant expectations and value functions are well-defined.

\subsection{Review of Q-Learning}
Before introducing our method, we briefly review the standard tabular Q-learning~\cite{watkins1992q}. 
Q-learning is a fundamental RL algorithm that iteratively updates the Q-function using the following update rule:
\[ Q(s_k,a_k) \leftarrow Q(s_k,a_k) + \alpha \left(y_k - Q(s_k,a_k)\right), \]
where $\alpha > 0$ is the learning rate. The target value $y_k$ is defined as
\[ y_k = r_{k+1} + {\bf 1}(s_{k+1}) \gamma \max_{a' \in {\cal A}} Q(s_{k+1},a'), \]
where ${\bf 1}(s_{k+1})$ is an indicator function that equals 0 if $s_{k+1}$ is a terminal state and 1 otherwise. Under the appropriate conditions, this iterative update converges to the optimal Q-function $Q^*$, from which an optimal policy can be obtained via the greedy rule $\pi^*(s) = \arg\max_{a \in {\cal A}} Q^*(s,a)$.

Despite its simplicity, standard Q-learning has several limitations. The usage of the max operator can lead to overestimation of Q-values, especially in large or continuous action spaces. Moreover, the greedy policy induced by Q-learning lacks inherent stochasticity, which may hinder effective exploration. To address these issues, soft Q-learning~\cite{haarnoja2017reinforcement,jeong2025unified,lee2024unified} introduces entropy regularization into the objective:
\[ J^\pi := \mathbb{E}\left[ \left. \sum_{k=0}^{\infty} \gamma^k \left( r_{k+1} + \lambda H(\pi(\cdot | s_k)) \right) \right| \pi \right], \]
where $H(\pi(\cdot|s))$ denotes the entropy of the policy and $\lambda > 0$ is a temperature parameter. This formulation encourages stochastic policies and improves the exploration-exploitation tradeoff. In the tabular setting, soft Q-learning modifies the target value as
\begin{equation} \label{eq:SQL_target}
    y_k^{\rm soft} = r_{k+1} + {\bf 1}(s_{k+1}) \gamma \max_{\pi \in \Delta_{|{\cal A}|}} \left\{ \sum_{a \in {\cal A}} \pi(a|s_{k+1}) Q_{\rm soft}(s_{k+1},a) + \lambda H(\pi(\cdot|s_{k+1})) \right\},
\end{equation}
which admits the following closed-form expression:
\begin{equation} \label{eq:SQL_target_closed_form}
    y_k^{\rm soft} = r_{k+1} + {\bf 1}(s_{k+1}) \gamma \lambda \ln \left( \sum_{a \in {\cal A}} \exp \left( \frac{Q_{\rm soft}(s_{k+1},a)}{\lambda} \right) \right).
\end{equation}
Compared to standard Q-learning, the max operator is replaced by a log-sum-exp function, which can be viewed as a smooth approximation of the maximum. Then, the corresponding optimal policy takes the softmax form:
\[ \pi_{\rm soft}^*(a|s) = \frac{\exp(Q_{\rm soft}^*(s,a)/\lambda)}{\sum_{a' \in {\cal A}} \exp(Q_{\rm soft}^*(s,a')/\lambda)}. \]
Note that the entropy regularization in soft Q-learning can be interpreted as a KL divergence to a uniform distribution, which encourages full support over the action space. In contrast, our method replaces this uniform prior with a behavior policy supported on the safe set so that both value estimation and policy learning are restricted to the safe set. Further details are provided in the next section.

\section{Proposed Approach}
In this paper, we aim to develop a Q-learning variant that learns near-optimal solution without visiting unsafe states. Let $S_{\text{safe}} \subset S$ denote the safe set that satisfy predefined safety constraints. Then, we propose a unified safe Q-learning framework in which all components of learning are restricted to the safe set $S_{\text{safe}}$, as follows:
\begin{itemize}
    \item We first construct a behavior policy supported on $S_{\text{safe}}$.

    \item We develop a safe Bellman target that incorporates this behavior policy, enabling Q-function learning without leaving it.

    \item We then construct the corresponding policy induced by the trained Q-function, which ensures that the resulting policy remains consistent with the behavior policy through regularization.
\end{itemize}
Note that while the overall principle is shared across all settings, the concrete implementation depends on the action space and the type of the behavior policy.

\subsection{Behavior policy}
We propose methods for constructing a stochastic behavior policy $\pi_b$, which is assumed to be supported on the safe set $S_{\text{safe}}$. This assumption is primarily conceptual, as ensuring that $\pi_b$ remains strictly within $S_{\text{safe}}$ may be trivial in some environments but challenging in others. We consider two practical instantiations of such behavior policies: one based on hand-crafted control strategies and the other based on pre-collected datasets. In the remainder of this paper, we refer to these approaches as \textit{HC safe} (hand-crafted) and \textit{DS safe} (dataset-based), respectively.

Unlike demonstration-based methods~\cite{thananjeyan2020safety, turchetta2020safe}, the behavior policy in our framework does not need to be near-optimal; instead, it must be stochastic to enable sufficient exploration within $S_{\text{safe}}$. Moreover, while conventional RL methods explore risky regions during training, our approach restricts learning to trajectories that remain within $S_{\text{safe}}$. Note that, in practice, our formulation allow a small probability of visiting unsafe states while still aiming to minimize them.
    
\subsubsection{Hand-Crafted behavior policy supported on the safe set (HC safe)}
We construct a behavior policy by leveraging well-performing rules or controllers and converting them into stochastic policies. In simple environments with small state and action spaces, unsafe actions can be manually specified for each state. However, in more complex environments, designing the rules manually becomes impractical. In such cases, we instead employ a controller and train a neural network to imitate its behavior. This approach allows us to obtain a stochastic behavior policy and enables adequate exploration. Under the safe set assumption, the resulting trajectories remain within $S_{\text{safe}}$. As a result, the policy can be used for online RL without leaving $S_{\text{safe}}$.

In continuous action spaces, we further introduce two neural network parameterizations of the behavior policy to balance simplicity and expressiveness. The first is \textit{mean-noise behavior policy}~\cite{Lillicrap2015ContinuousCW}, which outputs a deterministic action and injects Gaussian noise during execution:
\begin{equation} \label{mean-noise behavior policy}
    a = \text{clip} \left(\mu_b(s) + w_M, a_{\min}, a_{\max} \right),
\end{equation}
where $\mu_b:{\cal S} \to {\cal A}$ is a deterministic mapping implemented by a neural network and $w_M \sim {\cal N}(0,\sigma_M^2I)$ is a Gaussian noise with $\sigma _M>0$. Here, $\text{clip}(\cdot)$ projects its input onto $[a_{\min}, a_{\max}]$ to enforce valid actions. The second is \textit{distributional behavior policy}~\cite{haarnoja2018soft}, which directly models a stochastic action distribution conditioned on the state:
\begin{equation} \label{distributional behavior policy}
    z \sim {\cal N}(\mu_b(s),\sigma_b^2(s)I), \quad a = \text{tanh(z)},
\end{equation}
where $\mu_b(s)$ and $\sigma_b(s)$ are state-dependent outputs of the neural network. $\tanh(z)$ maps $z$ to $(-1,1)$ to ensure that the resulting action lies within the valid range.

\subsubsection{Dataset-based behavior policy supported on the safe set (DS safe)}
Since HC safe behavior policies are difficult to design in large or continuous action spaces, we introduce an additional approach that leverages an existing dataset collected from the environment. Let the safe dataset $D_\text{safe}$ contain only transitions that do not visit unsafe states. Then, we construct a behavior policy $\pi_b$ by maximizing the log-likelihood of actions in this dataset as stated in~\cref{alg1} of the Appendix. This objective can be efficiently optimized using stochastic gradient ascent. Here, we adopt both the mean-noise behavior policy~\eqref{mean-noise behavior policy} and distributional behavior policy~\eqref{distributional behavior policy} as neural network parameterizations in continuous action spaces.

It should be noted that the resulting policy does not strictly guarantee safety during rollouts, as log-likelihood maximization may assign non-zero probability to states outside $D_\text{safe}$. Therefore, we use it only as a support constraint during the policy update. The policy is trained exclusively on samples from $D_\text{safe}$, which results in an offline RL setting. Further details are provided in the next subsection.

\subsection{Discrete Action Safe-Support Q-Learning}
In this paper, we adopt a two-stage safe-support Q-learning framework in which the Q-function and policy are trained separately. Unlike methods such as SAC~\cite{haarnoja2018soft}, which optimize the Q-function and policy jointly, we fully train the Q-function before proceeding to policy learning. After that, the Q-network parameters are fixed, and only the policy parameters are updated.

We first present a safe Bellman target that regularizes Q-learning toward a behavior policy. Then, we theoretically characterize the resulting solution by showing that the trained Q-function induces a well-defined optimal policy. In particular, we show that the resulting policy can be expressed as a behavior policy-weighted softmax over the trained Q-values in discrete action spaces.

\subsubsection{Target Value Construction}
For Q-learning, limiting exploration to safe states poses a significant challenge. Since the agent never observes unsafe transitions, there are no explicit signals in the Q-function that differentiate between safe and risky behaviors. As a result, unseen or OOD state–action pairs can be incorrectly treated as safe, leading to overestimation of their values. To address this issue, we regularize the Bellman target using the behavior policy supported on a safe set $S_{\text{safe}}$. We introduce a KL divergence term that encourages the Q-function to align with the behavior policy. This prevents the Q-function from assigning high values to actions that are unlikely under the safe data distribution. Notably, this formulation can be interpreted as replacing the entropy term in soft Q-learning~\eqref{eq:SQL_target} with a KL regularization toward the behavior policy.

In this paper, we define the following safe target value:
\begin{equation} \label{eqn:idea_KL}
    y_k^{\rm safe} = r_{k + 1} + {\bf 1}(s_{k + 1})\gamma \max _{\pi \in \Delta_{|{\cal A}|}}\left\{ \sum_{a \in {\cal A}} \pi(a|s_{k + 1})Q(s_{k + 1},a) - \lambda D_{\rm KL} \left( \pi( \cdot|s_{k + 1}) \| \tilde{\pi}_b( \cdot|s_{k + 1}) \right) \right\},
\end{equation}
where $\Delta_{|{\cal A}|}$ denotes the set of all probability distributions over the action space ${\cal A}$, $D_{\rm KL}(\cdot \| \cdot)$ is the KL divergence, $\pi$ is a distribution over actions, and $\lambda>0$ is a weight on the KL divergence term. The smoothed behavior policy is defined as $\tilde{\pi}_b(a|s) = (1-\eta) \pi_b(a|s) + \eta/|{\cal A}|$ with $\eta \in (0,1)$, where $\pi_b$ is a behavior policy supported on $S_{\text{safe}}$. Here, the smoothed policy $\tilde{\pi}_b$ is introduced to ensure strictly positive probabilities, which makes the KL divergence well-defined. Note that~\eqref{eqn:idea_KL} mitigates the influence of OOD actions by penalizing deviations from the behavior policy.

\textbf{Remark.} In practice, this smoothing does not affect safety, as data are collected using the behavior policy $\pi_b$ that remains within $S_{\text{safe}}$. The smoothed policy $\tilde{\pi}_b$ is used in the target computation of the algorithm and does not directly interact with the environment.

As in the entropy-regularized formulation in~\eqref{eq:SQL_target} and~\eqref{eq:SQL_target_closed_form},~\eqref{eqn:idea_KL} can be rewritten in closed form by solving the inner maximization over policies, as shown in the following proposition.
\begin{proposition} \label{prop1}
The target value in~\eqref{eqn:idea_KL} can be equivalently expressed as
\begin{equation} \label{eqn:proposed_y}
    y_k^{\rm safe} = r_{k+1} + {\bf 1}(s_{k+1})\gamma \lambda \ln \left( \sum_{a \in {\cal A}} \tilde{\pi}_b(a|s_{k+1}) \exp\left( \frac{Q(s_{k+1},a)}{\lambda} \right) \right).
\end{equation}
\end{proposition}
The detailed proof is provided in the Appendix~\ref{prop1_proof}. Using the proposed target value, we obtain the following tabular safe-support Q-learning update presented in~\cref{alg3} of the Appendix. For the offline setting, the only modification is that transitions are sampled from the safe dataset instead of interacting with the environment.

Note that the proposed method is closely related to soft Q-learning. If $\tilde{\pi}_b$ is chosen as the uniform distribution over ${\cal A}$, then the KL divergence term reduces to an entropy term as ${D_\text{KL}}(\pi(\cdot|s_{k+1})\|\tilde{\pi}_b(\cdot|s_{k + 1})) = -H(\pi(\cdot |s_{k+1})) + \ln (|{\cal A}|)$, which corresponds to the entropy term plus a constant. Substituting this relation into~\eqref{eqn:idea_KL}, the objective reduces to the entropy-regularized form used in soft Q-learning~\eqref{eq:SQL_target} with an additional constant term. This additional constant only shifts the resulting optimal solution toward all one vector direction, the actual learned policy is unchanged. Therefore, soft Q-learning can be interpreted as a special case of the proposed method when $\tilde{\pi}_b$ is uniform.

The proposed safe Q-learning can be extended to the deep Q-network (DQN)~\cite{mnih2013playing} framework with minimal modifications. In standard DQN, two Q-networks are utilized: the online Q-network $Q_\theta$ parameterized by $\theta$, and the target Q-network $Q_{\theta'}$ parameterized by $\theta'$. The online Q-network is trained by minimizing the following mean-squared Bellman error:
\begin{equation} \label{eq:loss-deep-safe-Q}
    L(\theta;B) := \frac{1}{2|B|} \sum_{(s,a,r,s') \in B} \left( y^{\rm DQN} - Q_\theta(s,a) \right)^2,
\end{equation}
where $B$ is a mini-batch uniformly sampled from the replay buffer $D$, and $y^{\rm DQN}$ is the target value given by $y^{\rm DQN} = r + {\bf 1}(s') \gamma \max_{a \in {\cal A}} Q_{\theta'}(s',a)$. The parameters $\theta$ are updated via gradient descent, $\theta \leftarrow \theta - \alpha \nabla_\theta L(\theta;B)$, while the target parameters $\theta'$ are held fixed and updated periodically using the online parameters. To incorporate safety into this framework, we replace the standard Bellman target $y^{\rm DQN}$ with the proposed safe target $y^{\rm safe, DQN}$:
\begin{equation} \label{eq:target-deep-safe-Q}
    y^{\rm safe, DQN} = r + {\bf 1}(s')\gamma \lambda \ln \left( \sum_{a \in {\cal A}} \tilde{\pi}_b(a|s') \exp \left( \frac{Q_{\theta'}(s',a)}{\lambda} \right) \right).
\end{equation}
This modification biases the target toward actions that are likely under the behavior policy supported on $S_{\text{safe}}$. The overall procedure is summarized in~\cref{alg2} of the Appendix, which is presented for the online setting. For the offline setting, we replace environment interaction with mini-batch sampling from a safe dataset. We first train the DS behavior policy as in~\cref{alg1} of the Appendix, and then update the Q-function following~\cref{alg4} of the Appendix.

\subsubsection{Policy Extraction from Trained Q-functions}
In this subsection, we characterize the policy induced by the trained Q-function. Similar to soft Q-learning, the modified Bellman operator admits a well-defined fixed point, from which the corresponding optimal policy can be derived.
\begin{proposition} \label{prop2}
Consider the Bellman equation $Q = T_\lambda Q$, where the operator $T_\lambda: {\mathbb R}^{|{\cal S}||{\cal A}|} \to {\mathbb R}^{|{\cal S}||{\cal A}|}$ is defined as
\[ (T_\lambda Q)(s,a) := {\mathbb E} \left[ \left. r_{k+1} + \gamma \lambda \ln \left( \sum_{a' \in {\cal A}} {\tilde{\pi}_b(a'|s_{k+1}) \exp \left( \frac{Q(s_{k+1},a')}{\lambda} \right)} \right) \right| s_k=s,a_k=a \right]. \]
Then, $T_\lambda$ admits a unique solution $Q_{\lambda}^*$. Moreover, the corresponding optimal policy is given by 
\[ \pi_\lambda^*(a|s) = \frac{\tilde{\pi}_b(a|s) \exp \left( Q_\lambda^*(s,a) / \lambda \right)}{\sum_{a' \in {\cal A}} \tilde{\pi}_b(a'|s) \exp \left(Q_\lambda^*(s,a') / \lambda \right)}. \]
\end{proposition}
The proof is given in the Appendix~\ref{prop2_proof}. This proof follows from the contraction property of $T_\lambda$, extending the standard arguments used in entropy-regularized RL~\cite{dai2018sbeed,fox2016taming}. Under the standard assumptions, convergence of the proposed Q-learning algorithm can also be established using techniques similar to those in~\cite{jeong2025unified,lee2024unified}; however, we omit the proof in this paper.

This optimal policy $\pi_{\lambda}^*$ admits an interpretation analogous to soft Q-learning. Specifically, it can be viewed as the solution to the following KL-regularized control problem:
\[ \pi_\lambda^* := \arg\max_\pi J_\lambda^\pi, \quad J_\lambda^\pi : = (1-\gamma) {\mathbb E} \left[ \left. \sum_{k = 0}^\infty \gamma^k \left( r_{k+1} - \lambda D_{\rm KL} \left( \pi(\cdot|s_k) \| \tilde{\pi}_b( \cdot|s_k) \right) \right) \right| \pi \right]. \]
This objective can be equivalently written as
\[ J_\lambda^\pi = \sum_{s \in {\cal S}} \sum_{a \in {\cal A}} \rho^\pi(s) \pi(a|s) R(s,a) - \lambda \sum_{s \in {\cal S}} \rho^\pi (s) D_{\rm KL}(\pi(\cdot|s) \| \tilde{\pi}_b(\cdot|s)), \]
where $\rho^\pi(s) := (1-\gamma) \sum_{k=0}^\infty \gamma^k {\mathbb P}(s_k=s | \pi)$ denotes the discounted state visitation distribution, ${\mathbb P}(s_k=s | \pi)$ denotes the probability of visiting state $s$ at $k$ under policy $\pi$, and $R(s,a) := {\mathbb E} \left[ r_{k+1} | s_k=s, a_k=a \right]$ denotes the expected immediate reward. This formulation shows that the learned policy balances reward maximization and deviation from the behavior policy. Moreover, this formulation provides a global objective interpretation of the proposed Bellman update. The KL-regularized objective implicitly enforces safety while enabling controlled exploration through regularization toward the behavior policy.

\subsection{Continuous Action Safe-Support Q-Learning}
In continuous action spaces, we also adopt a two-stage framework in which the Q-function and policy are trained separately. We construct the safe Bellman target using a Monte Carlo approximation of the intractable expectation and train the Q-function accordingly. Then, we characterize the policy induced by the trained Q-function. In contrast to the discrete case, the optimal policy cannot be computed in closed form; therefore, we introduce a surrogate policy and learn it by minimizing a tractable objective.

\subsubsection{Target Value Construction}
We extend the smoothed behavior policy $\tilde{\pi}_b$ from the discrete case by adding small noise to $\pi_b$ to define a continuous distribution over actions. Then, the safe target value in~\eqref{eq:target-deep-safe-Q} becomes 
\begin{align}
    y^{\rm safe} =& r + {\bf 1}(s') \gamma \lambda \ln \left( {\mathbb E}_{a \sim {\tilde{\pi}_b(\cdot|s')}} \left[ \exp \left( \frac{Q_{\theta'}(s',a)}{\lambda } \right) \right] \right) \label{eq:integral} \\
    =& r + {\bf 1}(s') \gamma \lambda \ln \left( \int_{a \in {\cal A}} \tilde{\pi}_b(a|s') \exp \left( \frac{Q_{\theta'}(s',a)}{\lambda} \right) da \right), \nonumber
\end{align}
where the smoothed behavior policy is defined as $\tilde{\pi}_b \sim {\cal N}(\pi_b, \sigma^2 I)$, and $\pi_b$ denotes a deterministic behavior policy over a continuous action space supported on a safe set $S_{\text{safe}}$. However, this generally does not admit a closed-form solution (i.e., no simple analytic form exists) and is therefore intractable to compute exactly. We thus approximate the expectation in~\eqref{eq:integral} using Monte Carlo sampling, following~\cite{haarnoja2017reinforcement}.
\begin{equation} \label{cont target}
    y^{\rm safe} \approx r + {\bf 1}(s') \gamma \lambda \ln \left( \frac{1}{N} \sum_{i=1}^N \exp \left( \frac{Q_{\theta'}(s',a_i)}{\lambda} \right) \right),
\end{equation}
where $\{a_i\}_{i=1}^N$ are i.i.d. samples drawn from the smoothed behavior policy, and $N$ denotes the number of sampled actions.

\subsubsection{Policy Extraction from Trained Q-functions}
According to~\cref{prop2}, the optimal policy induced by the trained Q-function is given by 
\[ \pi_\lambda^*(a|s) \cong \frac{\tilde{\pi}_b(a|s) \exp \left( \frac{Q_\theta(s,a)}{\lambda} \right)}{\int_{u \in {\cal A}} \tilde{\pi}_b(u|s) \exp \left( \frac{Q_\theta(s,u)}{\lambda} \right) du}. \]
However, the integral in the denominator is generally intractable in continuous action spaces, making direct evaluation of $\pi_\lambda^*$ impractical. To address this issue, we introduce a surrogate policy $\pi_\phi$ and obtain it by minimizing the KL divergence to the optimal policy:
\begin{equation} \label{cont policy obj1}
\begin{aligned}
    L(\phi) = {\mathbb E}_{s \sim {\cal N}(0,\sigma_1^2 I)} \left[ D_{\rm KL} \left( \pi_\phi(\cdot|s) \left\| \frac{\tilde{\pi}_b(\cdot|s) \exp \left( \frac{Q_\theta(s,\cdot)}{\lambda} \right)}{\int_{u \in {\cal A}} \tilde{\pi}_b(u|s) \exp \left( \frac{Q_\theta(s,u)}{\lambda} \right) du} \right. \right) \right],
\end{aligned}
\end{equation}
where ${\cal N}(0,\sigma_1^2 I)$ denotes a Gaussian distribution with $\sigma_1^2 > 0$. Note that since the objective depends only on the state distribution, we adopt a Gaussian distribution as a simple reference distribution instead of the replay buffer.

To simplify the objective~\eqref{cont policy obj1}, we parameterize $\pi_\phi$ as a Gaussian policy:
\[ \pi_\phi(a|s) = {\cal N}(a;\mu_\phi(s),\sigma_2^2 I), \]
where $\mu_\phi(s)$ and $\sigma_2 > 0$ denote the mean and standard deviation, respectively. Moreover, we consider two cases depending on the form of the behavior policy: mean-noise behavior policy~\eqref{mean-noise behavior policy} and distributional behavior policy~\eqref{distributional behavior policy}. Then, the approximate objective for the former case can be written as
\begin{equation} \label{cont policy obj_case1_fin}
    \tilde L_M(\phi) := \sum_{i=1}^p \sum_{j=1}^q \left[ \frac{1}{2\sigma_M^2} \left\| \mu_\phi(s_i) + w_j - \mu_b(s_i) \right\|_2^2 - \frac{Q_\theta(s_i,\mu_\phi(s_i)+w_j)}{\lambda} \right],
\end{equation}
where $s_i \sim {\cal N}(0,\sigma_1^2 I)$ and $w_j \sim {\cal N}(0,\sigma_2^2 I)$. For the latter case, the approximate objective is given by
\begin{equation} \label{cont policy obj_case2_fin}
\tilde L_D(\phi) := \sum_{i=1}^p \sum_{j=1}^q \left[ -\ln \tilde{\pi}_b(\mu_\phi(s_i) + w_j|s_i) -\frac{Q_\theta(s_i,\mu_\phi(s_i) + w_j)}{\lambda} \right].
\end{equation}
The detailed derivation is provided in the Appendix~\ref{cont pol obj}. Furthermore, the overall procedures are summarized in~\cref{alg5} and~\cref{alg6} of the Appendix, respectively, which are presented for the online setting.

\subsection{Unified Algorithmic Variants}
The proposed framework can be instantiated in several algorithmic variants depending on the action space and the type of behavior policy, as summarized in~\cref{tab1}. All these variants share the same principle: the Q-function is trained using a safe Bellman target, and the policy is subsequently constructed from the trained Q-function while remaining consistent with the behavior policy. The differences arise only from the choice of action space and the parameterization of the behavior policy, which affect the form of the target value and the policy learning objective. Note that we employ an $\epsilon$-greedy behavior policy in discrete action spaces, which is a standard choice for exploration. This safe policy is also restricted to safe actions in our setting so that the resulting trajectories remain within the safe set $S_{\text{safe}}$.

\begin{table}[ht]
\centering
\begin{tabular}{|P{50pt}|P{70pt}|P{140pt}|}
\hline
Case & Action Space & Behavior Policy \\
\hline
Case 1 & Discrete & HC safe + $\epsilon$-greedy \\
\hline
Case 2 & Discrete & DS safe + $\epsilon$-greedy \\
\hline
Case 3 & Continuous & HC safe + mean-noise \\
\hline
Case 4 & Continuous & HC safe + distributional \\
\hline
Case 5 & Continuous & DS safe + mean-noise \\
\hline
Case 6 & Continuous & DS safe + distributional \\
\hline
\end{tabular}
\caption{Unified variants of the proposed framework}
\label{tab1}
\end{table}

\section{Experiment}
In this section, we evaluate the proposed methods by addressing the following questions:
\begin{enumerate}
    \item Can the proposed method improve performance when learning only from a behavior policy supported on the safe set?
    
    \item Does the proposed method yield less biased Q-value estimates compared to baseline algorithms?

    \item Can the proposed method learn a safer policy at comparable return levels under a behavior policy supported on the safe set?
    
\end{enumerate}
We conduct experiments on the FrozenLake and CartPole environments, both of which have well-established control strategies or known optimal behaviors. Details of the experimental settings are provided in the Appendix~\ref{exp setting}. Note that our proposed framework can be readily extended to other environments. For each result graph, performance is averaged over five different seeds; solid curves represent the mean, and shaded regions indicate the standard deviation.

\subsection{FrozenLake Gym}

\begin{figure}[ht]
    \centering
    \begin{subfigure}{0.4\linewidth}
    \includegraphics[width=\linewidth]{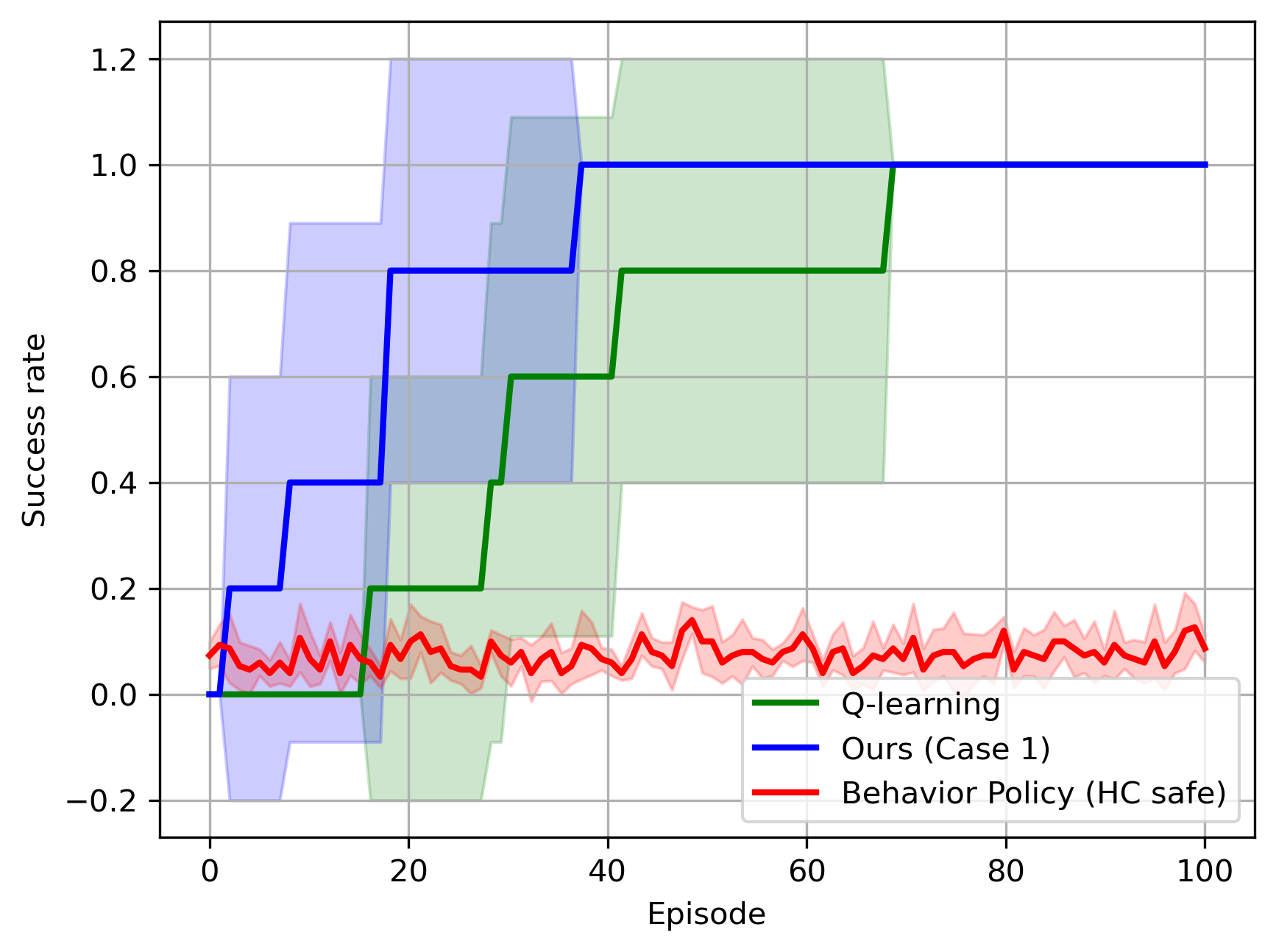}
    \caption{Case 1: Discrete, HC safe}
    \label{fig2a}
    \end{subfigure}
    \begin{subfigure}{0.4\linewidth}
    \includegraphics[width=\linewidth]{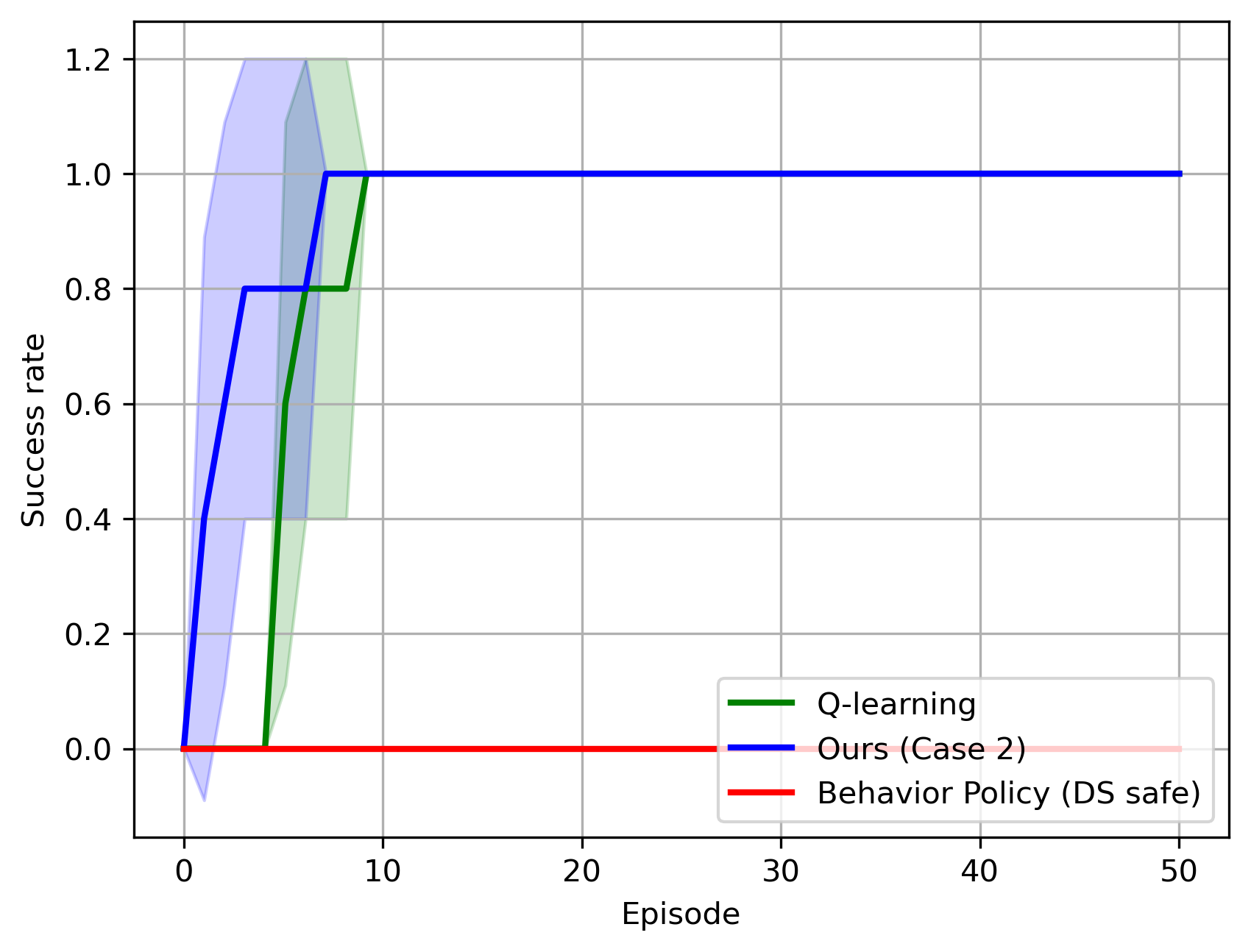}
    \caption{Case 2: Discrete, DS safe}
    \label{fig2b}
    \end{subfigure}
    \caption{Learning curve of success rate on FrozenLake-v1 comparing the proposed method with the standard Q-learning under different settings.}
\end{figure}

To answer the first question, we compare the success rate learning curves of the proposed algorithm and standard Q-learning, as shown in~\cref{fig2a,fig2b}. To train the standard Q-learning algorithm, we use the same behavior policy as in the proposed algorithm, described in the Appendix~\ref{exp setting}. Since the behavior policy is not updated during training, its performance remains suboptimal. In contrast, both methods learn policies that outperform the behavior policy. Specifically, the proposed algorithm converges faster and exhibits lower variance, indicating more stable learning. These results demonstrate that effective policy learning is achievable even when learning from a behavior policy supported on the safe set.

\begin{figure}[ht]
    \centering
    \begin{subfigure}{0.49\linewidth}
    \includegraphics[width=\linewidth]{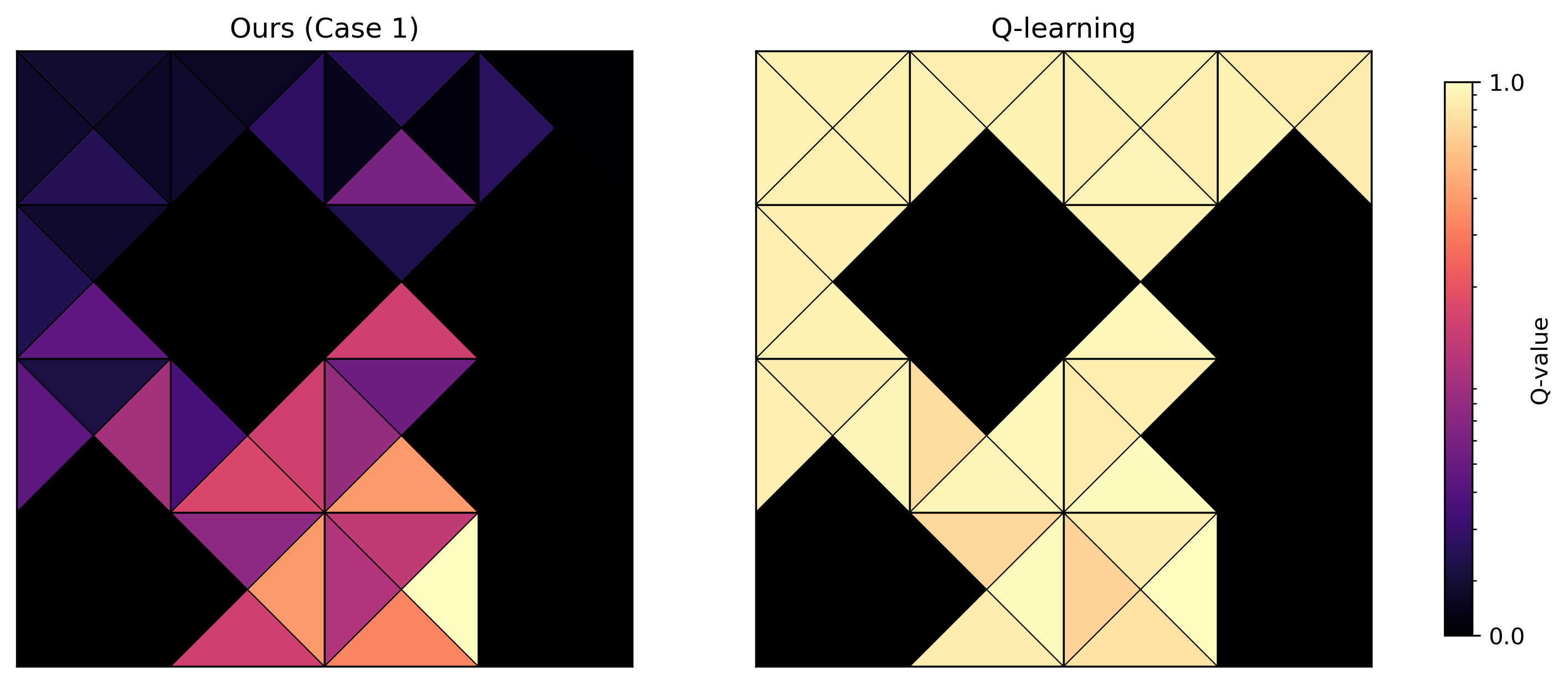}
    \caption{Comparison of trained Q-values between the proposed method (Case 1: Discrete, HC safe; left) and standard Q-learning (right)}
    \label{fig3a}
    \end{subfigure}
    \begin{subfigure}{0.49\linewidth}
    \includegraphics[width=\linewidth]{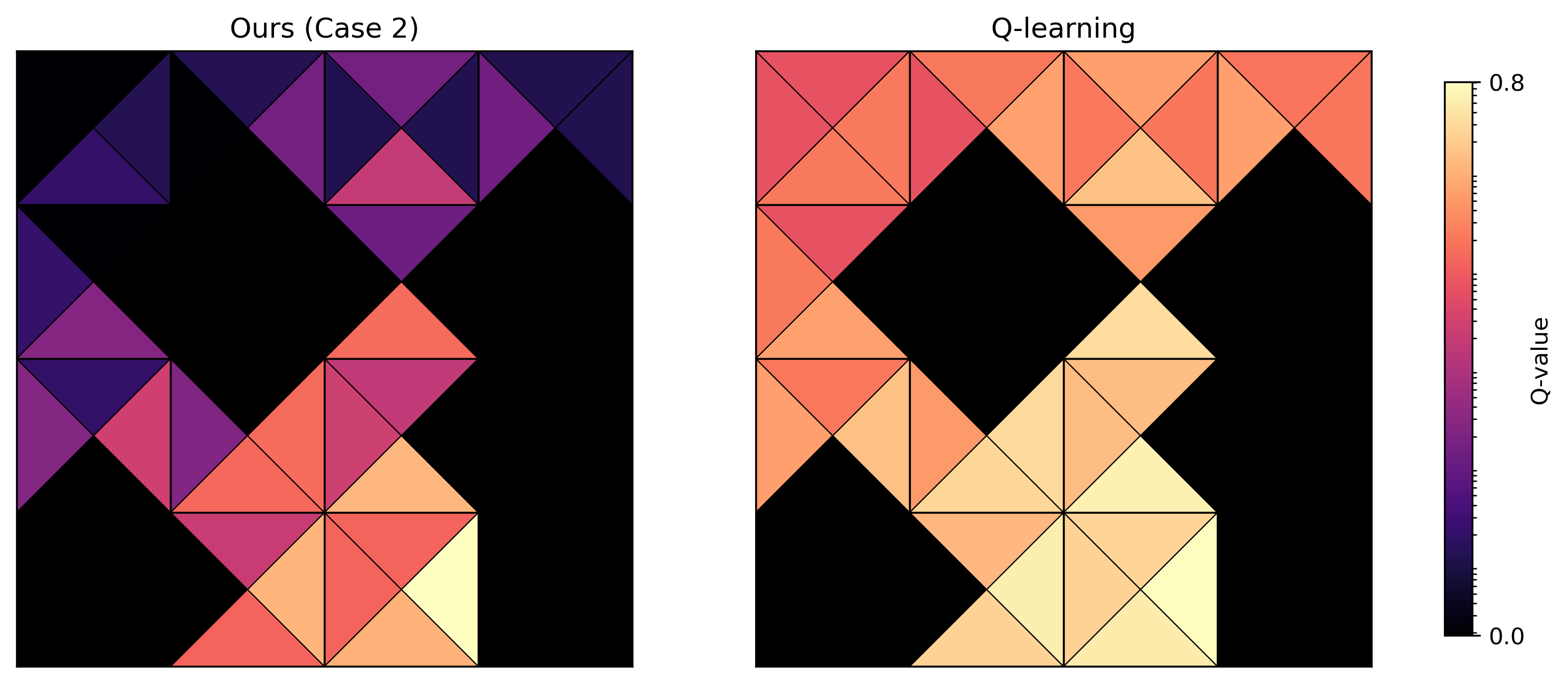}
    \caption{Comparison of trained Q-values between the proposed method (Case 2: Discrete, DS safe; left) and standard Q-learning (right)}
    \label{fig3b}
    \end{subfigure}
    \caption{Comparison of trained Q-values on FrozenLake-v1 under different settings.}
\end{figure}

To answer the second question, we compare the trained Q-values of the proposed algorithm and standard Q-learning as shown in~\cref{fig3a,fig3b}. Each figure visualizes the Q-values for the 16 states as illustrated in~\cref{fig1a} of the Appendix, with each square representing a state and each triangle corresponding to an action (up, down, left, and right). Since both methods are trained without visiting unsafe states, the Q-values for unsafe state-action pairs remain zero, as shown in black. Compared to standard Q-learning, our method learns more structured Q-values that emphasize safe trajectories and suppress values near hazardous states. In particular, Q-values of our method increase smoothly along safe paths toward the goal, while standard Q-learning exhibits relatively uniform and overestimated values near unsafe states. These results indicate that the proposed method captures safety-aware value estimates and achieves improved performance through more accurate value estimation.

\subsection{CartPole Gym}
In this environment, we consider DQN, CQL, BCQ, IQL, TD3+BC, and BEAR as baselines described in the Appendix~\ref{exp setting} for a fair comparison under our behavior policy setting. Note that TD3+BC and BEAR are only applicable to continuous action spaces.

\begin{figure}[ht]
    \centering
    \begin{subfigure}{0.8\linewidth}
    \includegraphics[width=\linewidth]{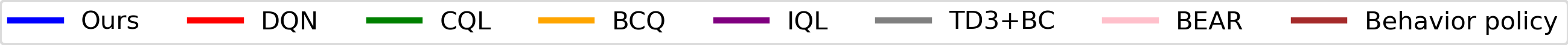}
    \end{subfigure}
    
    \begin{subfigure}{0.32\linewidth}
    \includegraphics[width=\linewidth]{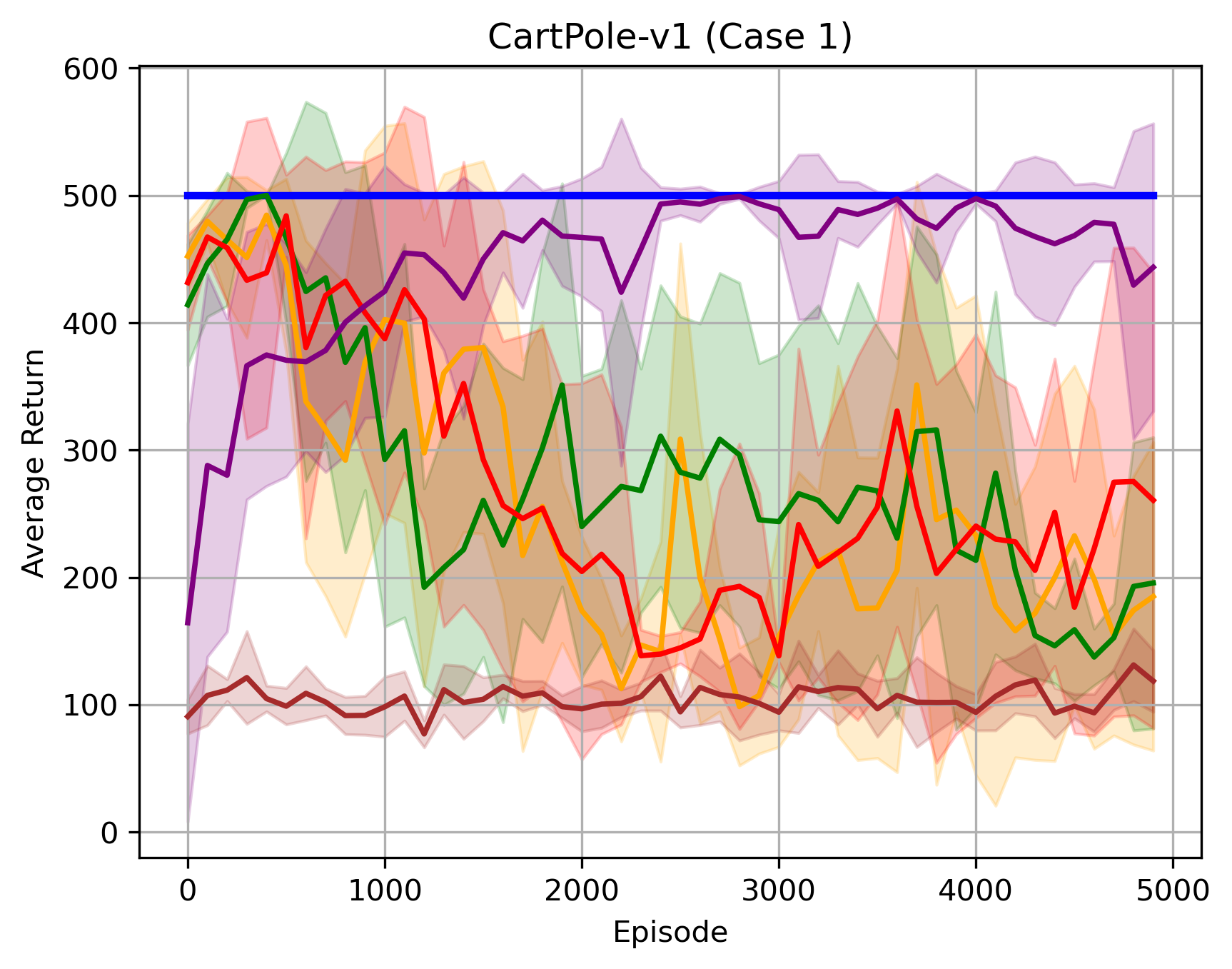}
    \caption{\centering Case 1:\\Discrete, HC safe}
    \label{fig4a}
    \end{subfigure}
    \begin{subfigure}{0.32\linewidth}
    \includegraphics[width=\linewidth]{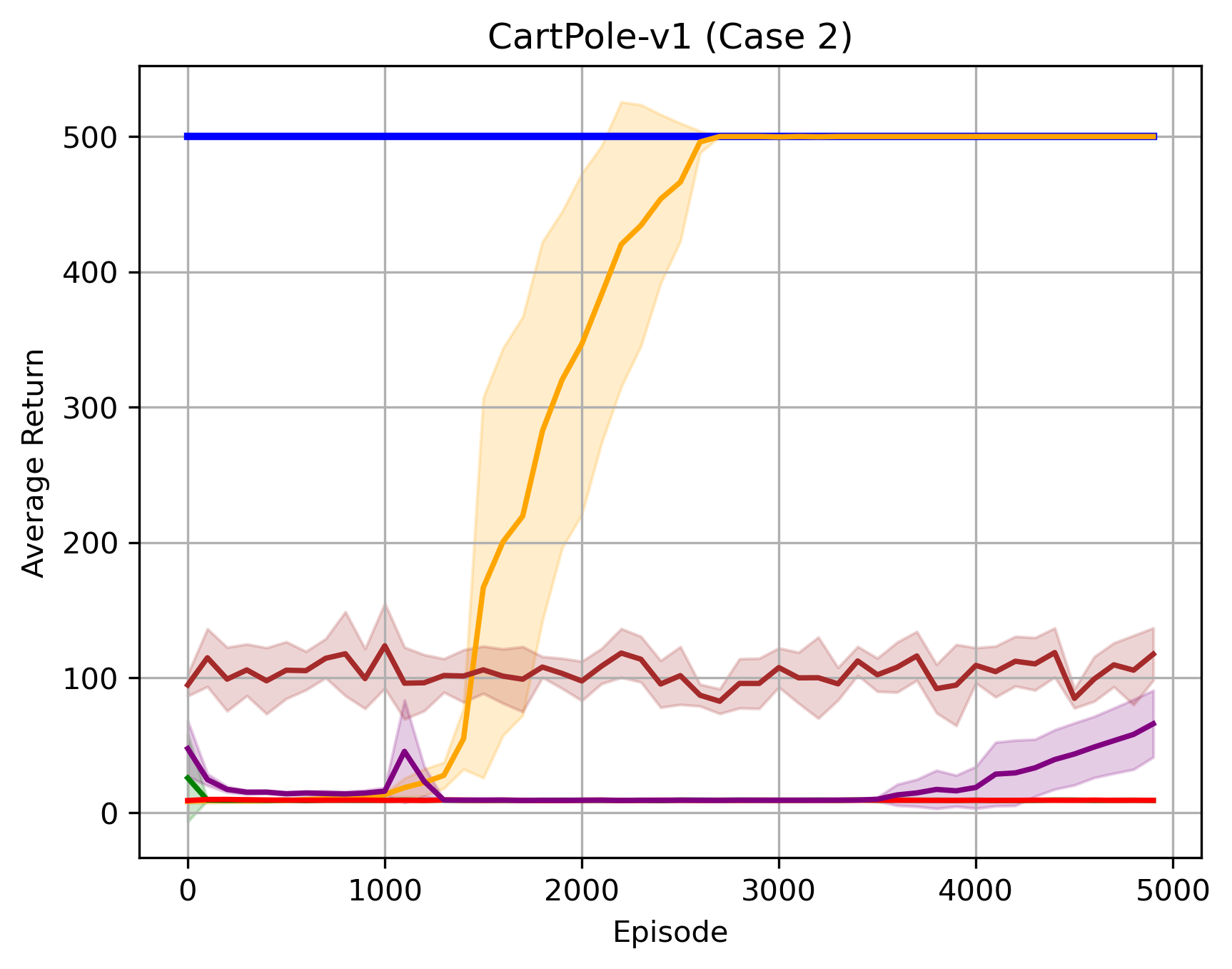}
    \caption{\centering Case 2:\\Discrete, DS safe}
    \label{fig4b}
    \end{subfigure}
    \begin{subfigure}{0.32\linewidth}
    \includegraphics[width=\linewidth]{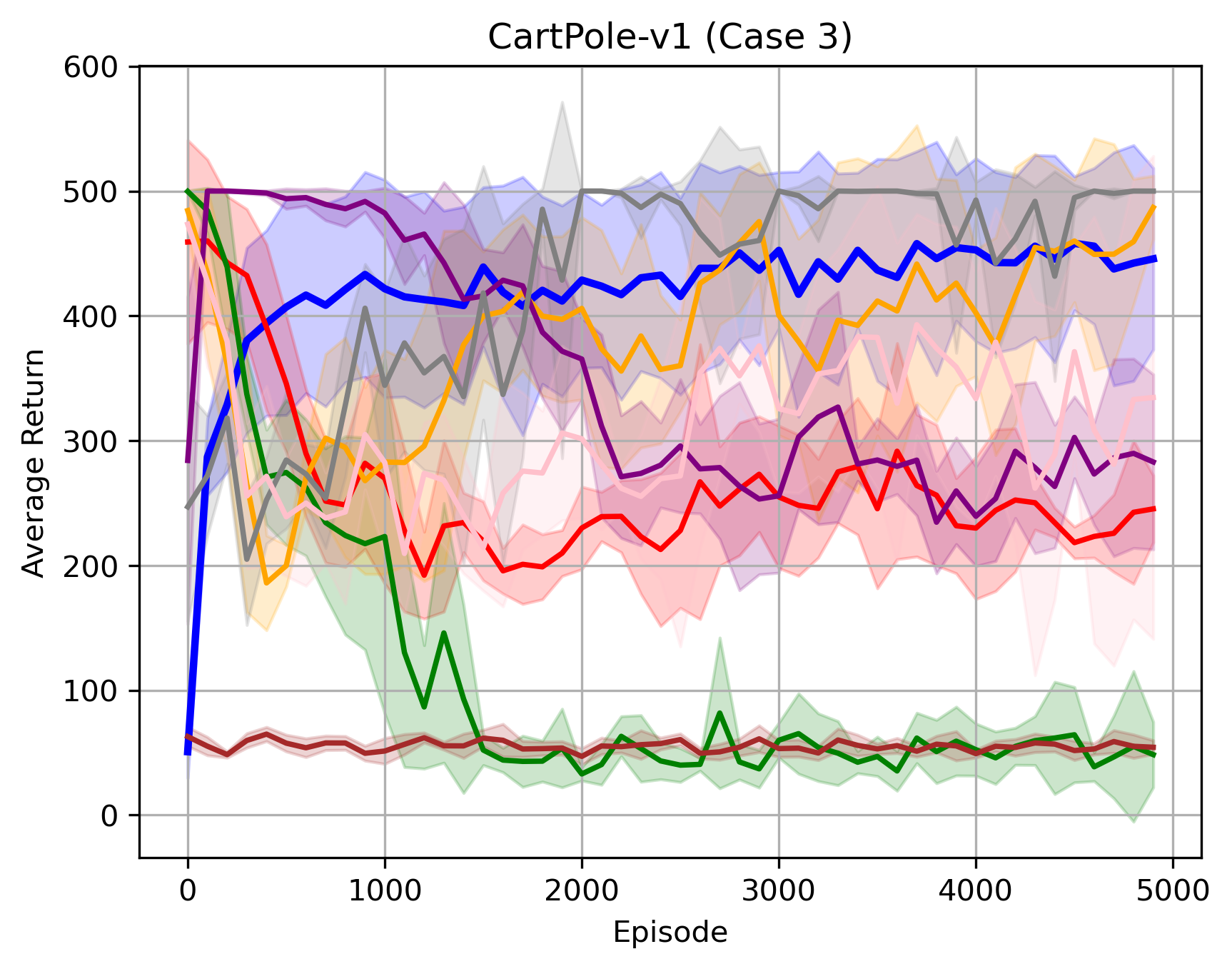}
    \caption{\centering Case 3:\\Continuous, HC safe, mean-noise}
    \label{fig4c}
    \end{subfigure}
    
    \begin{subfigure}{0.32\linewidth}
    \includegraphics[width=\linewidth]{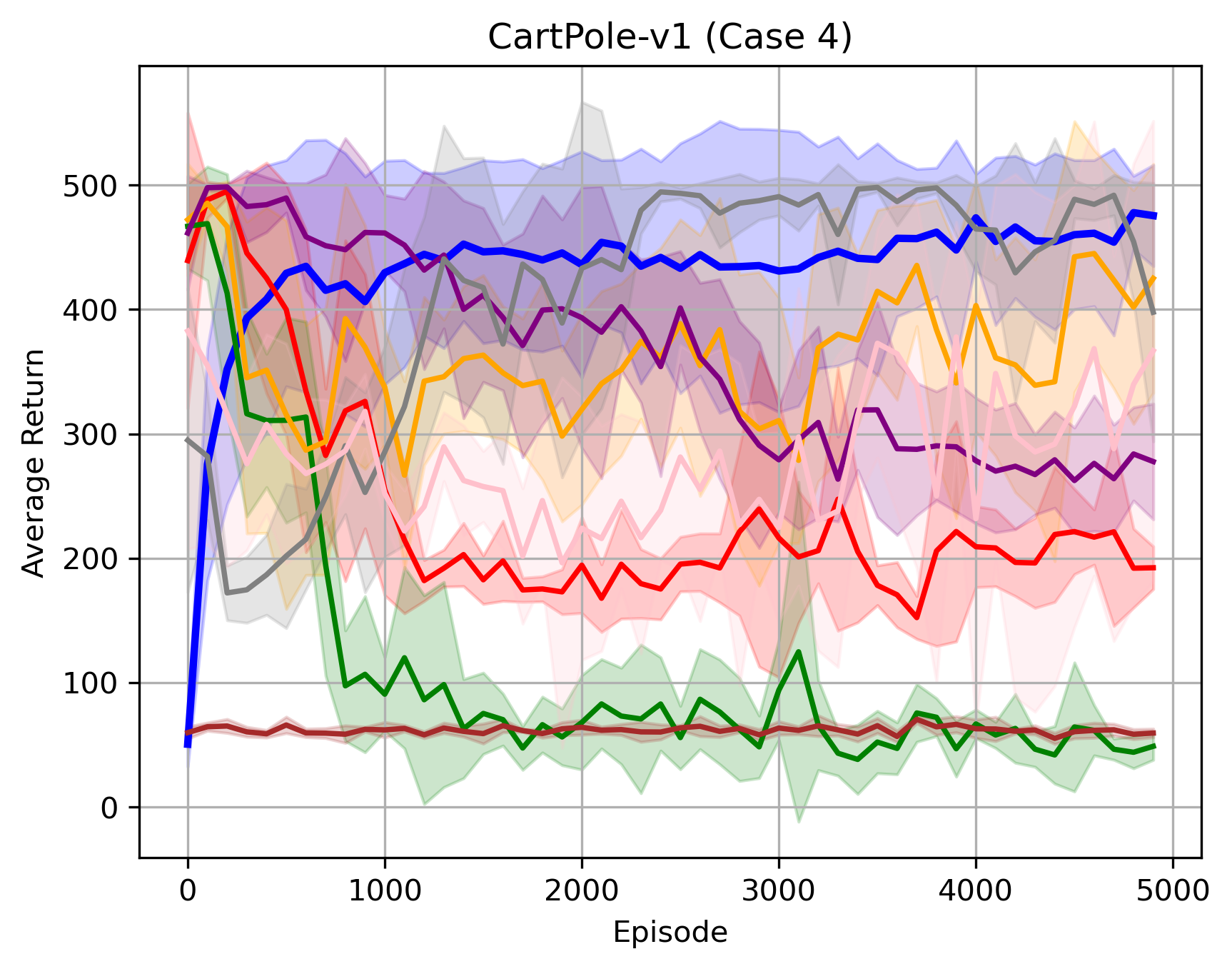}
    \caption{\centering Case 4:\\Continuous, HC safe, distributional}
    \label{fig4d}
    \end{subfigure}
    \begin{subfigure}{0.32\linewidth}
    \includegraphics[width=\linewidth]{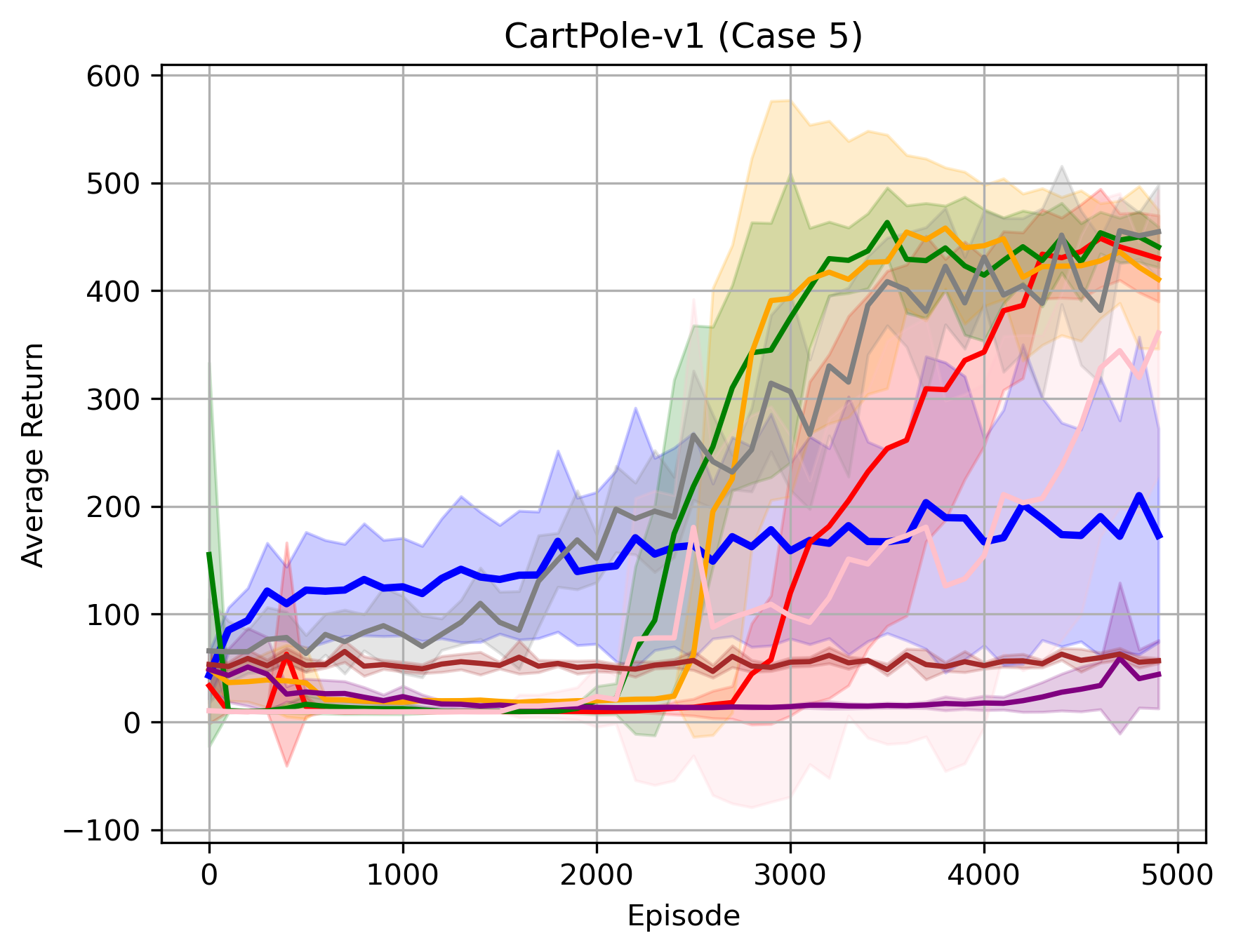}
    \caption{\centering Case 5:\\Continuous, DS safe, mean-noise}
    \label{fig4e}
    \end{subfigure}
    \begin{subfigure}{0.32\linewidth}
    \includegraphics[width=\linewidth]{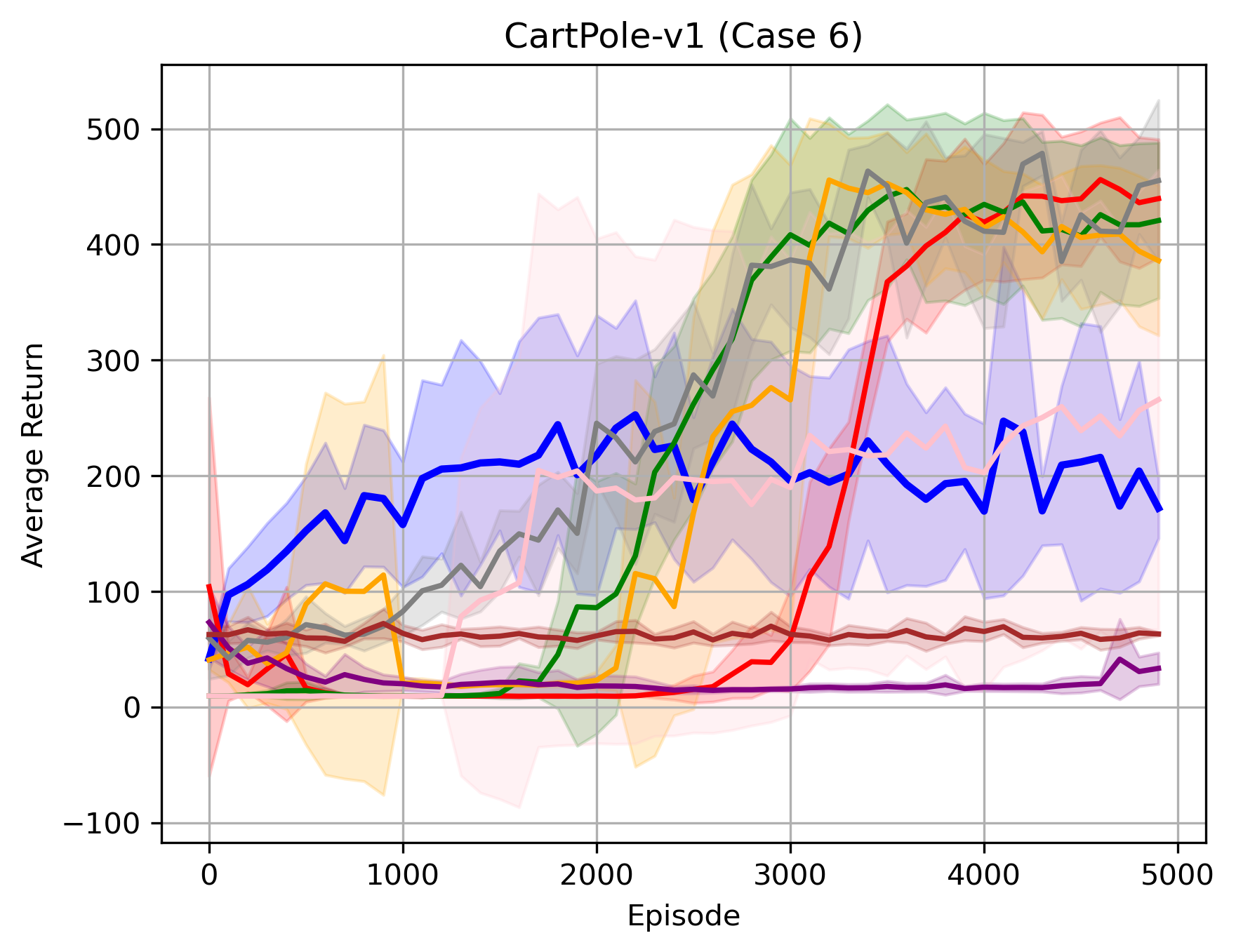}
    \caption{\centering Case 6:\\Continuous, DS safe, distributional}
    \label{fig4f}
    \end{subfigure}

    \caption{Learning curves on CartPole-v1 comparing the proposed method with baselines under different settings.}
\end{figure}

To answer the first question, we evaluate the performance of the proposed methods against baselines across different settings. To train the baselines, we use the same behavior policy as in the proposed algorithm, described in the Appendix~\ref{exp setting}. In the discrete setting illustrated in~\cref{fig4a,fig4b}, most baselines fail to learn effective policies or exhibit unstable performance. In contrast, our method consistently achieves optimal performance, showing that the optimal policy can be recovered even from suboptimal behavior policies supported on the safe set. In the continuous setting illustrated in~\cref{fig4c,fig4d}, some baselines still show instability, likely due to limited action diversity. On the other hand, the proposed method demonstrates stable learning with smoother convergence, suggesting robust policy learning under a behavior policy supported on the safe set. When the setting becomes more restricted, as illustrated in~\cref{fig4e,fig4f}, however, the proposed method shows limited performance due to the imperfect behavior policy estimation. Unlike some baselines that benefit from extrapolation, our method relies on data support, which may limit its ability to discover optimal strategies. This reflects a trade-off between safety and performance under imperfect behavior policies.

\begin{figure}[ht]
    \centering
    \begin{subfigure}{0.8\linewidth}
    \includegraphics[width=\linewidth]{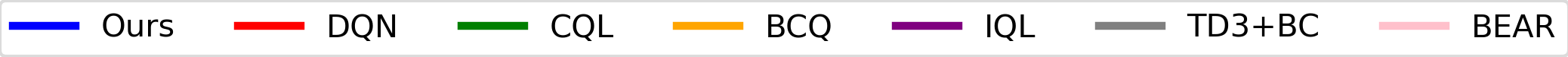}
    \end{subfigure}
    
    \begin{subfigure}{0.32\linewidth}
    \includegraphics[width=\linewidth]{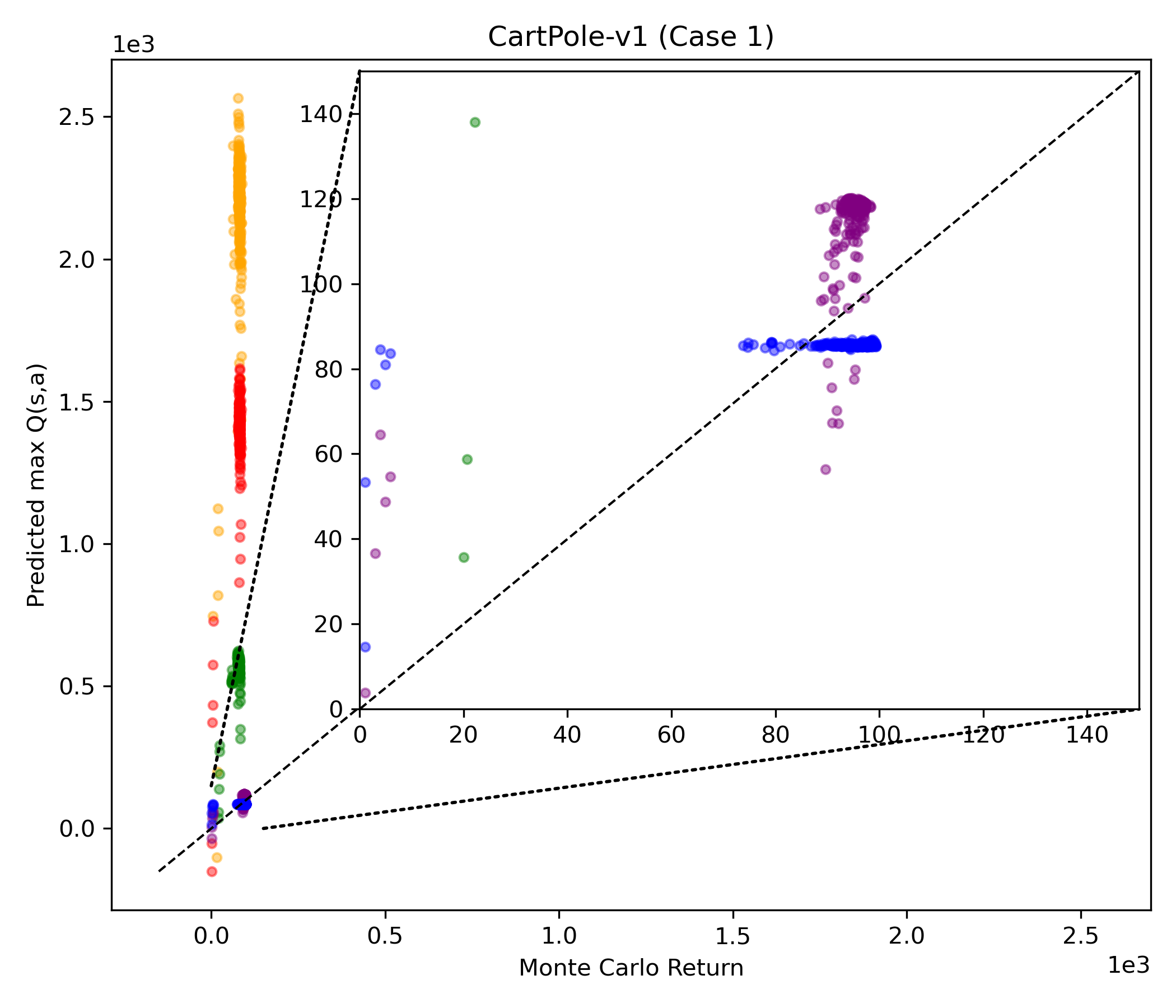}
    \caption{\centering Case 1:\\Discrete, HC safe}
    \label{fig5a}
    \end{subfigure}
    \begin{subfigure}{0.32\linewidth}
    \includegraphics[width=\linewidth]{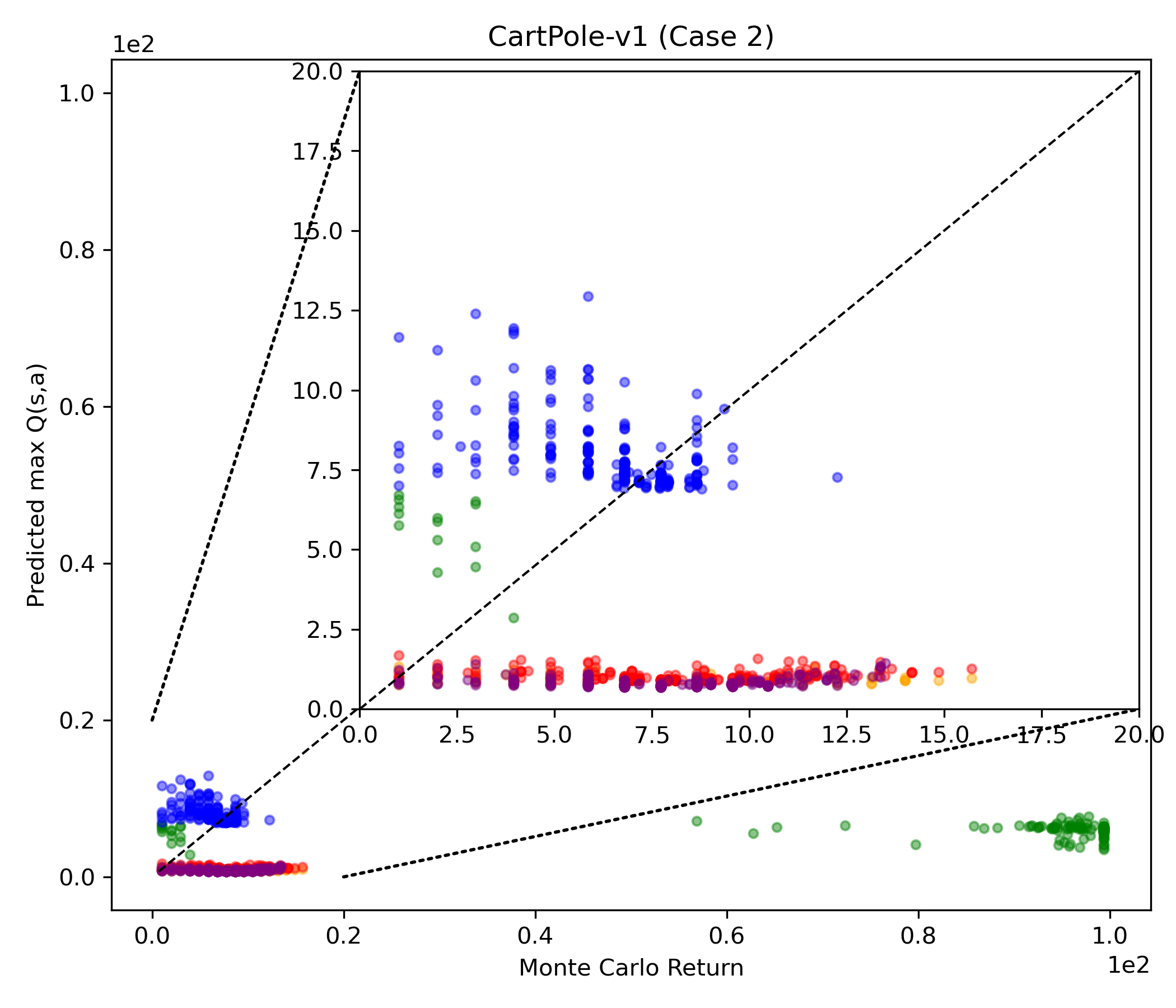}
    \caption{\centering Case 2:\\Discrete, DS safe}
    \end{subfigure}
    \begin{subfigure}{0.32\linewidth}
    \includegraphics[width=\linewidth]{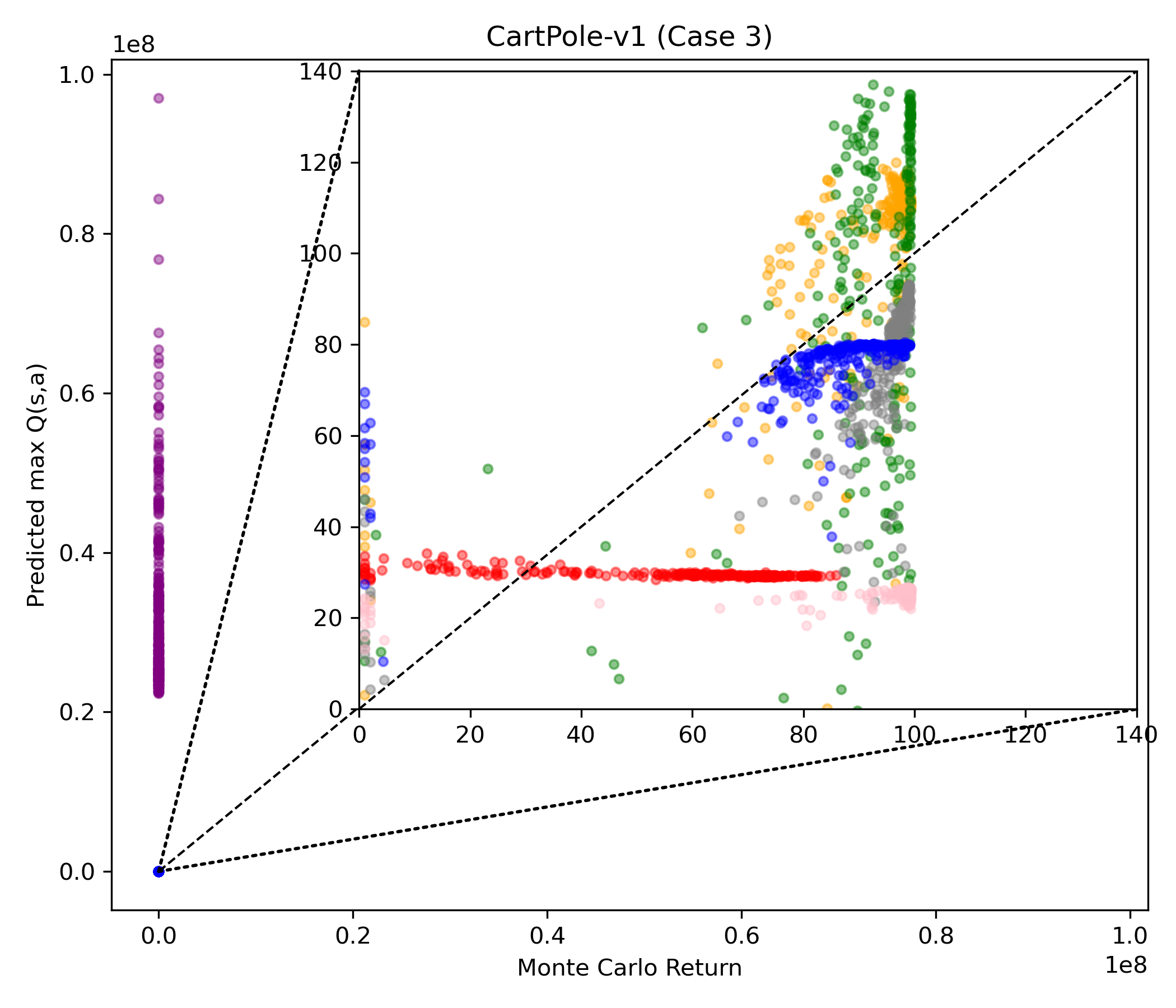}
    \caption{\centering Case 3:\\Continuous, HC safe, mean-noise}
    \label{fig5c}
    \end{subfigure}
    
    \begin{subfigure}{0.32\linewidth}
    \includegraphics[width=\linewidth]{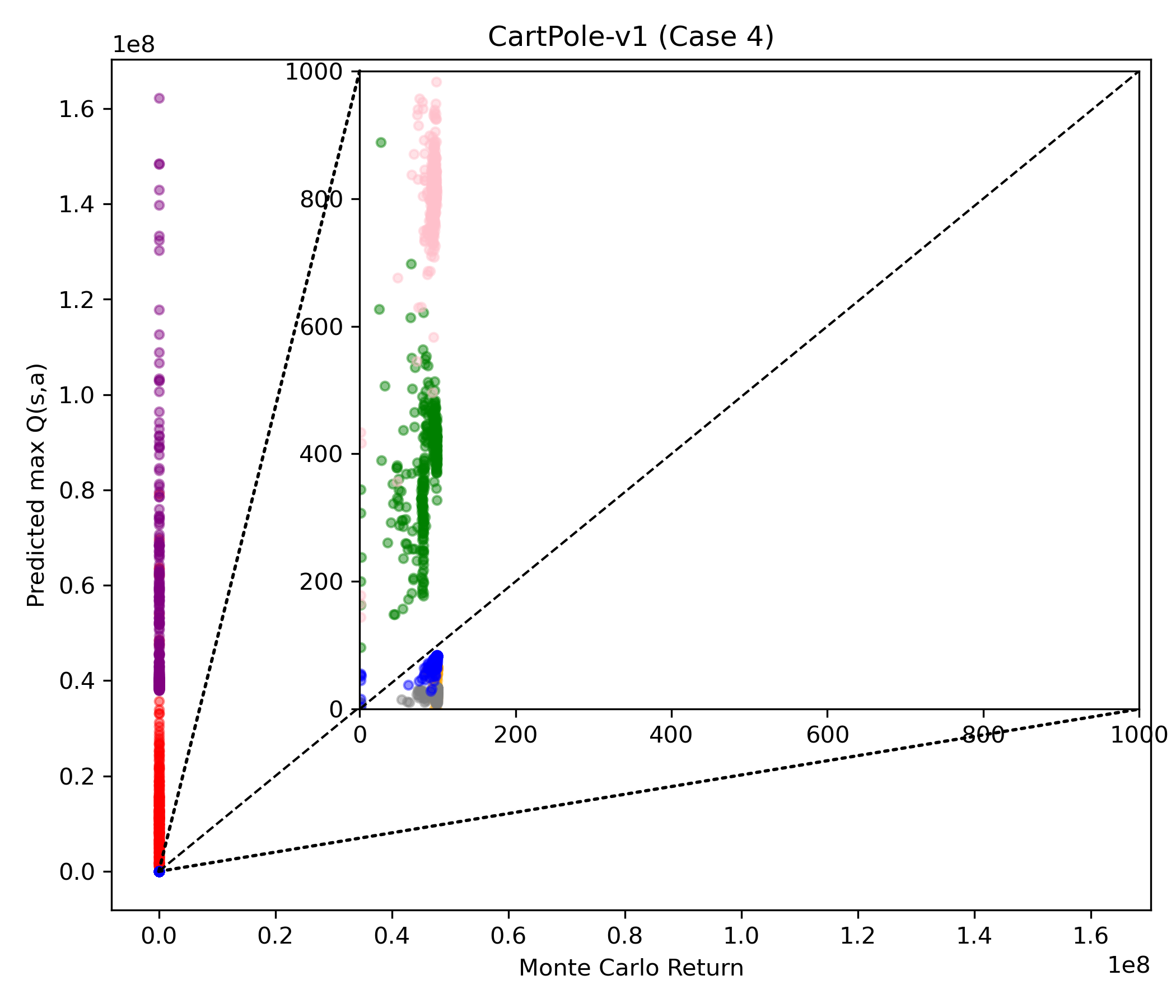}
    \caption{\centering Case 4:\\Continuous, HC safe, distributional}
    \label{fig5d}
    \end{subfigure}
    \begin{subfigure}{0.32\linewidth}
    \includegraphics[width=\linewidth]{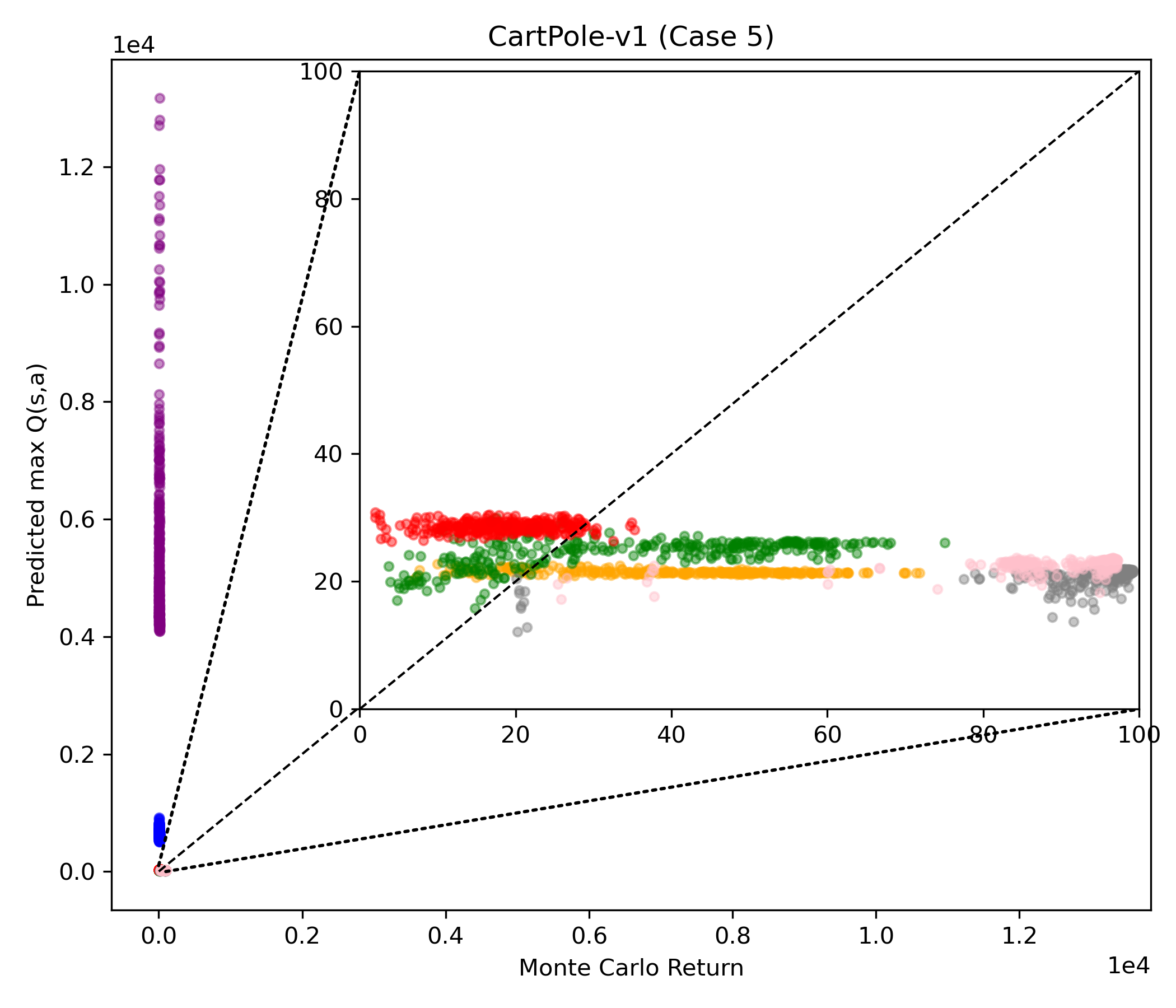}
    \caption{\centering Case 5:\\Continuous, DS safe, mean-noise}
    \label{fig5e}
    \end{subfigure}
    \begin{subfigure}{0.32\linewidth}
    \includegraphics[width=\linewidth]{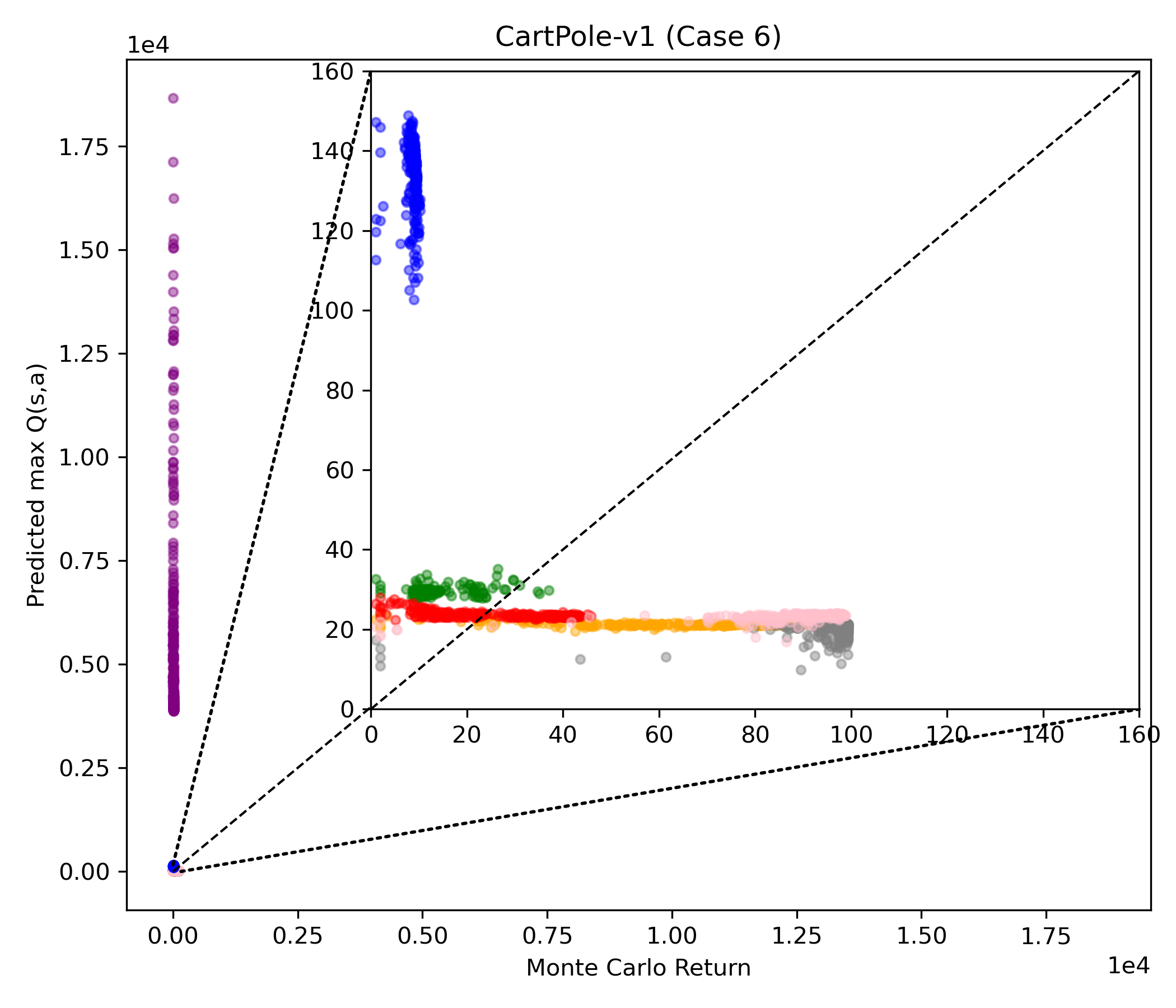}
    \caption{\centering Case 6:\\Continuous, DS safe, distributional}
    \label{fig5f}
    \end{subfigure}

    \caption{Q-value estimation bias on CartPole-v1 comparing the proposed method with baselines under different settings.}
    \label{fig5}
\end{figure}

For the second question, we assess estimation bias by comparing $\max_a Q(s,a)$ with the Monte Carlo return obtained by rolling out the corresponding greedy policy. As shown in~\cref{fig5}, well-calibrated Q-values should lie along the diagonal (dotted line), with points above and below the line indicating overestimation and underestimation, respectively. In~\cref{fig5a} to~\cref{fig5d}, the proposed method produces Q-values closely aligned with the Monte Carlo returns, which indicates well-calibrated value estimates. In contrast, baseline methods exhibit various issues, including value collapse (nearly constant Q-values), overly conservative estimation (below the diagonal), or divergence (above the diagonal). These results suggest that inaccurate Q-value estimation can lead to either policy collapse or overly conservative behavior. In particular, some baselines display widely scattered points as shown in~\cref{fig5c}. This reflects high variance and inconsistent Q-value estimation. Even when similar returns are achieved (e.g., BCQ and TD3+BC in~\cref{fig5c,fig5d}), their Q-values are more dispersed and deviate further from the diagonal. This suggests that their value functions are less stable and inaccurately calibrated, whereas the proposed method maintains low-variance and more reliable value estimation. However, in~\cref{fig5e,fig5f}, the proposed method no longer maintains well-calibrated Q-values. This is likely due to imperfect behavior policy estimation.

\begin{figure}[ht]
    \centering
    \begin{subfigure}{0.8\linewidth}
    \includegraphics[width=\linewidth]{figure/experiment_cartpole/legend_overestimation.png}
    \end{subfigure}
    
    \begin{subfigure}{0.32\linewidth}
    \includegraphics[width=\linewidth]{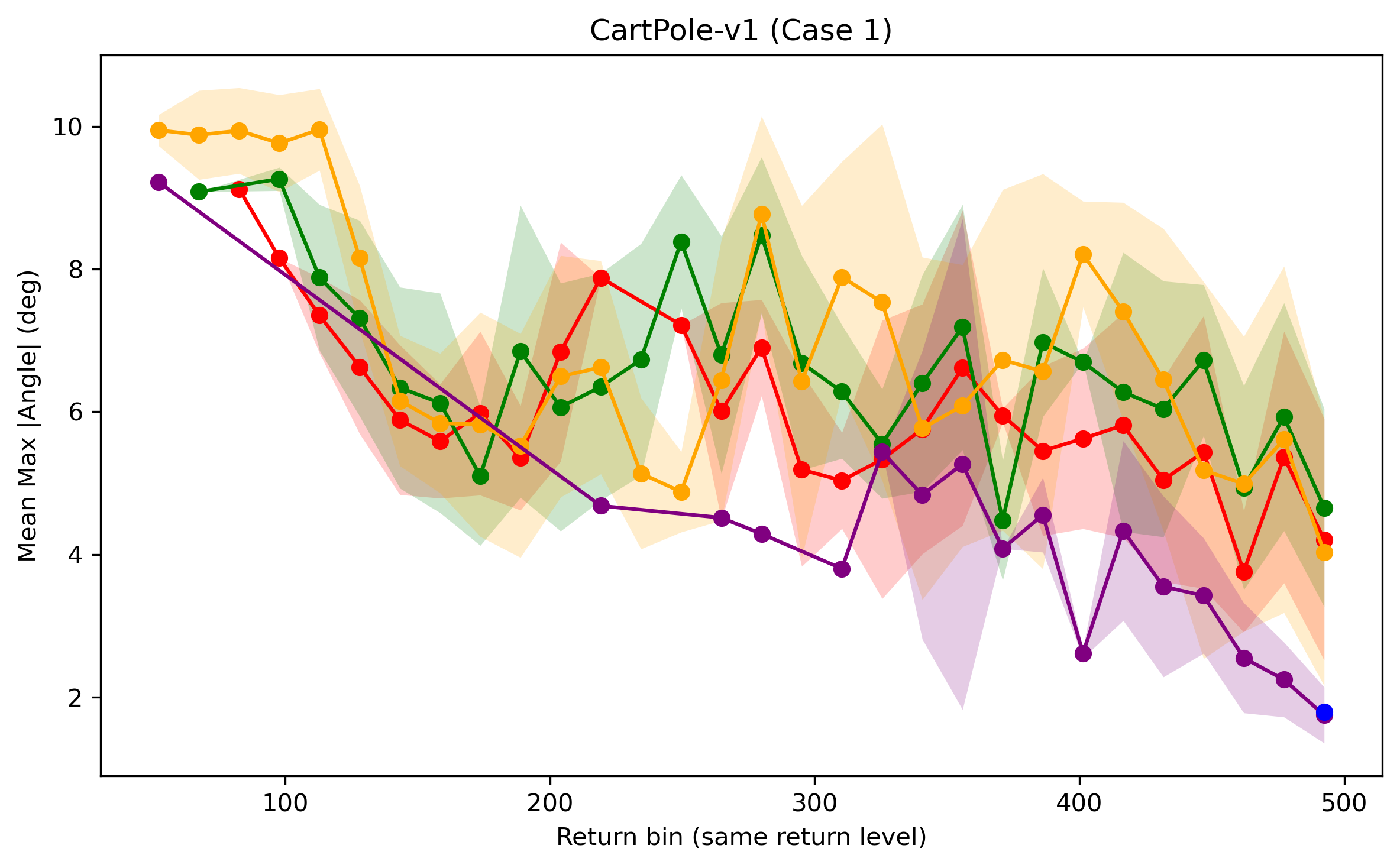}
    \caption{\centering Case 1:\\Discrete, HC safe}
    \end{subfigure}
    \begin{subfigure}{0.32\linewidth}
    \includegraphics[width=\linewidth]{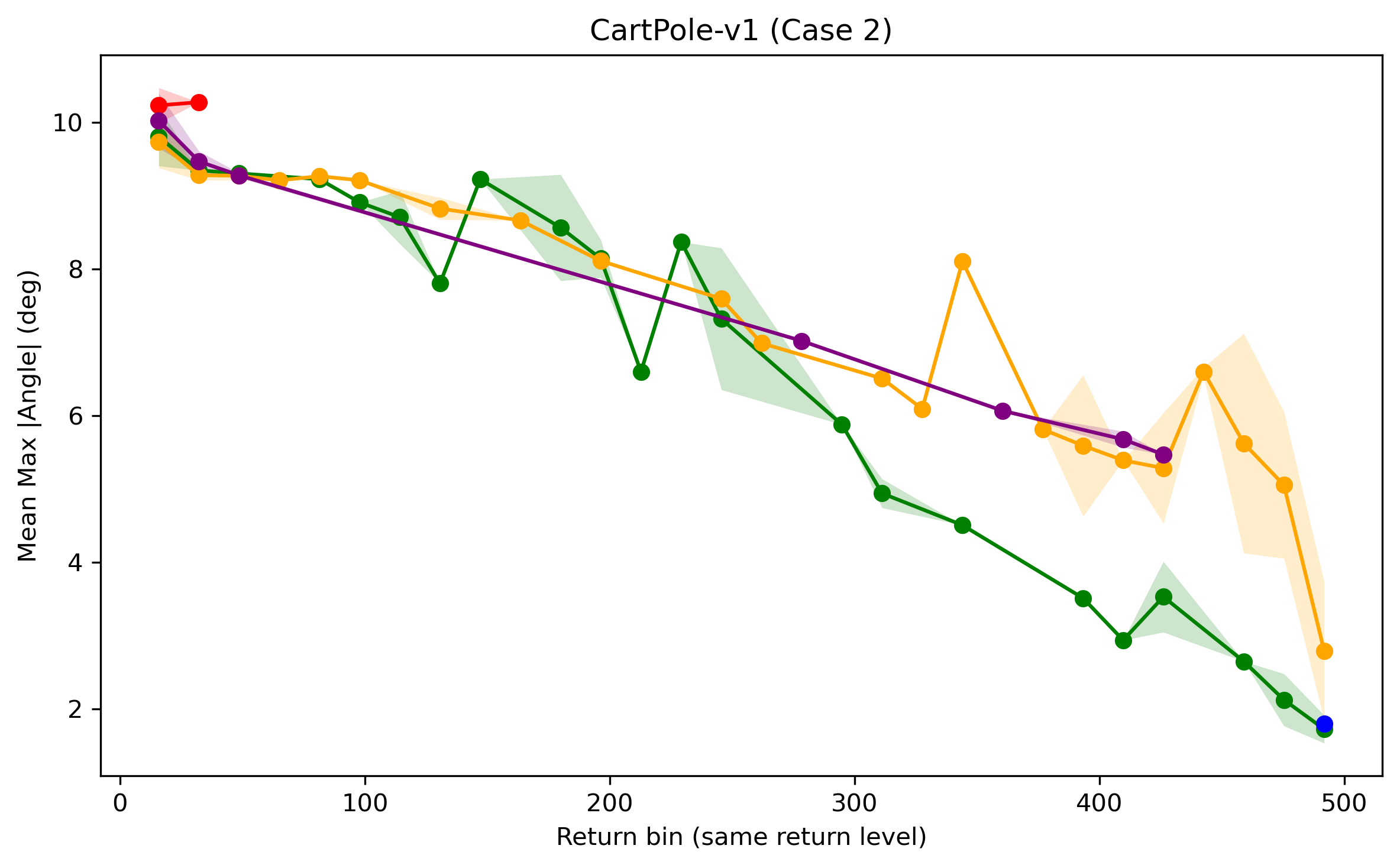}
    \caption{\centering Case 2:\\Discrete, DS safe}
    \label{fig6b}
    \end{subfigure}
    \begin{subfigure}{0.32\linewidth}
    \includegraphics[width=\linewidth]{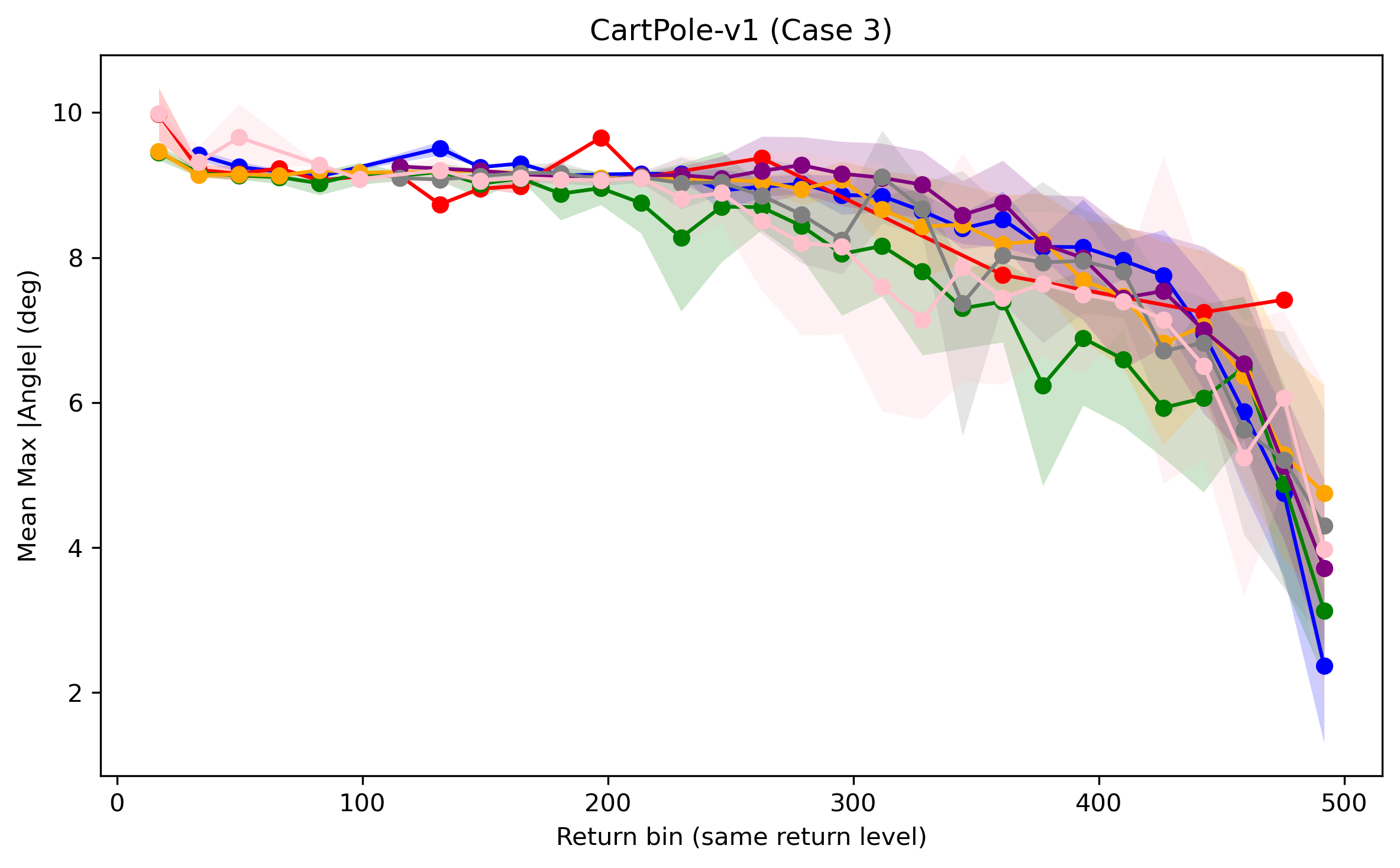}
    \caption{\centering Case 3:\\Continuous, HC safe, mean-noise}
    \label{fig6c}
    \end{subfigure}
    
    \begin{subfigure}{0.32\linewidth}
    \includegraphics[width=\linewidth]{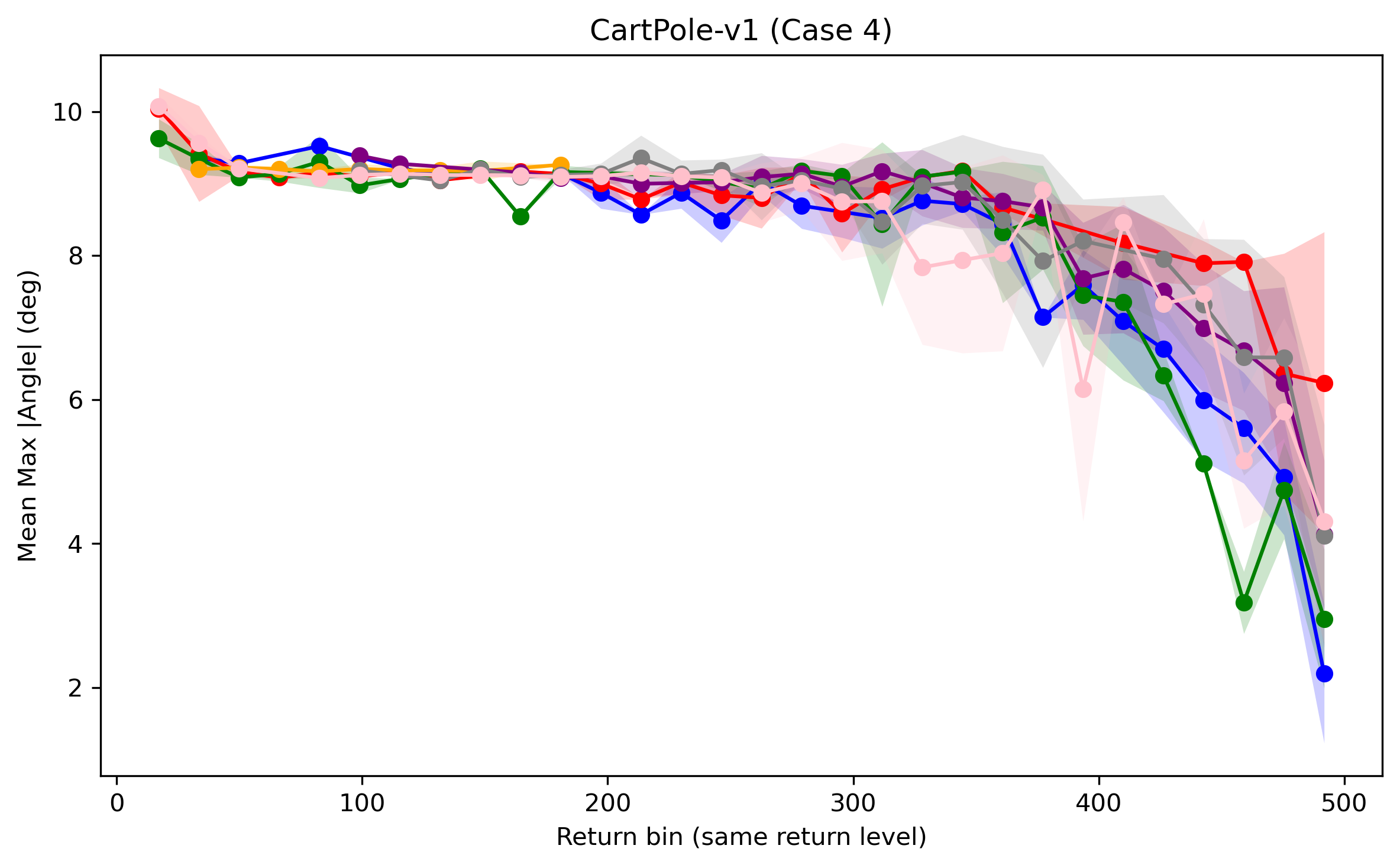}
    \caption{\centering Case 4:\\Continuous, HC safe, distributional}
    \label{fig6d}
    \end{subfigure}
    \begin{subfigure}{0.32\linewidth}
    \includegraphics[width=\linewidth]{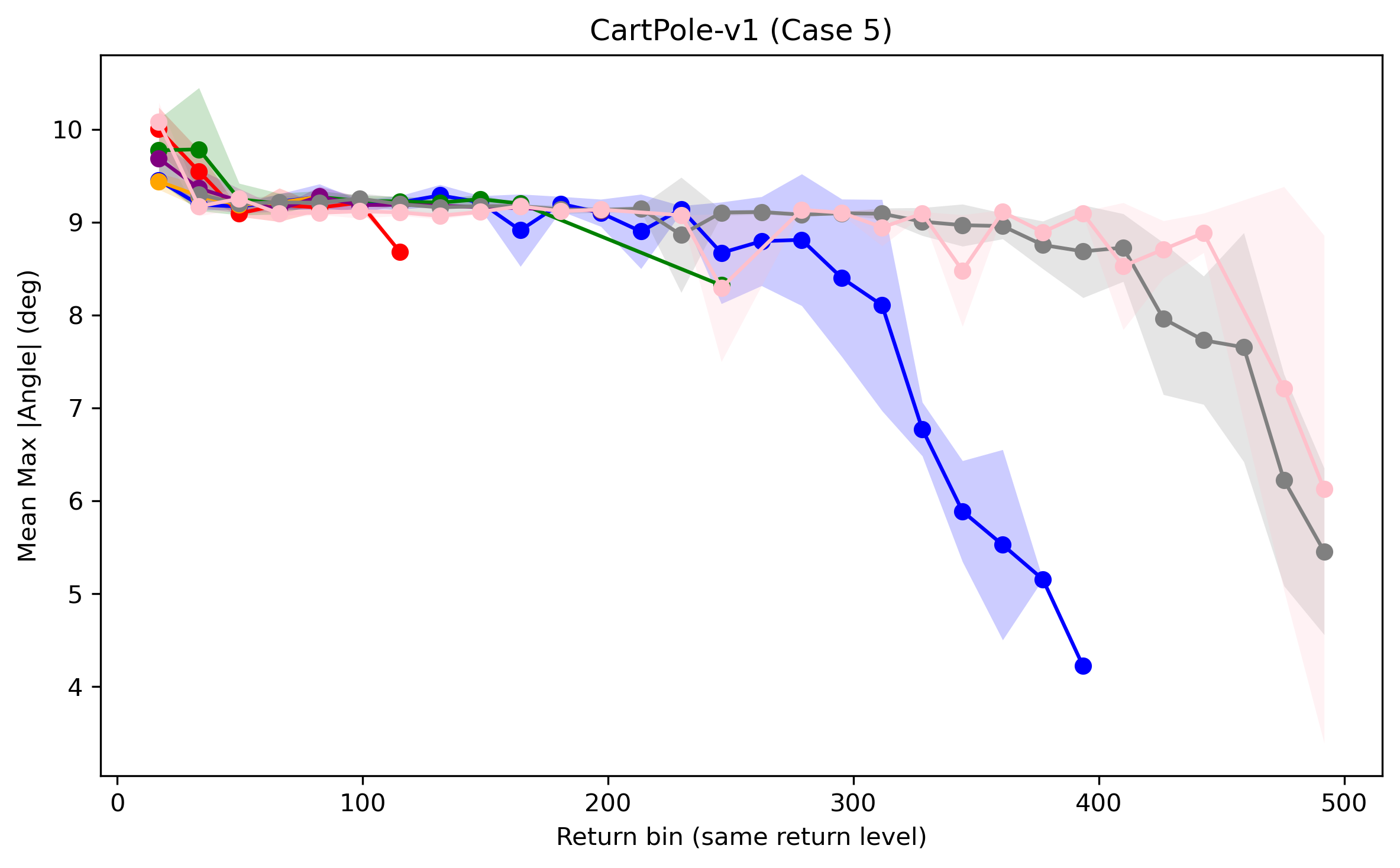}
    \caption{\centering Case 5:\\Continuous, DS safe, mean-noise}
    \label{fig6e}
    \end{subfigure}
    \begin{subfigure}{0.32\linewidth}
    \includegraphics[width=\linewidth]{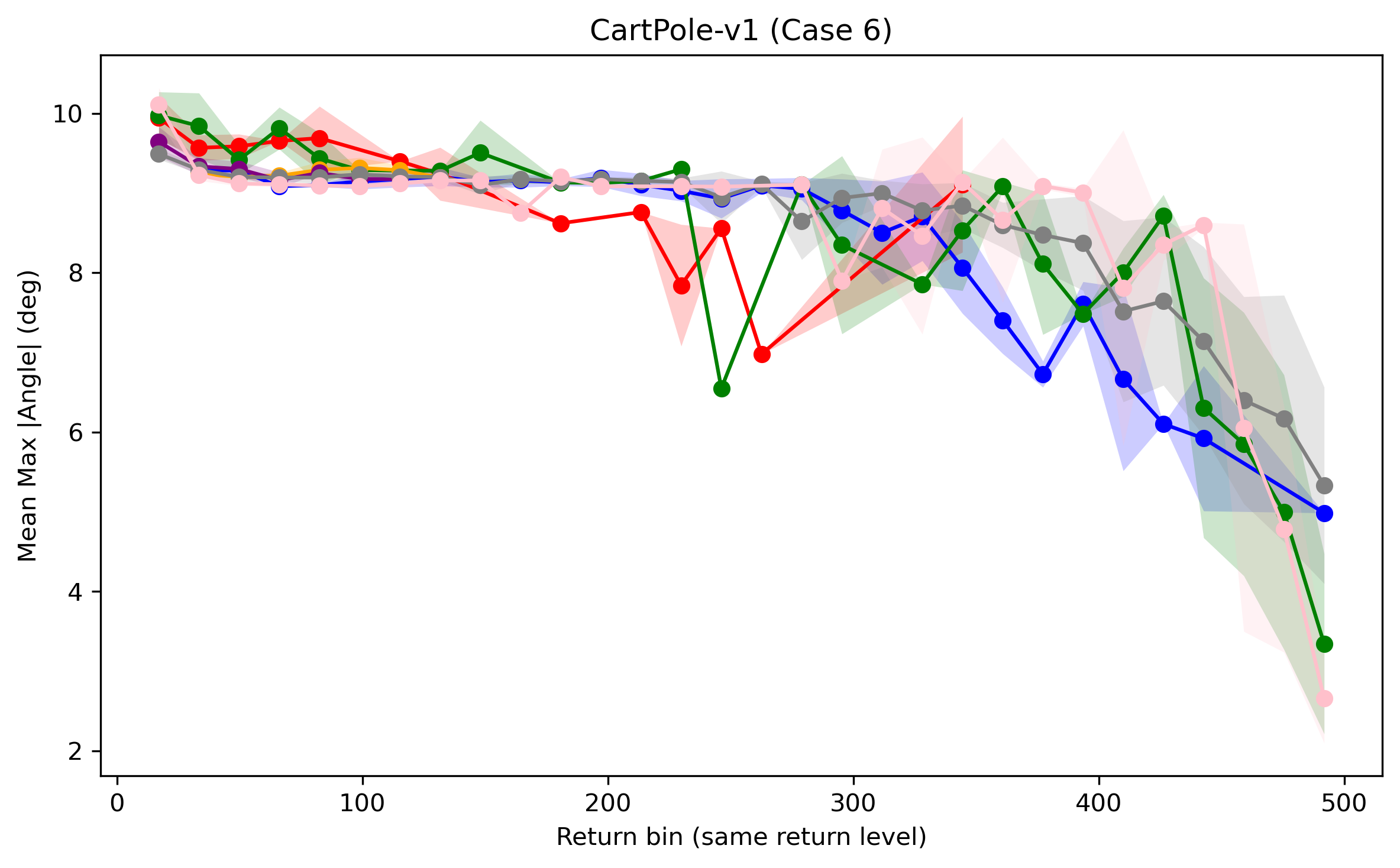}
    \caption{\centering Case 6:\\Continuous, DS safe, distributional}
    \label{fig6f}
    \end{subfigure}

    \caption{Same-return mean maximum pole angle curve on CartPole-v1 comparing the proposed method with baselines under different settings.}
    \label{fig6}
\end{figure}

To answer the last question, we group results into return bins and compare the mean maximum pole angle across models within each bin. Specifically, we aggregate returns from all models that were collected during training at every 100 episodes, partition them into bins defined between the maximum and minimum returns, and compute the average maximum pole angle within each bin. Bins without data indicate that no checkpoints achieve returns in that range. In~\cref{fig6}, the proposed method achieves lower maximum pole angles than baselines across most return ranges, which indicates safer behavior. While CQL attains low angles in~\cref{fig6b,fig6c,fig6d}, its overall performance is poor in~\cref{fig4b,fig4c,fig4d}. In contrast, compared to methods with similar returns (e.g., BCQ and TD3+BC in~\cref{fig4c,fig4d}), our method maintains comparable or lower angles and notably achieves the lowest angle at return 500 across all settings. Although all methods show decreasing angles at higher returns, ours produces safer trajectories at equivalent performance levels. In~\cref{fig6e,fig6f}, the proposed method appears safe but is overly conservative, leading to low performance in~\cref{fig4e,fig4f}.

\section{Conclusion}
In this work, we presented a safe-support Q-learning framework that enforces a strict safety requirement by restricting training to trajectories that remain within the safe set. By leveraging a stochastic behavior policy supported on the safe set, the proposed method enables exploration within a safe region without relying on unsafe interactions. We introduced a KL-regularized Bellman target to align the Q-function with the safe behavior distribution, promoting both safety and performance improvement. Theoretical analysis showed that the proposed Bellman operator admits a unique fixed point and that the induced policy corresponds to a KL-regularized optimal control solution. For continuous action spaces, we further proposed a practical policy extraction method based on parametric approximation. Empirical results demonstrated that the proposed method achieves stable learning, well-calibrated value estimates, and consistently safer behavior compared to baselines. In particular, it reduces risk and variance while maintaining competitive performance across a wide range of settings. Despite these advantages, the proposed method relies on a predefined behavior policy, which may limit its applicability in more complex environments and introduce a trade-off between safety and performance when the behavior policy is imperfect. Future work will focus on learning such behavior policies and extending the framework to offline RL scenarios with OOD challenges, as well as to more complex and high-dimensional environments.


\bibliographystyle{acm}
\bibliography{reference}

\section{Appendix}
\subsection{Proof of~\cref{prop1}} \label{prop1_proof}
To derive our result, we first introduce the following definition and lemma.

\begin{definition}[Conjugate function~\cite{boyd2004convex}] \label{def1}
For a given function $f:{\mathbb R}^n \to {\mathbb R}$, its conjugate, $f^*:{\mathbb R}^n \to {\mathbb R}$, is defined as
\[ f^*(y): = \sup_{x \in {\rm dom}(f)}({y^T}x - f(x)). \]
\end{definition}

\begin{lemma} \label{lem1}
Let us consider the function $f$ defined as
\begin{align*}
f(x) = \lambda D_{\rm KL}(x \| q) = \lambda \sum\limits_{i = 1}^n {{x_i}} \ln \left( {\frac{{{x_i}}}{{{q_i}}}} \right),
\end{align*}
where $\lambda$ is a regularization parameter and $x = (x_1,\ldots,x_n)$ and $q = (q_1,\ldots,q_n)$ are probability distributions over $\{1,2,\ldots,n\}$. Moreover, we assume that $q_i>0$ for all $i \in \{1,2,\ldots,n\}$. Then, its conjugate $f^*$ is given as
\[ f^*(y) = \max_{x \in {\rm dom}(f)} \left( {x^T}y - \lambda D_{\rm KL}(x \| q) \right) = \lambda \ln \left( \sum_{i= 1}^n q_i \exp \left( \frac{y_i}{\lambda} \right) \right), \]
where ${\rm dom}(f) = \Delta_n$.
\end{lemma}
\begin{proof}
Utilizing~\cref{def1} and ${\rm dom}(f) = \Delta_n$, we obtain
\[ f^*(y) = \max_{x \in \Delta_n} \left( \sum_{i= 1}^n x_i y_i + \lambda \sum_{i= 1}^n x_i \ln(q_i) - \lambda \sum_{i= 1}^n x_i \ln(x_i) \right), \]
where $x \in \Delta_n$ satisfies $\sum_{i=1}^n x_i = 1$ and $x_i \ge 0$ for all $i \in \{1,2,\ldots,n\}$.

To solve the constrained optimization problem, we introduce the Lagrangian function
\[ L(x,h_0,h_1,\ldots,h_n) := \sum_{i= 1}^n x_i y_i + \lambda \sum_{i= 1}^n x_i \ln q_i - \lambda \sum_{i= 1}^n x_i \ln x_i + h_0 \left( 1-\sum_{i= 1}^n x_i \right) + \sum_{i = 1}^n (-x_i)h_i, \]
where $h_0,h_1,\ldots,h_n \in {\mathbb R}$ are Lagrangian multipliers associated with the equality and inequality constraints, respectively. The corresponding Karush–Kuhn–Tucker (KKT) conditions~\cite{boyd2004convex} are given by
\begin{itemize}
    \item Stationarity:
    \begin{equation} \label{appendix:KKT-eq1}
        \frac{\partial L}{\partial x_i} = y_i + \lambda \ln q_i - \lambda (1 + \ln x_i) - h_0 - h_i = 0, \quad \forall i.
    \end{equation}

    \item Primal feasibility:
    \begin{equation} \label{appendix:primal-eq1}
        \frac{\partial L}{\partial h_0} = 1 - \sum_{i=1}^n x_i = 0,
    \end{equation}
    \[ \frac{\partial L}{\partial h_i} = -x_i \le 0, \quad \forall i. \]

    \item Dual feasibility:
    \[ h_i \ge 0, \quad \forall i. \]

    \item Complementary slackness:
    \[ \sum_{i= 1}^n x_i h_i = 0.\]
\end{itemize}
From the stationarity condition~\eqref{appendix:KKT-eq1}, we obtain
\begin{equation} \label{appendix:KKT-eq2}
    x_i = \exp \left( \frac{y_i}{\lambda} + \ln q_i - \frac{h_0 + h_i}{\lambda} - 1 \right) = \exp \left( \frac{y_i}{\lambda} \right) q_i \exp \left( -\frac{h_0 + h_i}{\lambda} - 1 \right).
\end{equation}
Since the exponential function is strictly positive, it follows that $x_i > 0$ for all $i$, and hence $h_i=0$ for $i=\{1,2,\ldots,n\}$ by complementary slackness. Moreover, from the first primal feasibility condition~\eqref{appendix:primal-eq1}, we have 
\[ \sum_{i=1}^n x_i = \sum_{i=1}^n \exp \left( \frac{y_i}{\lambda} \right) q_i \exp \left( -\frac{h_0}{\lambda} - 1 \right) = 1, \]
which implies
\[ \exp \left( -\frac{h_0}{\lambda} - 1 \right) = \frac{1}{\sum_{i=1}^n \exp \left( \frac{y_i}{\lambda} \right) q_i}. \]
Then, substituting this into~\eqref{appendix:KKT-eq2} yields
\begin{equation} \label{appendix:KKT-eq3}
    x_i = \frac{\exp \left( \frac{y_i}{\lambda} \right) q_i}{\sum_{i=1}^n \exp \left( \frac{y_i}{\lambda} \right) q_i}.
\end{equation}
By defining
\[ g(x,y) := {x^T}y - \lambda D_{\rm KL}(x \| q) = {x^T}y - \lambda \sum_{i=1}^n x_i \ln \left( \frac{x_i}{q_i} \right), \]
and using~\eqref{appendix:KKT-eq3}, one gets
\begin{align*}
    g(x,y) 
    =& \frac{\sum_{j=1}^n y_j \exp \left( \frac{y_j}{\lambda} \right) q_j}{\sum_{i=1}^n \exp \left( \frac{y_i}{\lambda} \right) q_i} - \lambda \frac{\sum_{j=1}^n \exp \left( \frac{y_j}{\lambda} \right) q_j \ln \left( \frac{\exp \left( \frac{y_j}{\lambda} \right) q_j}{\sum_{l=1}^n \exp \left( \frac{y_l}{\lambda} \right) q_l} \cdot \frac{1}{q_i} \right)}{\sum_{i=1}^n \exp \left( \frac{y_i}{\lambda} \right) q_i} \\
    =& \frac{\sum_{j=1}^n y_j \exp \left( \frac{y_j}{\lambda} \right) q_j}{\sum_{i=1}^n \exp \left( \frac{y_i}{\lambda} \right) q_i} - \lambda \frac{\sum_{j=1}^n \exp \left( \frac{y_j}{\lambda} \right) q_j \cdot \frac{y_j}{\lambda}}{\sum_{i=1}^n \exp \left( \frac{y_i}{\lambda} \right) q_i} +
    \lambda \frac{\sum_{j=1}^n \exp \left( \frac{y_j}{\lambda} \right) q_j \ln \left( \sum_{l=1}^n \exp \left( \frac{y_l}{\lambda} \right) q_l \right)}{\sum_{i=1}^n \exp \left( \frac{y_i}{\lambda} \right) q_i} \\
    =& \lambda \ln \left( \sum_{i=1}^n \exp \left( \frac{y_i}{\lambda} \right) q_i \right),
\end{align*}
which completes the proof.
\end{proof}

We now prove~\cref{prop1}. 

\vspace{0.3cm}
\textbf{\cref{prop1} Restated. } The target value in~\eqref{eqn:idea_KL} can be equivalently expressed as
\[ y_k^{\rm safe} = r_{k+1} + {\bf 1}(s_{k+1})\gamma \lambda \ln \left( \sum_{a \in {\cal A}} \tilde{\pi}_b(a|s_{k+1}) \exp\left( \frac{Q(s_{k+1},a)}{\lambda} \right) \right). \]
\begin{proof}
Using the definition of the KL divergence, the target value in~\eqref{eqn:idea_KL} can be rewritten as
\[ y_k^{\rm safe} = r_{k+1} + {\bf 1}(s_{k+1})\gamma \max _{\pi \in \Delta_{|{\cal A}|}} \left\{ \sum_{a \in {\cal A}} \pi(a | s_{k+1})Q(s_{k+1},a) - \lambda \sum_{a \in {\cal A}} \pi(a | s_{k+1}) \ln\left( \frac{\pi(a | s_{k+1})}{\tilde{\pi}_b(a | s_{k+1})}\right) \right). \]
Let
\begin{align*}
    y 
    =& \left[ {\begin{array}{*{20}{c}}
        y_1 \\
        y_2 \\
         \vdots \\
        y_{|{\cal A}|}
        \end{array}} \right] 
    = \left[ {\begin{array}{*{20}{c}}
        \pi(1|s_{k+1}) \\
        \pi(2|s_{k+1}) \\
         \vdots \\
        \pi(|{\cal A}||s_{k+1})
        \end{array}} \right],
    x = \left[ {\begin{array}{*{20}{c}}
        x_1 \\
        x_2 \\
         \vdots \\
        x_{|{\cal A}|}
        \end{array}} \right] 
    = \left[ {\begin{array}{*{20}{c}}
        Q(s_{k+1},1) \\
        Q(s_{k+1},2) \\
         \vdots \\
        Q(s_{k+1},|{\cal A}|)
        \end{array}} \right], 
    q = \left[ {\begin{array}{*{20}{c}}
        q_1 \\
        q_2 \\
         \vdots \\
        q_{|{\cal A}|}
        \end{array}} \right] 
    = \left[ {\begin{array}{*{20}{c}}
        \tilde{\pi}_b(1|s_{k+1})\\
        \tilde{\pi}_b(2|s_{k+1})\\
         \vdots \\
        \tilde{\pi}_b(|{\cal A}||s_{k+1})
        \end{array}} \right].
\end{align*}
Then, the target value can be expressed as
\begin{align*}
    y_k^{\rm safe}
    =& r_{k+1} + {\bf 1}(s_{k+1})\gamma \max_{y \in \Delta_{|{\cal A}|}} \left( x^T y - \lambda \sum_{a \in {\cal A}} y_a \ln \left( \frac{y_a}{q_a} \right) \right) \\
    =& r_{k+1} + {\bf 1}(s_{k+1})\gamma \lambda \ln \left( \sum_{i=1}^n q_a \exp \left( \frac{x_a}{\lambda} \right) \right) \\
    =& r_{k+1} + {\bf 1}(s_{k+1}) \gamma \lambda \ln \left( \sum_{a \in {\cal A}} \tilde{\pi}_b(a | s_{k+1}) \exp \left( \frac{Q(s_{k+1},a)}{\lambda} \right) \right),
\end{align*}
where the second equality follows from~\cref{lem1}. This completes the proof.
\end{proof}

\subsection{Proof of~\cref{prop2}} \label{prop2_proof}
\textbf{\cref{prop2} Restated. }
Consider the Bellman equation $Q = T_\lambda Q$, where the operator $T_\lambda: {\mathbb R}^{|{\cal S}||{\cal A}|} \to {\mathbb R}^{|{\cal S}||{\cal A}|}$ is defined as
\[ (T_\lambda Q)(s,a) := {\mathbb E} \left[ \left. r_{k+1} + \gamma \lambda \ln \left( \sum_{a' \in {\cal A}} {\tilde{\pi}_b(a' | s_{k+1}) \exp \left( \frac{Q(s_{k+1},a')}{\lambda} \right)} \right) \right| s_k=s,a_k=a \right]. \]
Then, $T_\lambda$ admits a unique solution $Q_{\lambda}^*$. Moreover, the corresponding optimal policy is given by 
\[ \pi_\lambda^*(a | s) = \frac{\tilde{\pi}_b(a | s) \exp \left( Q_\lambda^*(s,a) / \lambda \right)}{\sum_{a' \in {\cal A}} \tilde{\pi}_b(a' | s) \exp \left(Q_\lambda^*(s,a') / \lambda \right)}. \]
\begin{proof}
To establish the existence and uniqueness of the solution to the Bellman equation, it suffices to show that $T_\lambda$ is a contraction mapping. The proof follows standard arguments in~\cite{dai2018sbeed,fox2016taming}. In particular, for any $Q_1, Q_2 \in {\mathbb R}^{|{\cal S} \times {\cal A}|}$, we have 
\begin{align*}
    & \left\| T_\lambda Q_1 - T_\lambda Q_2 \right\|_\infty \\
    =& \left\| \gamma \max_{\pi \in \Delta_{|{\cal A}|}} {\mathbb E} \left[ \sum_{a \in {\cal A}} \pi(a | s') Q_1(s',a) - \lambda D_{\rm KL} \left( \pi(\cdot | s') \| \tilde{\pi}_b(\cdot | s') \right) \right] \right. \\
    &- \left. \gamma \max_{\pi \in \Delta_{|{\cal A}|}} {\mathbb E} \left[ \sum_{a \in {\cal A}} \pi(a | s') Q_2(s',a) - \lambda D_{\rm KL} \left( \pi(\cdot | s') \| \tilde{\pi}_b(\cdot | s') \right) \right] \right\|_\infty \\
    \le& \gamma \left\| \max_{\pi \in \Delta_{|{\cal A}|}} \left\{ {\mathbb E} \left[ \sum_{a \in {\cal A}} \pi(a | s') Q_1(s',a) - \lambda D_{\rm KL} \left( \pi(\cdot | s') \| \tilde{\pi}_b(\cdot | s') \right) \right] \right. \right. \\
    &- \left. \left. {\mathbb E} \left[ \sum_{a \in {\cal A}} \pi(a | s') Q_2(s',a) - \lambda D_{\rm KL} \left( \pi(\cdot | s') \| \tilde{\pi}_b(\cdot | s') \right) \right] \right\} \right\|_\infty \\
    =& \gamma \left\| \max_{\pi \in \Delta_{|{\cal A}|}} {\mathbb E} \left[ \sum_{a \in {\cal A}} \pi(a | s') \left\{ Q_1(s',a) - Q_2(s',a)\right\} \right] \right\|_\infty \\
    =& \gamma \left\| \max_{a \in {\cal A}} {\mathbb E} \left[ Q_1(s',a) - Q_2(s',a) \right] \right\|_\infty \\
    \le& \gamma \left\| Q_1 - Q_2 \right\|_\infty ,
\end{align*}
where the first equality follows from~\eqref{eqn:idea_KL}, and the last inequality follows from the non-expansiveness of the max operator. Therefore, the corresponding Bellman equation admits a unique solution.

For the second statement, let $Q_{\rm safe}^*$ denote the unique solution to the Bellman equation, i.e., $Q_{\rm safe}^*(s,a) = (T Q_{\rm safe}^*)(s,a)$, which implies 
\begin{align*}
    & Q_{\rm safe}^*(s,a) \\
    =& {\mathbb E} \left[ r_{k+1} + \gamma \max_{\pi \in \Delta_{|{\cal A}|}} \left\{ \left. \sum_{a \in {\cal A}} \pi(a | s_{k+1}) Q_{\rm safe}^*(s_{k+1},a) - \lambda D_{\rm KL} \left( \pi(\cdot | s_{k+1}) \| \tilde{\pi}_b(\cdot | s_{k+1}) \right) \right\} \right| s_0=s, a_0=a \right].
\end{align*}
Then, the corresponding policy is given by
\[ \pi_{\rm safe}^*(\cdot | s) = \arg\max_{\pi \in \Delta_{|{\cal A}|}} \left\{ \sum_{a \in {\cal A}} \pi(a | s) Q_{\rm safe}^*(s,a) - \lambda D_{\rm KL} \left( \pi(\cdot | s) \| \tilde{\pi}_b(\cdot | s) \right) \right\}. \]
The $\arg\max$ can be derived following the same steps as in the proof of~\cref{prop1}. This completes the proof.
\end{proof}

\subsection{Derivation of the Continuous Policy Objective} \label{cont pol obj}
To simplify the objective~\eqref{cont policy obj1}, we expand the KL divergence as
\begin{equation} \label{cont policy obj}
\begin{aligned}
    L(\phi) =& {\mathbb E}_{s \sim {\cal N}(0,\sigma_1^2 I)} \left[ \int_{a \in {\cal A}} \pi_\phi(a | s) \left\{ \ln \pi_\phi(a | s) + \ln \left( \int_{u \in {\cal A}} \tilde{\pi}_b(u | s) \exp \left( \frac{Q_\theta(s,u)}{\lambda} \right) du \right) \right. \right. \\
    &- \left.\left. \ln \left( \tilde{\pi}_b(a | s) \exp \left( \frac{Q_\theta(s,a)}{\lambda} \right) \right) \right\} da \right],
\end{aligned}
\end{equation}
where the constant term does not depend on $\phi$ and can therefore be ignored during optimization.

Then, we parameterize $\pi_\phi$ as a Gaussian policy:
\[ \pi_\phi(a | s) = {\cal N}(a;\mu_\phi(s),\sigma_2^2 I), \]
where $\mu_\phi(s)$ and $\sigma_2 > 0$ denote the mean and standard deviation, respectively. Using the reparameterization trick,
\[ a = \mu_\phi(s) + w, \quad w \sim {\cal N}(0,\sigma_2^2 I), \]
the objective in~\eqref{cont policy obj} can be rewritten as
\begin{equation} \label{cont policy obj2}
    L(\phi) = {\mathbb E}_{s \sim {\cal N}(0,\sigma_1^2I), w \sim {\cal N}(0,\sigma_2^2 I)} \left[ \ln \pi_\phi \left(\mu_\phi(s) + w | s \right) - \ln \tilde{\pi}_b \left( \mu_\phi(s) + w | s \right) - \frac{Q_\theta(s,\mu_\phi(s) + w)}{\lambda} \right].
\end{equation}
Here, the first term in~\eqref{cont policy obj2} can be written as
\begin{align*}
    \ln \pi_\phi \left( \mu_\phi(s) + w | s \right)
    =& \ln {\cal N} \left( \mu_\phi(s) + w;\, \mu_\phi(s), \sigma_2^2 I \right) \\
    =& \ln \left\{ \frac{1}{(2\pi)^{d/2} \sigma_2^d} \exp \left( -\frac{1}{2\sigma_2^2} \left\| \mu_\phi(s) + w - \mu_\phi(s) \right\|_2^2 \right) \right\} \\
    =& -\frac{d}{2}\ln(2\pi) - d \ln \sigma_2 - \frac{1}{2\sigma_2^2} \|w\|_2^2,
\end{align*}
where $d$ denotes the action dimension. This term is constant with respect to $\phi$ and can therefore be omitted during optimization.

To simplify the objective~\eqref{cont policy obj2}, we consider two cases depending on the form of the behavior policy.
\paragraph{Case 1: mean-noise behavior policy}
Suppose that the behavior policy is given by~\eqref{mean-noise behavior policy}. Then, the second term in~\eqref{cont policy obj2} can be written as
\begin{align*}
    \ln \tilde{\pi}_b(\mu_\phi(s) + w | s)
    =& \ln {\cal N} \left(\mu_\phi(s) + w;\, \mu_b(s), \sigma_M^2 I \right) \\
    =& \ln \left\{ \frac{1}{(2\pi)^{d/2} \sigma_M^d} \exp \left( -\frac{1}{2\sigma_M^2} \left\| \mu_\phi(s) + w - \mu_b(s) \right\|_2^2 \right) \right\} \\
    =& -\frac{d}{2}\ln(2\pi) - d \ln \sigma_M - \frac{1}{2\sigma_M^2} \left\| \mu_\phi(s) + w - \mu_b(s) \right\|_2^2.
\end{align*}
By removing terms that do not depend on $\phi$, the loss function~\eqref{cont policy obj2} reduces to
\begin{equation} \label{cont policy obj_case1}
    L(\phi) = {\mathbb E}_{s \sim {\cal N}(0,\sigma_1^2I), w \sim {\cal N}(0,\sigma_2^2 I)} \left[ \frac{1}{2\sigma_M^2} \left\| \mu_\phi(s) + w - \mu_b(s) \right\|_2^2 - \frac{Q_\theta(s,\mu_\phi(s)+w)}{\lambda} \right].
\end{equation}
This objective encourages the policy to maximize the Q-value while remaining close to the behavior policy $\mu_b(s)$, where the parameter $\lambda$ controls the trade-off between these two objectives.

In practice, the expectation in~\eqref{cont policy obj_case1} is approximated using Monte Carlo sampling:
\begin{equation}
    \tilde L(\phi) := \sum_{i=1}^p \sum_{j=1}^q \left[ \frac{1}{2\sigma_M^2} \left\| \mu_\phi(s_i) + w_j - \mu_b(s_i) \right\|_2^2 - \frac{Q_\theta(s_i,\mu_\phi(s_i)+w_j)}{\lambda} \right],
\end{equation}
where $s_i \sim {\cal N}(0,\sigma_1^2 I)$ and $w_j \sim {\cal N}(0,\sigma_2^2 I)$. The parameters $\phi$ are updated via gradient descent. The overall procedure is summarized in~\cref{alg5} of the Appendix, which is described for the online setting. For the offline setting, the replay buffer is replaced by a safe offline dataset, from which mini-batches are used to train both the Q-function and the policy.

\paragraph{Case 2: distributional behavior policy}
Suppose that the behavior policy is given by~\eqref{distributional behavior policy}. Then, the loss function~\eqref{cont policy obj2} becomes
\[ L(\phi) = {\mathbb E}_{s \sim {\cal N}(0,\sigma_1^2 I), w \sim {\cal N}(0,\sigma_2^2 I)} \left[ -\ln \tilde{\pi}_b(\mu_\phi(s) + w | s) - \frac{Q_\theta(s,\mu_\phi(s) + w)}{\lambda} \right]. \]
This objective encourages the policy to assign high probability under the behavior policy while maximizing the Q-value.

The expectation is approximated using Monte Carlo sampling as
\begin{equation}
\tilde L(\phi) := \sum_{i=1}^p \sum_{j=1}^q \left[ -\ln \tilde{\pi}_b(\mu_\phi(s_i) + w_j | s_i) -\frac{Q_\theta(s_i,\mu_\phi(s_i) + w_j)}{\lambda} \right],
\end{equation}
where $s_i \sim {\cal N}(0,\sigma_1^2 I)$ and $w_j \sim {\cal N}(0,\sigma_2^2 I)$. The parameters $\phi$ are updated via gradient descent. The overall procedure is summarized in~\cref{alg6} of the Appendix, which is described for the online setting. For the offline setting, the replay buffer is replaced by a safe offline dataset, from which mini-batches are used to train both the Q-function and the policy.

\subsection{Experimental Details} \label{exp setting}
\begin{figure}[ht]
    \centering
    \begin{subfigure}{0.25\linewidth}
    \includegraphics[width=\linewidth]{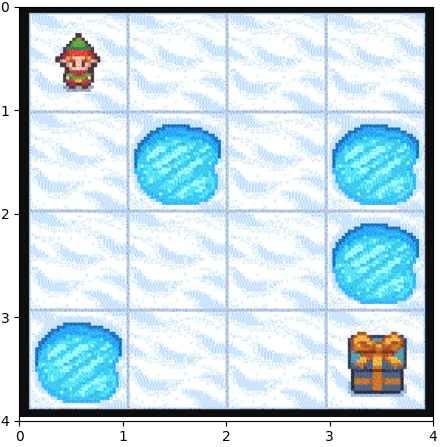}
    \caption{FrozenLake environment}
    \label{fig1a}
    \end{subfigure}
     \begin{subfigure}{0.5\linewidth}
    \includegraphics[width=\linewidth]{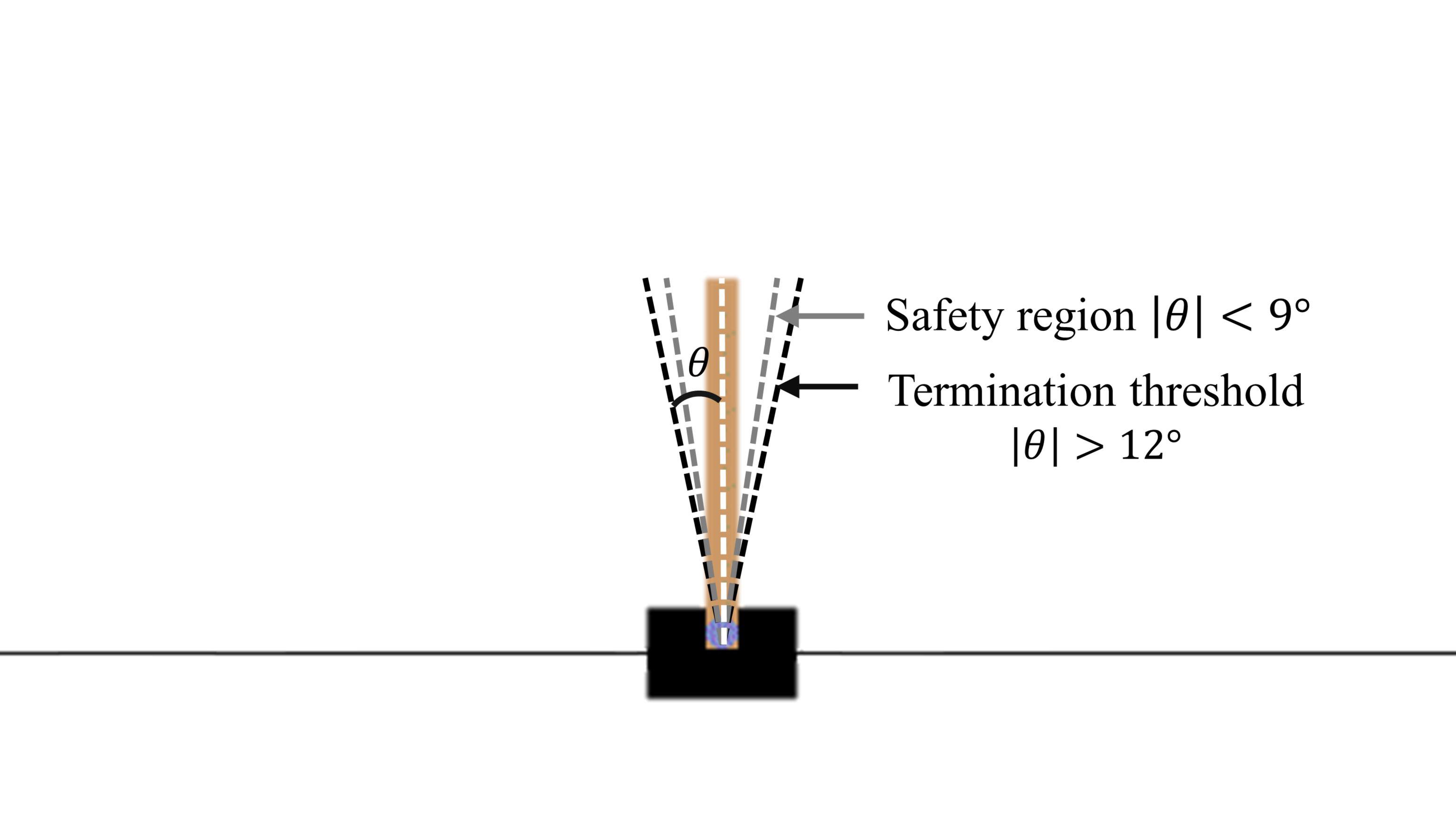}
    \caption{CartPole environment}
    \label{fig1b}
    \end{subfigure}
    
    \caption{Overview of the environments used in our experiments.}
\end{figure}

FrozenLake is a simple grid-world environment with discrete states and actions, where the agent navigates to reach a goal while avoiding unsafe hole states~\cite{brockman2016openai}. The environment has a fixed $4\times 4$ layout as shown in~\cref{fig1a}, and the agent receives a reward of $+1$ upon reaching the goal, with no penalty for falling into a hole. In this environment, the HC safe is handcrafted by assigning zero probability to actions leading to holes and a higher probability to optimal paths. For DS safe, the behavior policy is learned from the safe dataset $D_{\text{safe}}$, which contains only safe trajectories.

In the CartPole environment, the agent needs to balance a pole on a moving cart~\cite{brockman2016openai}. We evaluate our method in both discrete and continuous action settings: the agent either moves left or right, or applies a continuous force in the range $(-1, 1)$. The agent receives a reward $+1$ at each time step while the pole remains balanced, and the episode terminates when the pole angle exceeds $\pm 12^\circ$, the cart position exceeds the allowed bounds, or after $500$ steps. As the environment does not explicitly define unsafe states, we consider states with pole angles exceeding $\pm 9^\circ$ as unsafe, as illustrated in~\cref{fig1b}. For HC safe, we design a behavior policy based on a PID controller and train a neural network to imitate it, introducing stochasticity for balanced exploration. For DS safe, the behavior policy is trained from the safe dataset, which contains only safe trajectories collected by the PID controller.

In the CartPole environment, we consider DQN, CQL, BCQ, IQL, TD3+BC, and BEAR as baselines for a fair comparison under our behavior policy setting. Note that TD3+BC and BEAR are only applicable to continuous action spaces.
\begin{itemize}
    \item DQN: standard deep Q-learning without offline-specific regularization; serves as a basic off-policy baseline and does not address distributional shift.

    \item CQL: learns conservative Q-values by penalizing unseen actions, reducing overestimation in offline settings.

    \item BCQ: restricts action selection to the support of the dataset, preventing extrapolation error.

    \item IQL: avoids out-of-distribution actions via value regression and advantage-weighted behavior cloning.

    \item TD3+BC: combines policy gradient updates with behavior cloning to stabilize offline learning.

    \item BEAR: constrains the policy to remain close to the behavior policy via divergence constraint, reducing distributional shift.
\end{itemize}

\subsection{Safety Evaluation across Return Levels}
\begin{figure}[ht]
    \centering
    \begin{subfigure}{0.8\linewidth}
    \includegraphics[width=\linewidth]{figure/experiment_cartpole/legend_overestimation.png}
    \end{subfigure}
    
    \begin{subfigure}{0.32\linewidth}
    \includegraphics[width=\linewidth]{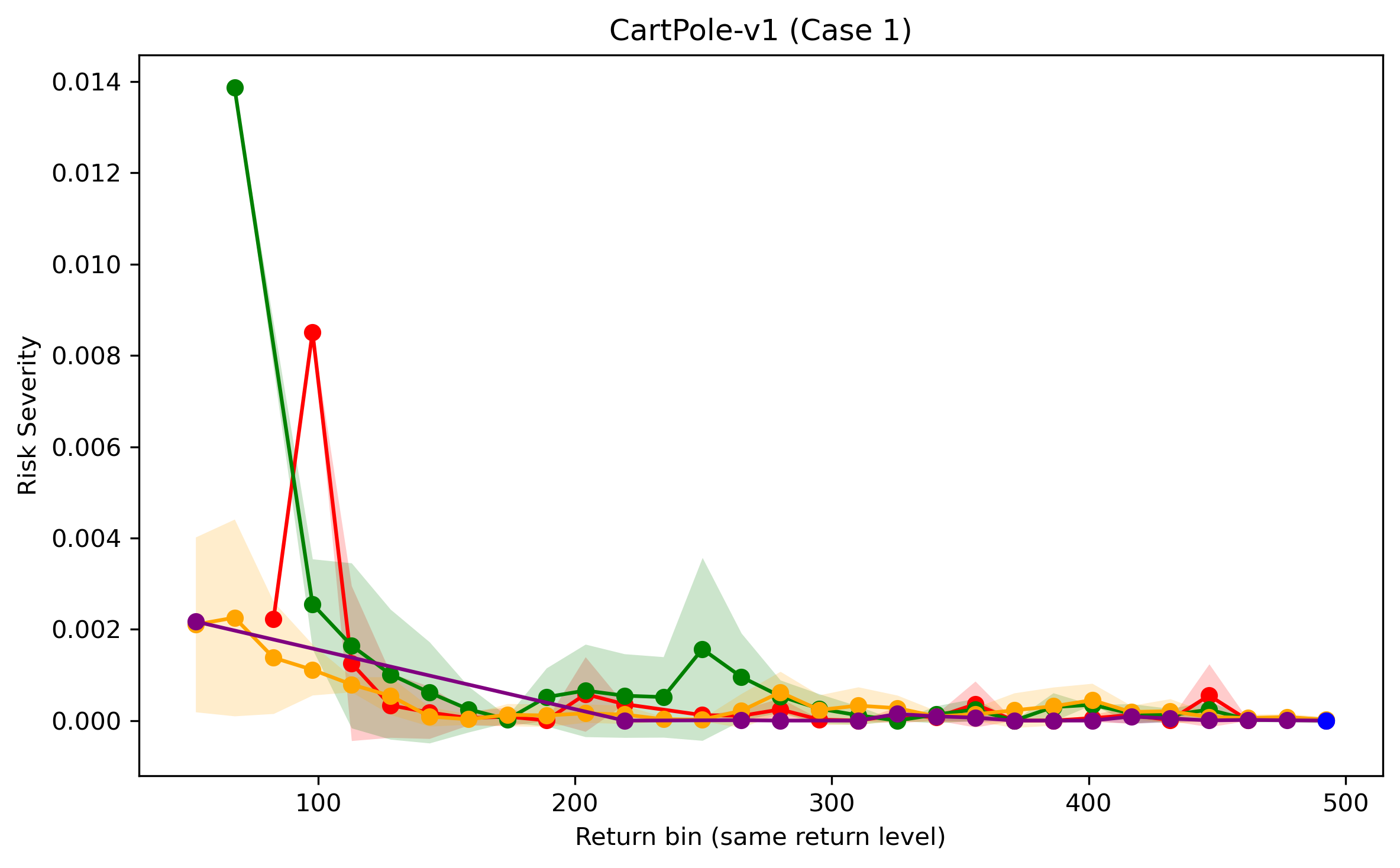}
    \caption{\centering Case 1:\\Discrete, HC safe}
    \end{subfigure}
    \begin{subfigure}{0.32\linewidth}
    \includegraphics[width=\linewidth]{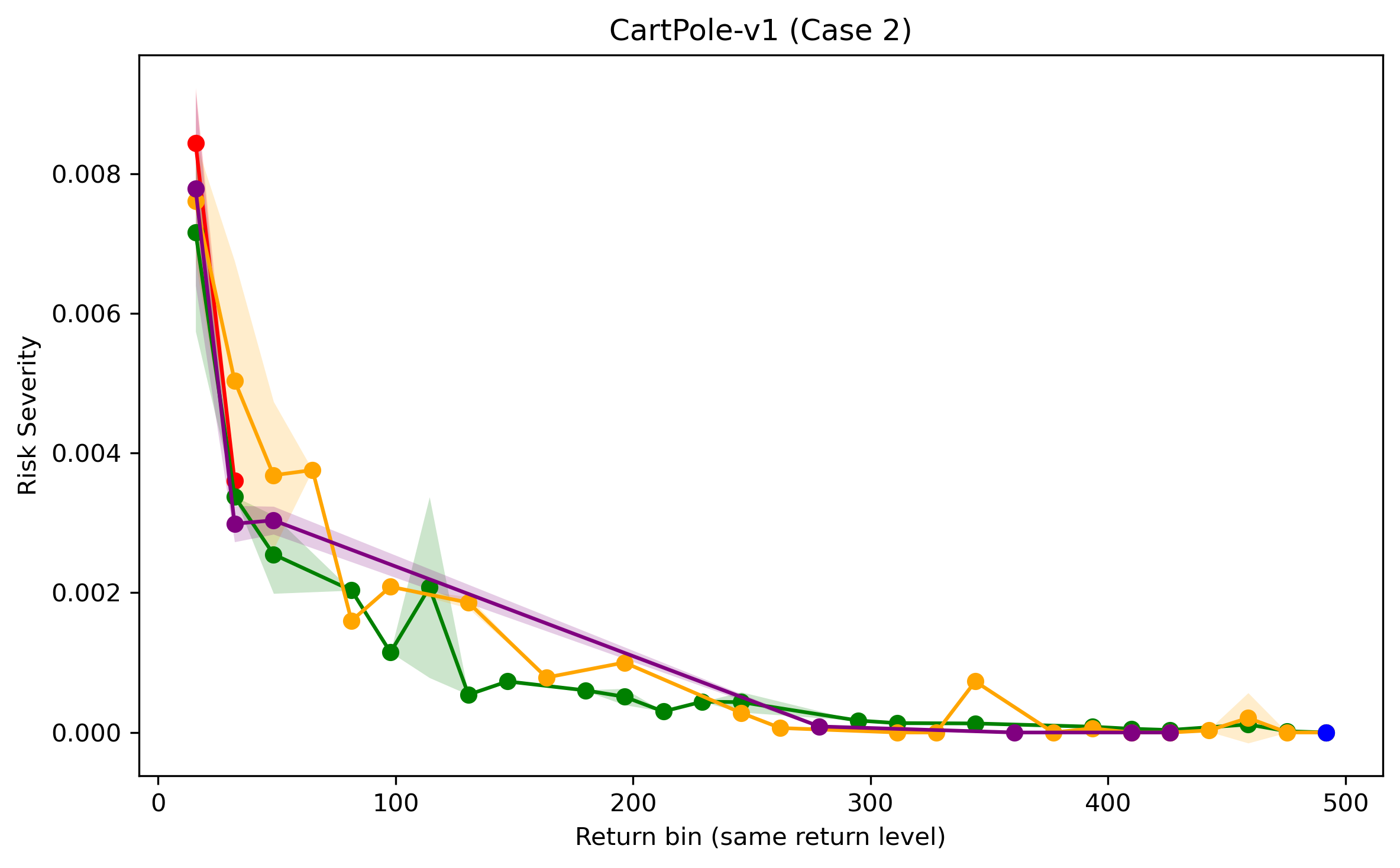}
    \caption{\centering Case 2:\\Discrete, DS safe}
    \end{subfigure}
    \begin{subfigure}{0.32\linewidth}
    \includegraphics[width=\linewidth]{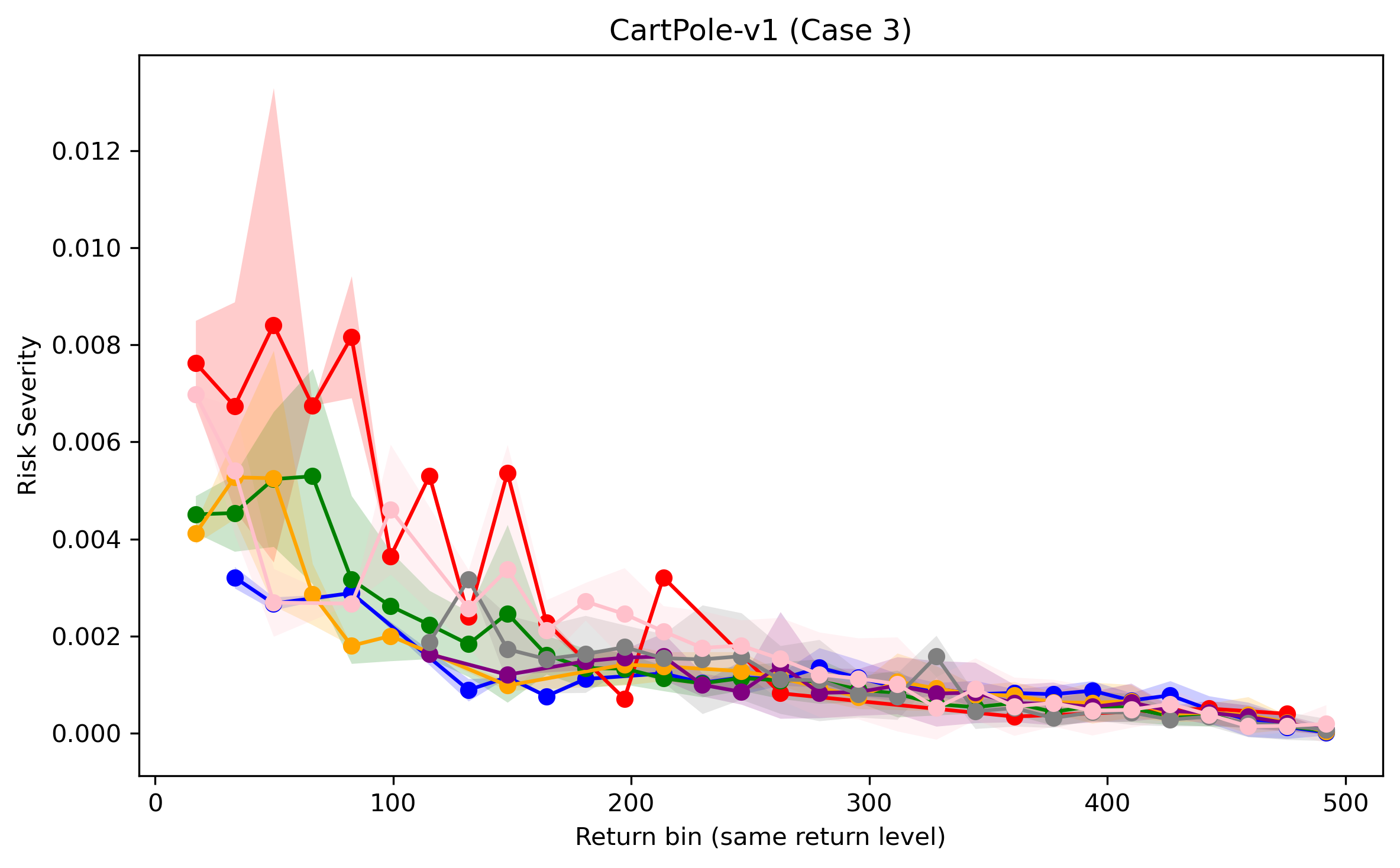}
    \caption{\centering Case 3:\\Continuous, HC safe, mean-noise}
    \end{subfigure}
    
    \begin{subfigure}{0.32\linewidth}
    \includegraphics[width=\linewidth]{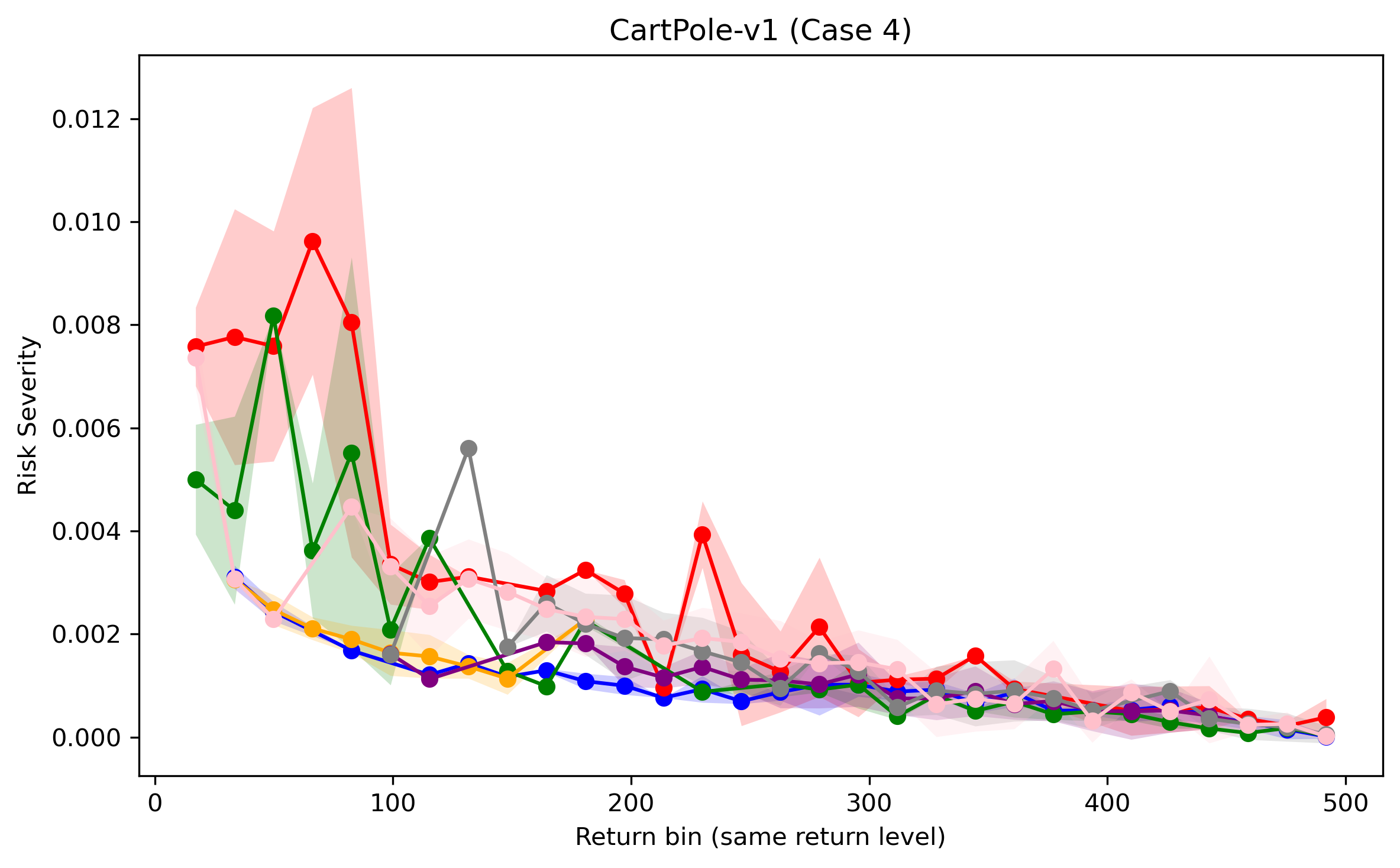}
    \caption{\centering Case 4:\\Continuous, HC safe, distributional}
    \end{subfigure}
    \begin{subfigure}{0.32\linewidth}
    \includegraphics[width=\linewidth]{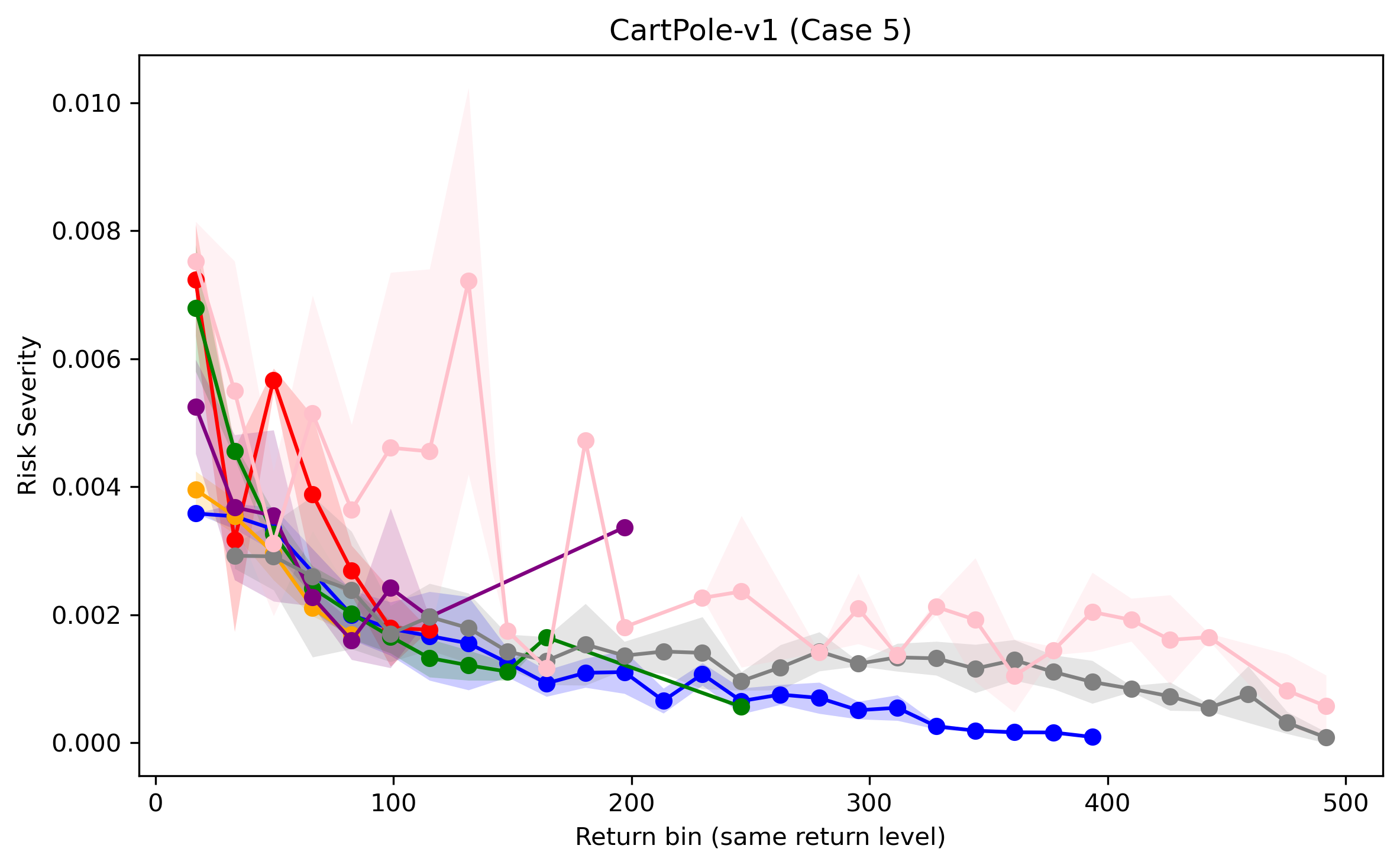}
    \caption{\centering Case 5:\\Continuous, DS safe, mean-noise}
    \end{subfigure}
    \begin{subfigure}{0.32\linewidth}
    \includegraphics[width=\linewidth]{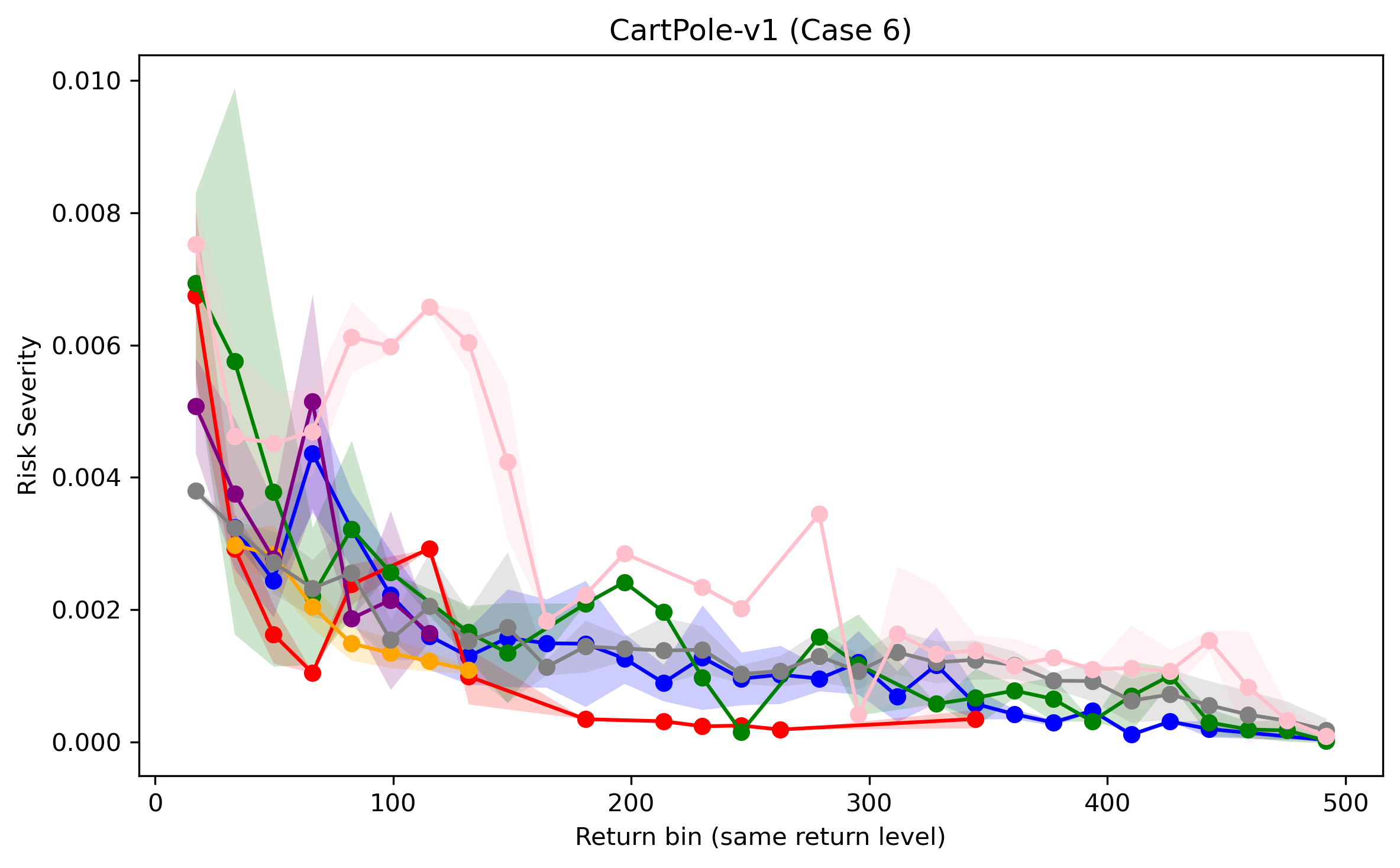}
    \caption{\centering Case 6:\\Continuous, DS safe, distributional}
    \end{subfigure}

    \caption{Same-return risk severity curve on CartPole-v1 comparing the proposed method with baselines under different settings.}
    \label{fig7}
\end{figure}

\begin{figure}[ht]
    \centering
    \begin{subfigure}{0.8\linewidth}
    \includegraphics[width=\linewidth]{figure/experiment_cartpole/legend_overestimation.png}
    \end{subfigure}
    
    \begin{subfigure}{0.32\linewidth}
    \includegraphics[width=\linewidth]{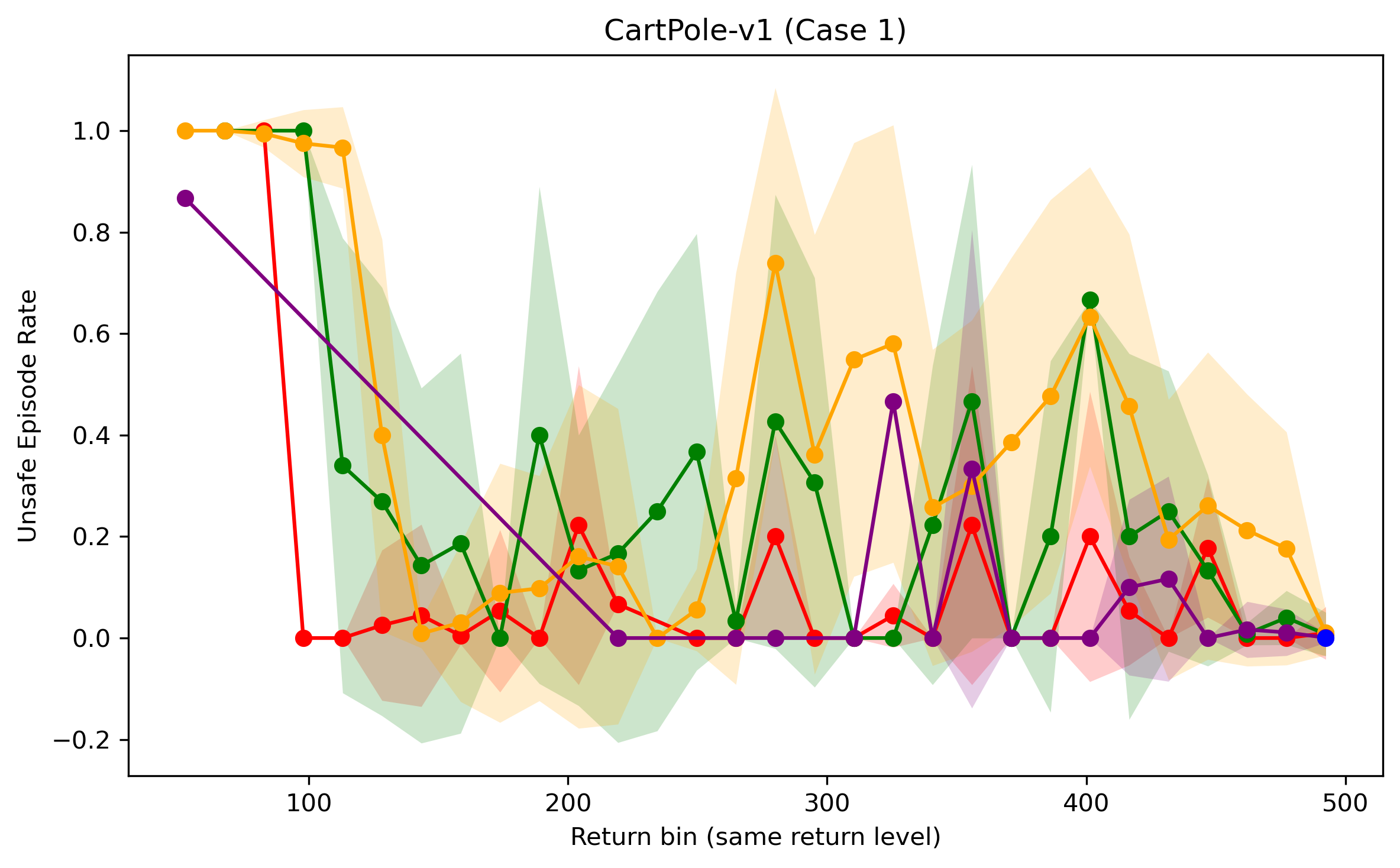}
    \caption{\centering Case 1:\\Discrete, HC safe}
    \end{subfigure}
    \begin{subfigure}{0.32\linewidth}
    \includegraphics[width=\linewidth]{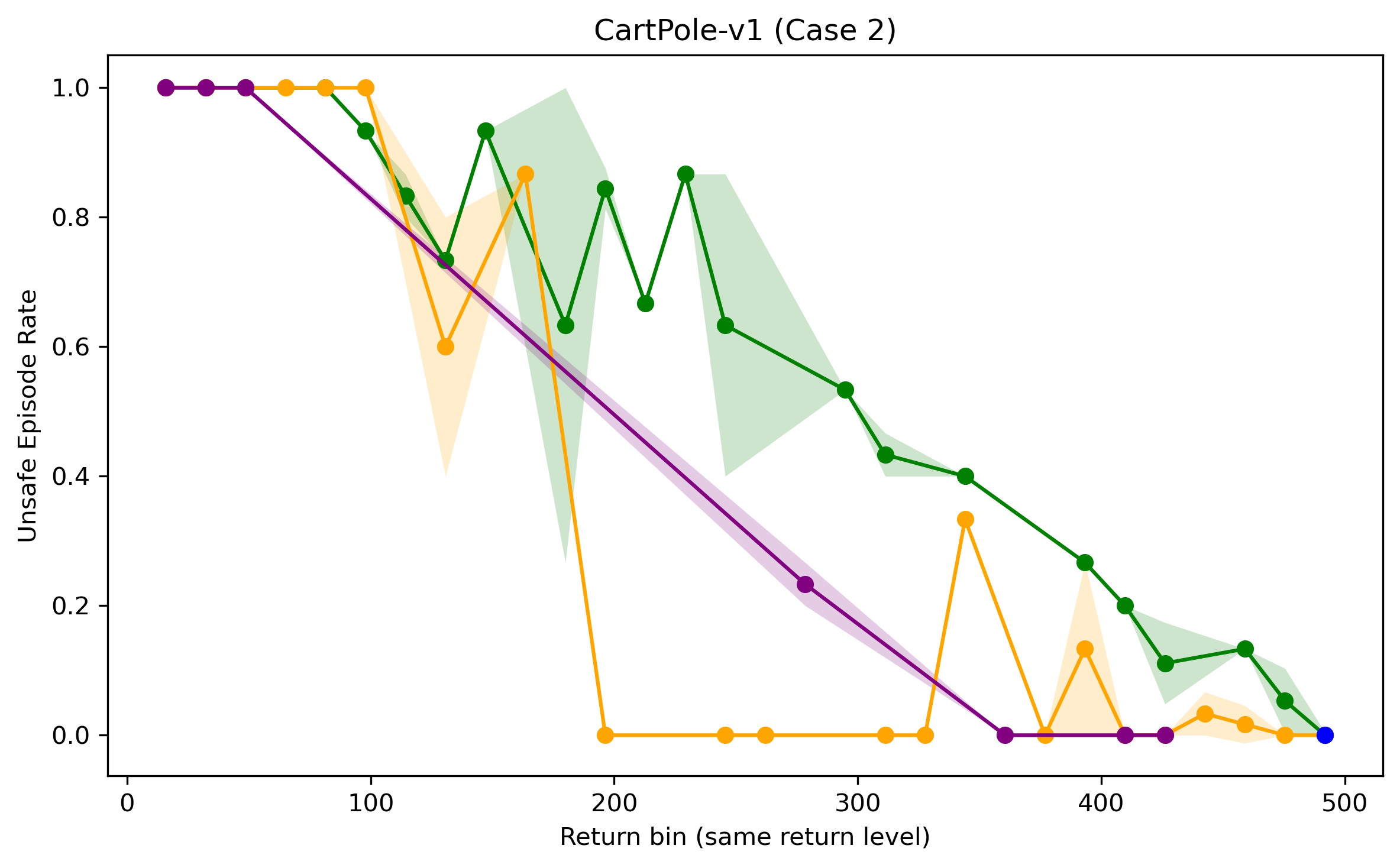}
    \caption{\centering Case 2:\\Discrete, DS safe}
    \end{subfigure}
    \begin{subfigure}{0.32\linewidth}
    \includegraphics[width=\linewidth]{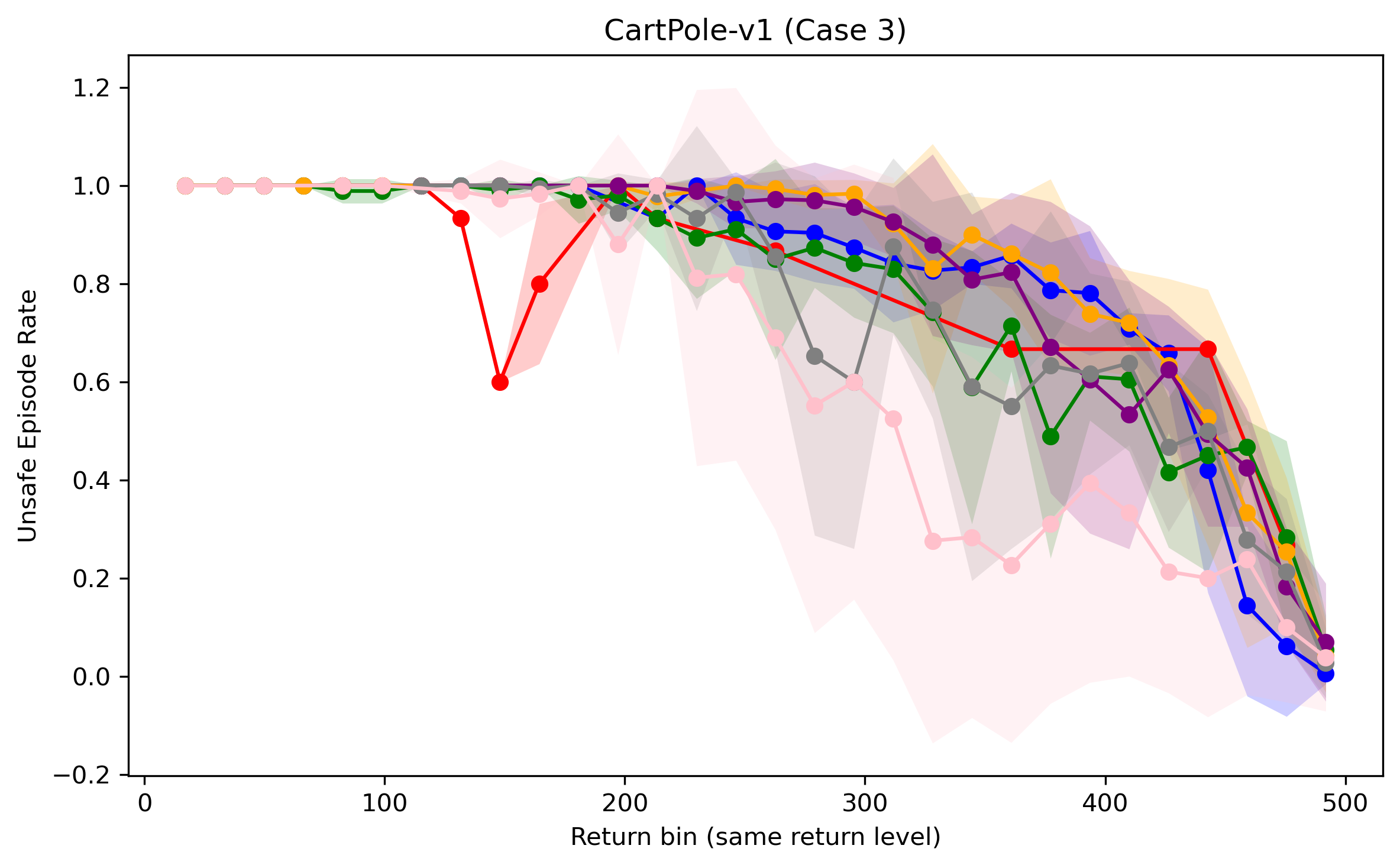}
    \caption{\centering Case 3:\\Continuous, HC safe, mean-noise}
    \end{subfigure}
    
    \begin{subfigure}{0.32\linewidth}
    \includegraphics[width=\linewidth]{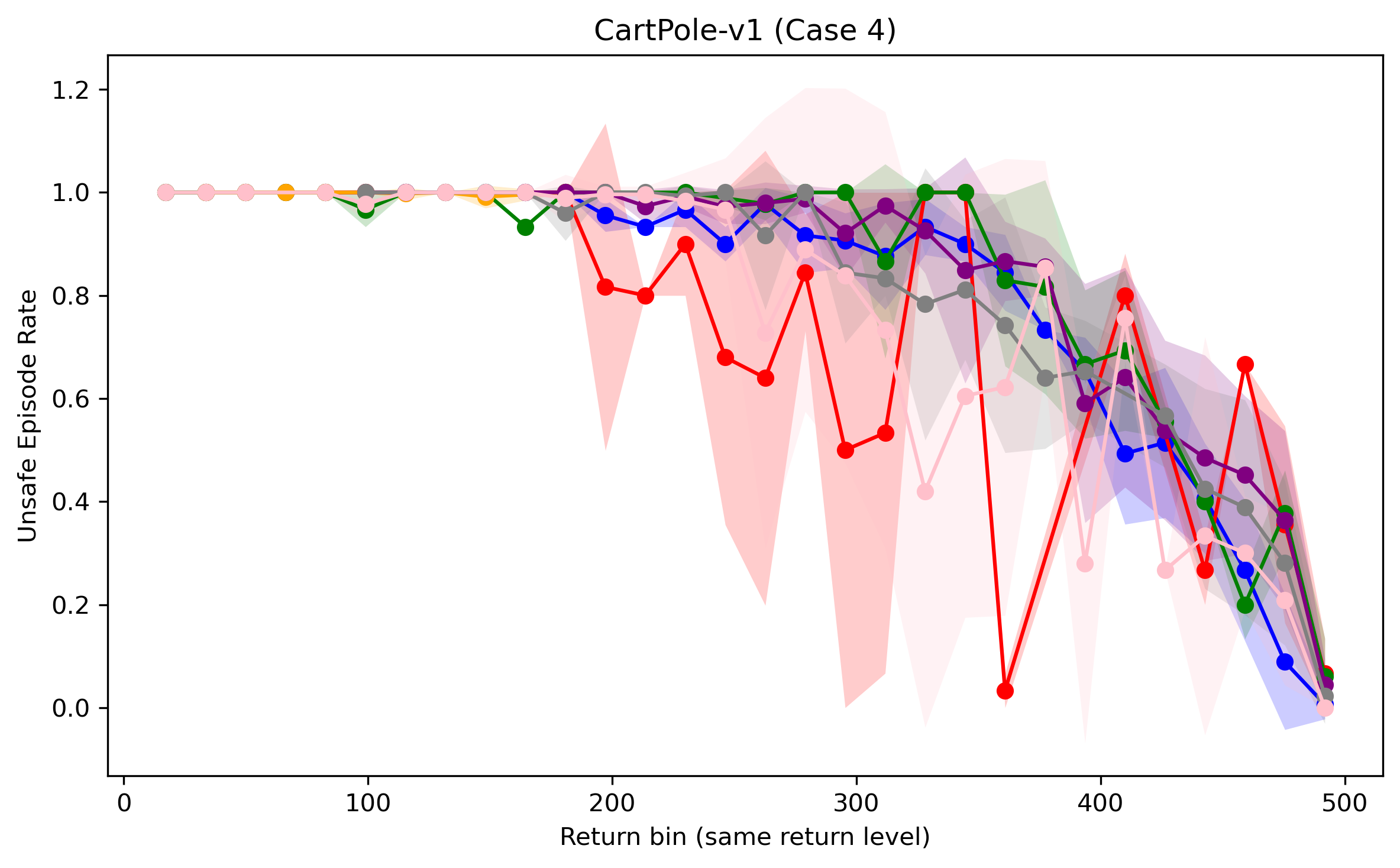}
    \caption{\centering Case 4:\\Continuous, HC safe, distributional}
    \end{subfigure}
    \begin{subfigure}{0.32\linewidth}
    \includegraphics[width=\linewidth]{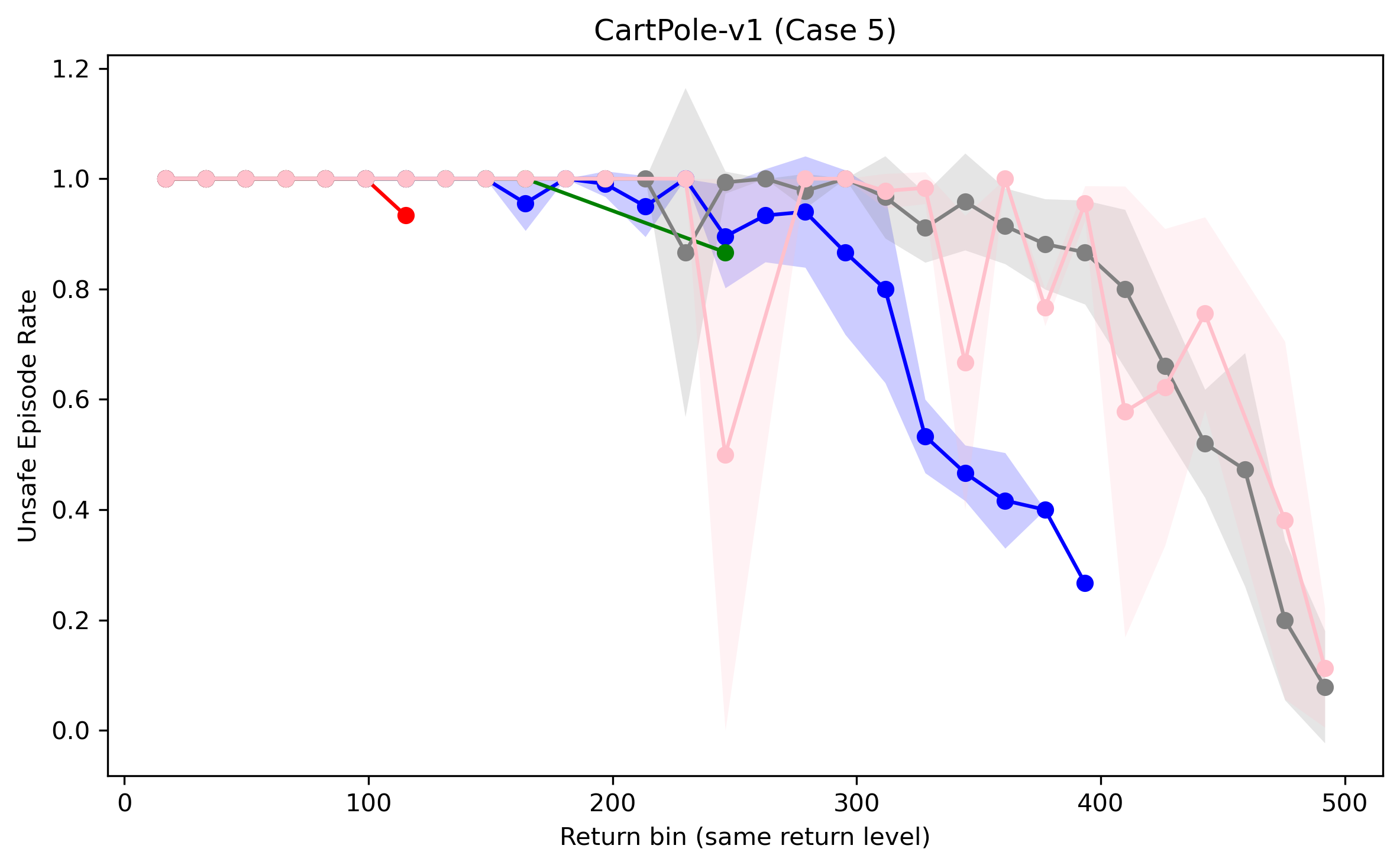}
    \caption{\centering Case 5:\\Continuous, DS safe, mean-noise}
    \end{subfigure}
    \begin{subfigure}{0.32\linewidth}
    \includegraphics[width=\linewidth]{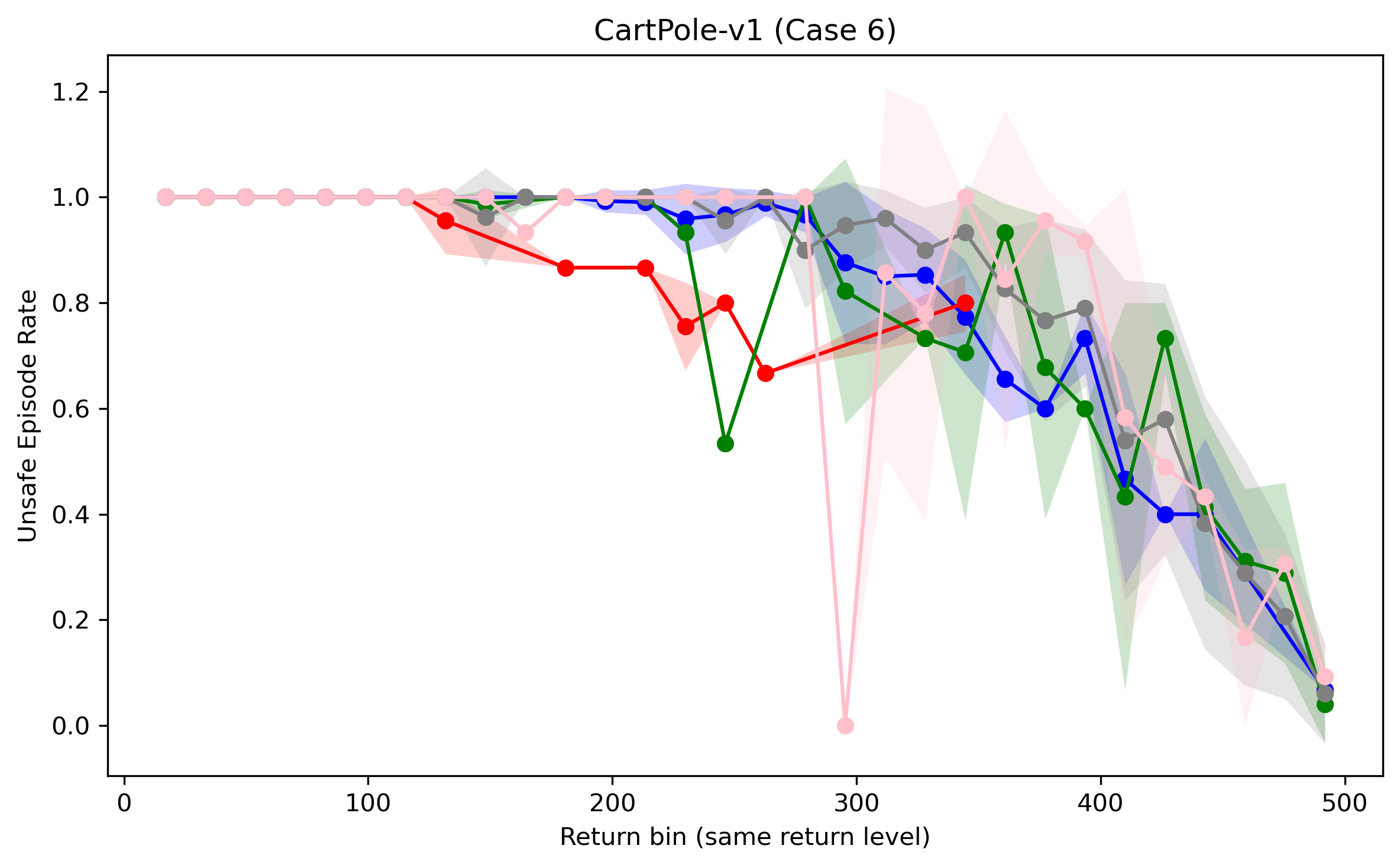}
    \caption{\centering Case 6:\\Continuous, DS safe, distributional}
    \end{subfigure}

    \caption{Same-return unsafe episode rate curve on CartPole-v1 comparing the proposed method with baselines under different settings.}
    \label{fig8}
\end{figure}

Following the same binning procedure as in~\cref{fig6}, we further evaluate additional safety metrics by grouping results into return bins and comparing them across models within each bin.

We first consider the \textit{risk severity}, defined as the average magnitude of violation beyond a safety margin. Specifically, at each time step, we measure the extent to which the pole angle exceeds a predefined margin and average it over the episode:
\[ {\rm risk severity} = \frac{1}{T} \sum_t \max(0,|\theta_t|-\theta_{\text{margin}}), \]
where $\theta_{\text{margin}} = 7^\circ$. This metric captures not only whether the safety boundary is violated, but also the severity and persistence of operation near or beyond unsafe regions. As shown in~\cref{fig7}, the proposed method consistently achieves low risk severity across all return levels with a particularly stable trend. This indicates that even when safety violations occur, their magnitude remains small. Moreover, our method effectively suppresses behavior near unsafe boundaries. This shows that the proposed method reduces pre-unsafe behavior compared to baselines, especially in low-return regimes.

We also evaluate the \textit{unsafe episode rate}, defined as the fraction of episodes in which the pole angle exceeds a predefined unsafe threshold $(\pm 9^\circ)$ at least once. This metric reflects episode-level safety violations. In~\cref{fig8}, the proposed method achieves lower unsafe rates than baselines across most return ranges. Notably, near the highest return bin (around return 500), our method consistently attains the lowest unsafe rate in almost all cases. This suggests that the trained policy is stable and effectively maintains safety without sacrificing performance.

Overall, these results demonstrate that the proposed method achieves consistently lower risk and safer behavior across all return levels, with clear advantages in preventing near-unsafe and high-variance behaviors.

\subsection{Proposed Algorithms}
\begin{algorithm}[!p]
\caption{Training DS safe}
\begin{algorithmic}[1]
\State Initialize policy parameter $\psi \in {\mathbb R}^n$
\State Initialize learning rate $\alpha>0$
\For{$k \in\{0,1,\ldots,T-1\}$}
\State Sample a mini-batch $B=\{(s,a,r,s')\}$ from $D_{\rm safe}$
\State Compute the log-likelihood objective:
\[ L(\psi;B): = \sum_{(s,a) \in B} \ln \pi_\psi(a | s) \]
\State Update parameters via gradient ascent:
\[ \psi \leftarrow \psi + \alpha \nabla_\psi L(\psi;B) \]
\EndFor
\end{algorithmic} \label{alg1}
\end{algorithm}

\begin{algorithm}[!p]
\caption{Safe-support Q-learning (Tabular, Online)}
\begin{algorithmic}[1]
\State Initialize $Q(s,a)$ for all $s \in {\cal S}$ and $a \in {\cal A}$
\State Train the behavior policy $\pi_b$
\For{episode $i \in 1,2,\ldots$}
\State Sample initial state $s_0$
\For{$k = 0,1,\ldots,\tau -1$}
\State Sample action from the behavior policy: $a_k \sim \pi_b(\cdot | s_k)$
\State Observe next state $s_{k+1} \sim P(\cdot|s_k,a_k)$ and reward $r_{k+1}=r(s_k,a_k,s_{k+1})$
\State Compute target value $y_k^{\rm safe}$ as defined in~\eqref{eqn:proposed_y}
\State Update Q-function
\[ Q(s_k,a_k) \leftarrow Q(s_k,a_k) + \alpha \left(y_k^{\rm safe} - Q(s_k,a_k)\right) \]
\EndFor
\EndFor
\end{algorithmic} \label{alg3}
\end{algorithm}

\begin{algorithm}[!p]
\caption{Safe-support Q-learning with DQN (Online)}
\begin{algorithmic}[1]
\State Initialize learning rate $\alpha>0$
\State Initialize replay memory $D$ with capacity $|D|$
\State Initialize online parameters $\theta \in {\mathbb R}^m$
\State Set target parameters $\theta' \leftarrow \theta$
\State Train HC safe $\pi_b$
\For{episode $i \in \{1,2,\ldots\}$}
\State Sample initial state $s_0$
\For{$k \in\{0,1,\ldots, \tau-1\}$}
\State Sample action from the behavior policy: $a_k \sim \pi_b(\cdot | s_k)$
\State Observe next state $s_{k+1} \sim P(\cdot | s_k,a_k)$ and reward $r_{k+1}:=r(s_k,a_k,s_{k+1})$
\State Store transition $(s_k,a_k,r_{k+1},s_{k+1})$ in $D$
\State Sample a mini-batch $B=\{(s,a,r,s')\}$ from $D$
\State Perform gradient descent
\[ \theta \leftarrow \theta - \alpha \nabla_\theta L(\theta;B), \]
\State where $L(\theta;B)$ is defined in~\eqref{eq:loss-deep-safe-Q} with the target value~\eqref{eq:target-deep-safe-Q}
\State Every $C$ steps, update target network:
\[ \theta' \leftarrow \theta \]
\EndFor
\EndFor
\end{algorithmic} \label{alg2}
\end{algorithm}

\begin{algorithm}[!p]
\caption{Safe-support Q-learning with DQN (Offline)}
\begin{algorithmic}[1]
\State Initialize learning rate $\alpha>0$
\State Initialize online parameters $\theta \in {\mathbb R}^m$
\State Set target parameters $\theta' \leftarrow \theta$
\State Train DS safe $\pi_b$ using~\cref{alg1}
\For{$k \in \{0,1,\ldots,T-1\}$}
\State Sample a mini-batch $B=\{(s,a,r,s')\}$ from $D_{\rm safe}$
\State Perform gradient descent
\[ \theta \leftarrow \theta - \alpha \nabla_\theta L(\theta;B), \]
\State where $L(\theta;B)$ is defined in~\eqref{eq:loss-deep-safe-Q} with the target value~\eqref{eq:target-deep-safe-Q}
\State Every $C$ steps, update target network:
\[ \theta' \leftarrow \theta \]
\EndFor
\end{algorithmic} \label{alg4}
\end{algorithm}

\begin{algorithm}[!p]
\caption{Safe-support Q-learning with mean-noise behavior policy (Continuous, Online)}
\begin{algorithmic}[1]
\State Initialize learning rate $\alpha_Q > 0$ and $\alpha_\pi > 0$
\State Initialize replay memory $D$ with capacity $|D|$
\State Initialize online parameters $\theta \in {\mathbb R}^m$
\State Set target parameters $\theta' \leftarrow \theta$
\State Initialize policy parameters $\phi \in \mathbb{R}^{m'}$
\State Train HC safe $\pi_b$
\State
\State \textbf{Stage 1: Training Q-function}
\For{episode $i \in \{1,2,\ldots\}$}
\State Sample initial state $s_0$
\For{$k \in\{0,1,\ldots, \tau-1\}$}
\State Sample action from the behavior policy: $a_k \sim \pi_b(\cdot | s_k)$
\State Observe next state $s_{k+1} \sim P(\cdot | s_k,a_k)$ and reward $r_{k+1}:=r(s_k,a_k,s_{k+1})$
\State Store transition $(s_k,a_k,r_{k+1},s_{k+1})$ in $D$
\State Sample a mini-batch $B=\{(s,a,r,s')\}$ from $D$
\For{each $(s,a,r,s') \in B$}
\State Sample i.i.d. actions $\{a_i'\}_{i=1}^N$ from $\pi_b(\cdot | s')$
\State Compute the corresponding target value $y^{\rm safe}$ as~\eqref{cont target}
\EndFor
\State Perform a gradient descent step on the Q-network:
\[ \theta \leftarrow \theta - \alpha_Q \nabla_\theta L(\theta;B), \]
\State where $L(\theta;B)$ is defined in~\eqref{eq:loss-deep-safe-Q} with the target value $y^{\rm safe}$
\State Every $C$ steps, update target network:
\[ \theta' \leftarrow \theta \]
\EndFor
\EndFor
\State
\State \textbf{Stage 2: Training Policy}
\State Freeze the Q-network parameters $\theta$
\For{$t = 1,2,\ldots,T$}
\State Sample states $\{s_i\}_{i=1}^p$ from $\mathcal{N}(0,\sigma_1^2 I)$
\State Sample i.i.d. noises $\{w_j\}_{j=1}^q$ from $\mathcal{N}(0,\sigma_2^2 I)$
\State Perform a gradient descent step on the policy network:
\[ \phi \leftarrow \phi - \alpha_\pi \nabla_\phi \tilde L_M(\phi), \]
\State where $\tilde L_M(\phi)$ is defined in~\eqref{cont policy obj_case1_fin}
\EndFor
\end{algorithmic} \label{alg5}
\end{algorithm}

\begin{algorithm}[!p]
\caption{Safe-support Q-learning with distributional behavior policy (Continuous, Online)}
\begin{algorithmic}[1]
\State Initialize learning rate $\alpha_Q > 0$ and $\alpha_\pi > 0$
\State Initialize replay memory $D$ with capacity $|D|$
\State Initialize online parameters $\theta \in {\mathbb R}^m$
\State Set target parameters $\theta' \leftarrow \theta$
\State Initialize policy parameters $\phi \in \mathbb{R}^{m'}$
\State Train HC safe $\pi_b$
\State
\State \textbf{Stage 1: Training Q-function}
\For{episode $i \in \{1,2,\ldots\}$}
\State Sample initial state $s_0$
\For{$k \in\{0,1,\ldots, \tau-1\}$}
\State Sample action from the behavior policy: $a_k \sim \pi_b(\cdot | s_k)$
\State Observe next state $s_{k+1} \sim P(\cdot | s_k,a_k)$ and reward $r_{k+1}:=r(s_k,a_k,s_{k+1})$
\State Store transition $(s_k,a_k,r_{k+1},s_{k+1})$ in $D$
\State Sample a mini-batch $B=\{(s,a,r,s')\}$ from $D$
\For{each $(s,a,r,s') \in B$}
\State Sample i.i.d. actions $\{a_i'\}_{i=1}^N$ from $\pi_b(\cdot | s')$
\State Compute the corresponding target value $y^{\rm safe}$ as~\eqref{cont target}
\EndFor
\State Perform a gradient descent step on the Q-network:
\[ \theta \leftarrow \theta - \alpha_Q \nabla_\theta L(\theta;B), \]
\State where $L(\theta;B)$ is defined in~\eqref{eq:loss-deep-safe-Q} with the target value $y^{\rm safe}$
\State Every $C$ steps, update target network:
\[ \theta' \leftarrow \theta \]
\EndFor
\EndFor
\State
\State \textbf{Stage 2: Training Policy}
\State Freeze the Q-network parameters $\theta$
\For{$t = 1,2,\ldots,T$}
\State Sample states $\{s_i\}_{i=1}^p$ from $\mathcal{N}(0,\sigma_1^2 I)$
\State Sample i.i.d. noises $\{w_j\}_{j=1}^q$ from $\mathcal{N}(0,\sigma_2^2 I)$
\State Perform a gradient descent step on the policy network:
\[ \phi \leftarrow \phi - \alpha_\pi \nabla_\phi \tilde L_D(\phi), \]
\State where $\tilde L_D(\phi)$ is defined in~\eqref{cont policy obj_case2_fin}
\EndFor
\end{algorithmic} \label{alg6}
\end{algorithm}

\end{document}